%% file: intrinsic_offline_RL_bound_arxiv.tex
\documentclass[11pt]{article}
\usepackage[in]{fullpage}
\usepackage[parfill]{parskip}
\usepackage{amsmath,amsthm,amssymb,bbm}
\usepackage{mathtools}
\usepackage{cases}
\usepackage{dsfont}
\usepackage{microtype}
\usepackage{subfigure}
\usepackage{algorithm,algorithmic}
\usepackage{color}
\usepackage{appendix}

\usepackage{bm}
\usepackage{subfigure}
\usepackage{algorithm,algorithmic}
\usepackage{color}
\usepackage{booktabs}       
\usepackage{appendix}
\usepackage{authblk}

\usepackage{url}
\usepackage[authoryear]{natbib}
\usepackage[colorlinks,citecolor=blue,urlcolor=blue,linkcolor=blue,linktocpage=true]{hyperref}
\usepackage{cleveref}
\crefformat{equation}{(#2#1#3)}
\crefrangeformat{equation}{(#3#1#4) to~(#5#2#6)}
\crefname{equation}{}{}
\Crefname{equation}{}{}

\crefname{definition}{\textbf{definition}}{definitions}
\Crefname{definition}{Definition}{Definitions}
\crefname{assumption}{\textbf{assumption}}{assumptions}
\Crefname{assumption}{Assumption}{Assumptions}

\definecolor{maroon}{RGB}{192,80,77}
\definecolor{mypink3}{cmyk}{0, 0.7808, 0.4429, 0.1412}

\newtheorem{theorem}{Theorem}[section]
\newtheorem{lemma}[theorem]{Lemma}
\newtheorem{proposition}[theorem]{Proposition}

\newtheorem{definition}[theorem]{Definition}

\newtheorem{remark}[theorem]{Remark}
\newtheorem{assumption}[theorem]{Assumption}

\usepackage{amsmath}

\newcommand\norm[1]{\left\lVert#1\right\rVert}

\newcommand{\argmax}{\mathop{\mathrm{argmax}}}

\usepackage{tikz}

\def\E{\mathbb{E}}
\def\P{\mathbb{P}}

\def\Var{\mathrm{Var}}

\def\R{\mathbb{R}}

\begin{document}

\title{{Towards Instance-Optimal Offline Reinforcement Learning with Pessimism\footnote{To appear at {Conference on Neural Information Processing Systems} (NeurIPS), 2021.}}}

\author[1,2]{Ming Yin }
\author[1]{Yu-Xiang Wang}
\affil[1]{Department of Computer Science, UC Santa Barbara}
\affil[2]{Department of Statistics and Applied Probability, UC Santa Barbara}
\affil[ ]{\texttt{ming\_yin@ucsb.edu}   \quad
	\texttt{yuxiangw@cs.ucsb.edu}}

\date{}

\maketitle

\begin{abstract}
	\input{sections/abstract}

\end{abstract}

\tableofcontents

\input{sections/introduction}

\section{Preliminaries }\label{sec:formulation}
\input{sections/formulation}

\section{A warm-up case study: Vanilla Pessimistic Value Iteration}\label{sec:VPVI}
\input{sections/results}

\section{Sketch of the Analysis}\label{sec:proof_sketch}
\input{sections/proof_overview}

\section{Discussion and Conclusion}\label{sec:conclusion}

This work studies the offline reinforcement learning problem and contributes the intrinsic offline learning bound which is a near-optimal and strong adaptive bound that subsumes existing worst-case bounds under various assumptions. The adaptive characterization of the intrinsic bound abandons the explicit dependence on $H,S,A,C^\star,d_m$ and helps reveal the fundamental hardness of each individual instances. In this sense, it draws a clearer picture of what offline reinforcement learning looks like and serves as a step towards instance optimality in offline RL.

Nevertheless, it is still unclear whether \eqref{eqn:APVI} is optimal over all the instances. For example, for fully deterministic systems, our bound provides a faster convergence $H^3/n\bar{d}_m$, however, $H^3$ might be very suboptimal comparing to algorithms that are designed specifically for deterministic MDPs, since the agent only need to experience each location $(s,a)$ once to fully acquire the dynamic $P(\cdot|s,a)$ and $r(s,a)$. Recently, \cite{xiao2021optimality} goes beyond the minimax (worst case) optimality and studies the instance optimality behavior for the simplified batch bandit setting. One of their findings is: for ``easy enough'' tasks, different type of algorithms can be equally good, provably. This seems to suggest instance optimality only matters for problems that are hard to learn. How to formally define the instance optimality metric for different problems remains an open problem and how to design a single algorithm that can achieve optimality for all instances could be challenging (or even infeasible). We leave those as the future works.




\subsection*{Acknowledgment}
The research is partially supported by NSF Awards \#2007117 and \#2003257. MY would like to thank Chenjun Xiao for bringing up a related literature \citep{xiao2021optimality} and Masatoshi Uehara for helpful suggestions.

\bibliographystyle{plainnat}
\bibliography{sections/stat_rl}

\appendix

\clearpage
\begin{center}
	 {\LARGE \textbf{Appendix}}
\end{center}

\input{sections/proofs}

\end{document}

%% file: sections/abstract.tex

We study the \emph{offline reinforcement learning}  (offline RL) problem, where the goal is to learn a reward-maximizing policy in an unknown \emph{Markov Decision Process} (MDP) using the data coming from a policy $\mu$. In particular, we consider the sample complexity problems of offline RL for finite-horizon MDPs. Prior works study this problem based on different data-coverage assumptions, and their learning guarantees are expressed by the covering coefficients which lack the explicit characterization of system quantities. In this work, we analyze the \emph{Adaptive Pessimistic Value Iteration} (APVI) algorithm and derive the suboptimality upper bound that nearly matches
\begin{equation}\label{eqn:intrinsic}
O\left(\sum_{h=1}^H\sum_{s_h,a_h}d^{\pi^\star}_h(s_h,a_h)\sqrt{\frac{\mathrm{Var}_{P_{s_h,a_h}}{(V^\star_{h+1}+r_h)}}{d^\mu_h(s_h,a_h)}}\sqrt{\frac{1}{n}}\right).
\end{equation}
In complementary, we also prove a \emph{per-instance} information-theoretical lower bound under the weak assumption that $d^\mu_h(s_h,a_h)>0$ if $d^{\pi^\star}_h(s_h,a_h)>0$. Different from the previous minimax lower bounds, the \emph{per-instance lower bound} (via local minimaxity) is a much stronger criterion as it applies to individual instances separately. Here $\pi^\star$ is a optimal policy, $\mu$ is the behavior policy and $d_h^\mu$ is the marginal state-action probability. We call \eqref{eqn:intrinsic} the \emph{intrinsic offline reinforcement learning bound} since it directly implies all the existing optimal results: minimax rate under uniform data-coverage assumption, horizon-free setting, single policy concentrability, and the tight problem-dependent results. Later, we extend the result to the \emph{assumption-free} regime (where we make no assumption on $
\mu$) and obtain the assumption-free intrinsic bound. Due to its generic form, we believe the intrinsic bound could help illuminate what makes a specific problem hard and reveal the fundamental challenges in offline RL.



%% file: sections/introduction.tex

\section{Introduction}\label{sec:introduction}

In \emph{offline reinforcement learning} (offline RL \cite{levine2020offline,lange2012batch}), the goal is to learn a reward-maximizing policy in an unknown environment (\emph{Markov Decision Process} or MDP) using the historical data coming from a (fixed) behavior policy $\mu$. Unlike online RL, where the agent can keep interacting with the environment and gain new feedback by exploring unvisited state-action space, offline RL usually populates when such online interplays are expensive or even unethical. Due to its nature of without the access to interact with the MDP model (which causes the distributional mismatches), most of the literature that study the sample complexity / provable efficiency of offline RL (\emph{e.g.} \cite{le2019batch,chen2019information,xie2020q,xie2020batch,yin2021near,yin2021nearoptimal,ren2021nearly,rashidinejad2021bridging,xie2021policy}) rely on making different data-coverage assumptions for making the problem learnable and provide the near-optimal worst-case performance bounds that depend on their data-coverage coefficients. Those results are valuable in general as they do not depend on the structure of the particular problem, therefore, remain valid even for pathological MDPs. But is this good enough?

In practice, the empirical performances of offline reinforcement learning (\emph{e.g.} \cite{gulcehre2020rl,fu2020d4rl,fu2021benchmarks,janner2021reinforcement}) are often far better than what those non-adaptive / problem-independent bounds would indicate. Although empirical evidence can help explain why we may observe better or worse performances on different MDPs, a systematic understanding of what types of decision processes and what kinds of behavior policies are inherently easier or more challenging for offline RL is lacking. Besides, despite the fact that a non-adaptive bound can learn even the pathological examples within the assumption family, there is no guarantee for the instances outside the family. However, practical offline reinforcement learning problems are usually beyond the scope of certain data-coverage assumptions, which limits the applicability of those results. Can we make as few assumptions as possible? Or even more, what can we guarantee when no assumption is made about offline learning? 

Those motivate us to derive the provably efficient bounds that are adaptive to the individual instances but only require minimal assumptions so they can be widely applied in most cases. Ideally, such bounds should characterize the system structures of the specific problems, hold even for peculiar instances that do not satisfy the standard data-coverage assumptions, and recover the worst-case guarantees when the assumptions are satisfied. As mentioned in \cite{zanette2019tighter}, a fully adaptive characterization in RL is important as it might bring considerable saving in the time spent designing domain-specific RL solutions and in training a human expert to judge and recognize the complexity of different problems.

\subsection{Our contribution}\label{sec:contribution}

In this work, we provide the analysis for the \emph{adaptive pessimistic value iteration} (APVI) (Algorithm~\ref{alg:APVI}) with finite horizon time-inhomogeneous (non-stationary) MDPs and derive a strong adaptive bound that is near-optimal under the weak assumption $d^\mu_h(s_h,a_h)>0$ if $d^{\pi^\star}_h(s_h,a_h)>0$ (Theorem~\ref{thm:APVI}). Specifically, our bound (quantity \eqref{eqn:intrinsic}) explicitly depends on the marginal importance ratios (between the optimal policy $\pi^\star$ and the behavior policy $\mu$) and the per-step conditional variances. In addition, we provide an instance-dependent (local minimax) lower bound (Theorem~\ref{thm:adaptive_lower_bound}) to certify \eqref{eqn:intrinsic} is nearly optimal at the instance level for offline learning and call it \emph{the intrinsic offline learning bound}. The intrinsic bound has the following consequences.

\begin{itemize}
	\item In the non-adaptive / worst-case regime (\ref{subsec:one}-\ref{subsec:two}), the intrinsic bound implies $\widetilde{O}(H^3/d_m\epsilon^2)$ complexity under the uniform data-coverage \ref{assum:uniform}, $\tilde{O}(H^3SC^\star/\epsilon^2)$ complexity under the single policy concentrability assumption \ref{assum:single_concen} and $\widetilde{O}(H/d_m\epsilon^2)$ complexity when the sum of rewards is bounded by $1$. All of those are optimal in their respectively regimes \citep{yin2021near,rashidinejad2021bridging,xie2021policy,ren2021nearly};
	\item In the adaptive domain (\ref{subsec:three}), the intrinsic bound implies the tight problem-dependent counterpart of \cite{zanette2019tighter}, yields $\tilde{O}(H^3/nd_m)$ fast convergence in the deterministic systems, has improved complexity in the partially deterministic systems and a family of highly mixing problems, and remains optimal when reducing to the tabular contextual bandits.   
\end{itemize}

Beyond the above, due to the generic form of the intrinsic bound, we could come up with as many problem instances (that are of our interests) as possible and study their properties. In this sense, the intrinsic bound helps illuminate the fundamental nature of offline RL.

Furthermore, as a step towards \emph{assumption-free} offline reinforcement learning, we build a {modified} AVPI and obtain an adaptive bound that could characterize the suboptimality gap in the state-action space that is agnostic to the behavior policy (Theorem~\ref{thm:AFRL}). To the best of our knowledge, all of these results are the first of its kinds.

\subsection{Related work}
Finite sample analysis for offline reinforcement learning can be traced back to \cite{szepesvari2005finite,antos2008fitted,antos2008learning} for the \emph{infinite horizon discounted setting} via Fitted Q-Iteration (FQI) type function approximation algorithms. \citep{chen2019information,le2019batch,xie2020batch,xie2020q} follow this line of research and derive the information-theoretical bounds. Recently, \cite{xie2020batch} considers the offline RL with only the {realizability} assumption, \cite{liu2020provably,chang2021mitigating} considers the offline RL {without sufficient coverage} and \cite{kidambi2020morel,uehara2021pessimistic} uses the model-based approach for addressing offline RL. Under those weak coverage assumption, their finite sample analysis are suboptimal (\emph{e.g.} in terms of the effective horizon $(1-\gamma)^{-1}$). Recently, \cite{yin2021near,yin2021nearoptimal,ren2021nearly} study the finite horizon case. In the linear MDP case, \cite{jin2020pessimism} studies the pessimistic algorithm for offline policy learning under only the compliance assumption, and, concurrently, \cite{xie2021bellman} proposes the general pessimistic function approximation framework with instantiation in linear MDP and \cite{zanette2021provable} shows actor-critic style algorithm is near-optimal for linear Bellman complete model. In addition, \cite{wang2020statistical,zanette2020exponential} prove some exponential lower bounds under their linear function approximation assumptions.

Among them, there are a few works that achieve the sample optimality under their respective assumptions. Under the uniform data coverage (minimal state-action probability $d_m>0$), \cite{yin2021near} first proves the optimal $\tilde{O}(H^3/d_m\epsilon^2)$ complexity in the time-inhomogeneous MDP. Recently, \cite{yin2021nearoptimal} designs the offline variance reduction algorithm to achieve the optimal $\tilde{O}(H^2/d_m\epsilon^2)$ rate for the time-homogeneous case.  Under the setting where the total cumulative reward is bounded by $1$, \cite{ren2021nearly} obtains the horizon-free result with $\tilde{O}(1/d_m)$. More recently, \cite{rashidinejad2021bridging} considers the single concentrability coefficient $C^\star:=\max_{s,a}{d^{\pi^\star}(s,a)}/{d^\mu(s,a)}$ and derives the upper bound $\tilde{O}[(1-\gamma)^{-5}SC^\star/\epsilon^2]$ in the infinite horizon setting which is recently improved by the concurrent work \cite{xie2021policy}. While those worst-case guarantees are desirable, none of them can explain the hardness of the individual problems.\footnote{We do mention \cite{zanette2021provable} is near-optimal in their setting, but it is unclear whether it remains optimal in the standard setting where $Q^\pi\in[0,H]$, since there is an additional $H$ factor by rescaling.}

%% file: sections/formulation.tex

\textbf{Episodic non-stationary (time-varying) reinforcement learning.} A finite-horizon \emph{Markov Decision Process} (MDP) is denoted by a tuple $M=(\mathcal{S}, \mathcal{A}, P, r, H, d_1)$ \citep{sutton2018reinforcement}, where $\mathcal{S}$ is the finite state space and $\mathcal{A}$ is the finite action space with $S:=|\mathcal{S}|<\infty,A:=|\mathcal{A}|<\infty$. A non-stationary transition kernel $P_h:\mathcal{S}\times\mathcal{A}\times\mathcal{S} \mapsto [0, 1]$ maps each state action$(s_h,a_h)$ to a probability distribution $P_h(\cdot|s_h,a_h)$ and $P_h$ can be different across the time. Besides, $r : \mathcal{S} \times{A} \mapsto \mathbb{R}$ is the expected instantaneous reward function satisfying $0\leq r\leq1$. $d_1$ is the initial state distribution. $H$ is the horizon. A policy $\pi=(\pi_1,\ldots,\pi_H)$ assigns each state $s_h \in \mathcal{S}$ a probability distribution over actions according to the map $s_h\mapsto \pi_h(\cdot|s_h)$ $\forall h\in[H]$.  An MDP together with a policy $\pi$ induce a random trajectory $ s_1, a_1, r_1, \ldots, s_H,a_H,r_H,s_{H+1}$ with $s_1 \sim d_1, a_h \sim \pi(\cdot|s_h), s_{h+1} \sim P_h (\cdot|s_h, a_), \forall h \in [H]$ and $r_h$ is a random realization given the observed $s_h,a_h$.

\textbf{$Q$-values, Bellman (optimality) equations.} The value function $V^\pi_h(\cdot)\in \R^S$ and Q-value function $Q^\pi_h(\cdot,\cdot)\in \R^{S\times A}$ for any policy $\pi$ is defined as:
$
V^\pi_h(s)=\E_\pi[\sum_{t=h}^H r_{t}|s_h=s] ,\;\;Q^\pi_h(s,a)=\E_\pi[\sum_{t=h}^H  r_{t}|s_h,a_h=s,a],\;\forall s,a\in\mathcal{S},\mathcal{A},h\in[H].
$ The performance is defined as $v^\pi:=\E_{d_1}\left[V^\pi_1\right]=\E_{\pi,d_1}\left[\sum_{t=1}^H  r_t\right]$, where we denote $V_h^\pi,Q_h^\pi$ as column vectors and $P_h\in\R^{SA\times S}$ the transition matrix, then the vector form Bellman (optimality) equations follow $\forall h\in[H]$:
$
Q^\pi_h=r_h+P_hV^\pi_{h+1},\;\;V^\pi_h=\E_{a\sim\pi_h}[Q^\pi_h], \;\;\;Q^\star_h=r_h+P_hV^\star_{h+1},\; V^\star_h=\max_a Q^\star_h(\cdot,a).
$In addition, we denote the per-step marginal state-action occupancy $d^\pi_h(s,a)$ as:
{$
d^\pi_h(s,a):=\P[s_h=s|s_1\sim d_1,\pi]\cdot\pi_h(a|s),
$
}which is the marginal state-action probability at time $h$. 


\textbf{Offline setting and the goal.} The offline RL requires the agent to find a policy $\pi$ such that the performance $v^\pi$ is maximized, given only the episodic data {\small$\mathcal{D}=\left\{\left(s_{h}^{\tau}, a_{h}^{\tau}, r_{h}^{\tau}, s_{h+1}^{\tau}\right)\right\}_{\tau\in[n]}^{h\in[H]}$} rolled out from some behavior policy $\mu$. The offline nature requires we cannot change $\mu$ and in particular we do not assume the functional knowledge of $\mu$. That is to say, given the batch data $\mathcal{D}$ and a targeted accuracy $\epsilon>0$, the offline RL seeks to find a policy $\pi_\text{alg}$ such that $v^\star-v^{\pi_\text{alg}}\leq\epsilon$.

\subsection{Assumptions in offline RL}

We revise several types of assumptions proposed by existing studies that can yield provably efficient results. Recall $d^\mu_h(s_h,a_h)$ is the marginal state-action probability and $\mu$ is the behavior policy.

\begin{assumption}[Uniform data coverage \citep{yin2021near}]\label{assum:uniform}
	The behavior policy obeys that $d_m:=\min_{h,s_h,a_h} d_h^\mu (s_h,a_h) > 0$. Here the infimum is over all the states satisfying there exists certain policy so that this state can be reached by the current MDP with this policy. 
\end{assumption}
This is the strongest assumption in offline RL as it requires $\mu$ to explore each state-action pairs with positive probability. Under \ref{assum:uniform}, it mostly holds $1/d_m\geq SA$. This reveals offline learning is generically harder than \emph{the generative model setting} \citep{agarwal2020model} in the statistical sense. On the other hand, this is required for the \emph{uniform OPE} task in \cite{yin2021near} as it seeks to simultaneously evaluate all the policies within the policy class and it is in general a harder task than offline learning itself.

\begin{assumption}[Uniform concentrability \cite{szepesvari2005finite,chen2019information}]\label{assum:concen}
	For all the policies, $C_\mu:=\sup_{\pi,h} ||d^\pi_h(\cdot,\cdot)/d^\mu_h(\cdot,\cdot)||_\infty<\infty$.
\end{assumption}

This is a classical offline RL condition that is commonly assumed in the function approximation scheme (\emph{e.g.} Fitted Q-Iteration). Qualitatively, this is a uniform data-coverage assumption that is similar to Assumption~\ref{assum:uniform}, but quantitatively, the coefficient $C_\mu$ can be smaller than $1/d_m$ due the $d^\pi_h$ term in the numerator. 

\begin{assumption}[\cite{liu2019off}]\label{assum:single_concen}
	There exists one optimal policy $\pi^\star$, s.t. $\forall s_h,a_h\in\mathcal{S},\mathcal{A}$, $d^\mu_h(s_h,a_h)>0$ if $d^{\pi^\star}_h(s_h,a_h)>0$. We further denote the trackable set as $\mathcal{C}_h:=\{(s_h,a_h):d^\mu_h(s_h,a_h)>0\}$. 
\end{assumption}
 Assumption~\ref{assum:single_concen} is (arguably) the weakest assumption needed for accurately learning the optimal value $v^\star$ and we will use \ref{assum:single_concen} for most parts of this paper. It only requires $\mu$ to trace the state-action space of one optimal policy and can be agnostic at other locations. \cite{rashidinejad2021bridging,xie2021policy} considers this assumption and provide analysis is based on the single concentrability coefficient  $C^\star:=\max_{s,a}{d^{\pi^\star}(s,a)}/{d^\mu(s,a)}$. The dependence on $C^\star$ makes their result less adaptive since there can be lots of locations that have the ratio ${d^{\pi^\star}(s,a)}/{d^\mu(s,a)}$ much smaller than $C^\star$. Furthermore, what could we end up with when \ref{assum:single_concen} is not met? We will provide our answers in the subsequent sections.

%% file: sections/results.tex

As a step towards the optimal and strong adaptive offline RL bound, we analyze \emph{the vanilla pessimistic value iteration} (VPVI), a tabular counterpart of \emph{pessimistic value iteration} (PEVI initiated in \cite{jin2020pessimism}), to understand what is missing for achieving the fully adaptivity. In particular, VPVI relies on the model-based construction.

\textbf{Model-based Components.} Given data {\small$\mathcal{D}=\left\{\left(s_{h}^{\tau}, a_{h}^{\tau}, r_{h}^{\tau}, s_{h+1}^{\tau}\right)\right\}_{\tau\in[n]}^{h \in[H]}$}, we denote $n_{s_h,a_h}:=\sum_{\tau=1}^n\mathbf{1}[s_h^{\tau},a_h^{\tau}=s_h,a_h]$ be the total counts that visit $(s_h,a_h)$ pair at time $h$, then we use the offline plug-in estimator to construct the estimators for ${P_h}$ and $r_h$ as:
{\small
\begin{equation}\label{eqn:mb_est}
\widehat{P}_h(s'|s_h,a_h)=\frac{\sum_{\tau=1}^n\mathbf{1}[(s^{\tau}_{h+1},a^{\tau}_h,s^{\tau}_h)=(s^\prime,s_h,a_h)]}{n_{s_h,a_h}},\; \widehat{r}_h(s_h,a_h)=\frac{\sum_{\tau=1}^n\mathbf{1}[(a^{\tau}_h,s^{\tau}_h)=(s_h,a_h)]\cdot r_h^\tau}{n_{s_h,a_h}},
\end{equation}
}if $n_{s_h,a_h}>0$ and $\widehat{P}_h(s'|s_h,a_h)={1}/{S},\widehat{r}_h(s_h,a_h)=0$ if $n_{s_h,a_h}=0$. In particular, we use the word ``vanilla'' as it directly mirrors \cite{jin2020pessimism} with a pessimistic penalty of order $O(H/\sqrt{n_{s_h,a_h}})$.\footnote{This is due to $\sqrt{\phi\left(s_{h}, a_{h}\right)^{\top} \Lambda_{h}^{-1} \phi\left(s_{h}, a_{h}\right)}$ reduces to $\sqrt{{1}/n_{s_h,a_h}}$ when setting $\phi(s_h,a_h)=\mathbf{1}(s_h,a_h)$ and $\lambda=0$.} With $\widehat{P}_h,\widehat{r}_h$ in Algorithm~\ref{alg:VPVI} (which we defer to Appendix), VPVI guarantees the following:

\begin{theorem}\label{thm:VPVI}
Under the Assumption~\ref{assum:single_concen}, denote $\bar{d}_m:=\min_{h\in[H]}\{d^\mu_h(s_h,a_h):d^\mu_h(s_h,a_h)>0\}$. For any $0<\delta<1$, there exists absolute constants $c_0,C'>0$, such that when $n>c_0 \cdot 1/\bar{d}_m\cdot\iota$ ($\iota=\log(HSA/\delta)$), with probability $1-\delta$, the output policy $\widehat{\pi}$ of VPVI satisfies
{
\begin{equation}\label{eqn:VPVI}
0\leq v^\star-v^{\widehat{\pi}}\leq C'H\sum_{h=1}^H\sum_{(s_h,a_h)\in\mathcal{C}_h}d^{\pi^\star}_h(s_h,a_h)\cdot\sqrt{\frac{\iota}{ n\cdot d^\mu_h{(s_h,a_h)}}}.
\end{equation}
}
\end{theorem}
The full proof can be found in Appendix~\ref{sec:VPVI_proof}. Theorem~\ref{thm:VPVI} makes some improvements over the existing works. First, it is more adaptive than the results with uniform data-coverage Assumption~\ref{assum:uniform} (\cite{yin2021near,ren2021nearly}). In addition, by straightforward calculation \eqref{eqn:VPVI} can be bounded by $\tilde{O}(\sqrt{H^4SC^\star/n})$ which improves VI-LCB \citep{rashidinejad2021bridging} by a factor of $H$.\footnote{To be rigorous, translating the result from the infinite horizon setting to the finite horizon setting requires explanation. We add this discussion in Appendix~\ref{sec:dis_VPVI}.} Besides, the analysis of VPVI also improves the direct reduction of PEVI \citep{jin2020pessimism} in the tabular case by a factor $SA$ since their $\beta=SAH$ when $d=SA$.  

However, VPVI is not optimal as the dependence on horizon is $H^4$ which does not match the optimal worst case guarantee $H^3$ \citep{yin2021near} in the nonstationary setting. Also, the explicit dependence on $H$ in \eqref{eqn:VPVI} possibly hides some key features of the specific offline RL instances. For example, no improvement can be made if the system has the deterministic transition.

\begin{algorithm}[H]
	\caption{Adaptive (\emph{assumption-free}) Pessimistic Value Iteration or LCBVI-Bernstein}
	\label{alg:APVI}
	\small{
		\begin{algorithmic}[1]
			\STATE {\bfseries Input:} Offline dataset $\mathcal{D}=\{(s_h^\tau,a_h^\tau,r_h^\tau,s_{h+1}^\tau)\}_{\tau,h=1}^{n,H}$. Set $C_1=2,C_2=14$, failure probability $\delta$.
			
			\STATE {\bfseries Initialization:} Set $\widehat{V}_{H+1}(\cdot)\leftarrow 0$.  Set $\iota=\log(HSA/\delta)$. (if assumption-free, set $M^\dagger,\widehat{M}^\dagger$ as in Section~\ref{sec:assumption_free}.)
			
			\FOR{time $h=H,H-1,\ldots,1$}
			\STATE Set $\widehat{Q}_h(\cdot,\cdot)\leftarrow {\widehat{r}_h(\cdot,\cdot)}+(\widehat{P}_{h}\cdot \widehat{V}_{h+1})(\cdot,\cdot)$\;\;(use ${\widehat{r}_h^\dagger}+(\widehat{P}_{h}^\dagger\cdot \widehat{V}_{h+1})$ if assumption-free)
			\STATE $\forall s_h,a_h$, set $\Gamma_h(s_h,a_h)=C_1\sqrt{\frac{\mathrm{Var}_{\widehat{P}_{s_h,a_h}}(\widehat{r}_h+\widehat{V}_{h+1})\cdot\iota}{n_{s_h,a_h}}}+\frac{C_2H\cdot\iota}{n_{s_h,a_h}}$ if $n_{s_h,a_h}\geq 1$, o.w. set to $\frac{CH\iota}{1}$. 
			 \STATE (If assumption-free, use $C_1\sqrt{{\mathrm{Var}_{\widehat{P}^\dagger_{s_h,a_h}}(\widehat{r}^\dagger_h+\widehat{V}_{h+1})\cdot\iota}/{n_{s_h,a_h}}}+\frac{C_2H\cdot\iota}{n_{s_h,a_h}}$ if $n_{s_h,a_h}\geq 1$, o.w. use $0$.)
			\STATE Set $\widehat{Q}^p_h(\cdot,\cdot)\leftarrow \widehat{Q}_h(\cdot,\cdot)-\Gamma_h(\cdot,\cdot)$.   Set $\overline{Q}_h(\cdot,\cdot)\leftarrow \min\{\widehat{Q}^p_h(\cdot,\cdot),H-h+1\}^{+}$.\COMMENT{Pessmistic update}
			\STATE $\forall s_h$, Select $\widehat{\pi}_h(\cdot|s_h)\leftarrow \argmax_{\pi_h}\langle \overline{Q}_h(s_h,\cdot),\pi_h(\cdot|s_h)\rangle$. Set $\widehat{V}_h(s_h)\leftarrow\langle \overline{Q}_h(s_h,\cdot), \widehat{\pi}_h(\cdot|s_h) \rangle$.
			\ENDFOR
			
			\STATE {\bfseries Output: $\{\widehat{\pi}_h\}$. } 
			
		\end{algorithmic}
	}
\end{algorithm}

\section{Intrinsic Offline Reinforcement Learning bound}\label{sec:intrinsic}

Now we go deeper to understand what is the more intrinsic characterization for offline reinforcement learning. From the study of VPVI, penalizing the Q-function by $\widetilde{O}(H/\sqrt{n_{s_h,a_h}})$ is crude as it estimates the confidence width of $\widehat{Q}_h$ in Algorithm~\ref{alg:VPVI} too conservatively therefore loses the accuracy (the bound is suboptimal). This motivates us to use empirical standard deviation instead to create a more adaptive (and also less conservative) Bernstein-type confidence width as the pessimistic penalty:
{\small
 \begin{equation}\label{eqn:pessimistic_pen}
\Gamma_h(s_h,a_h)=\widetilde{O}\bigg[\sqrt{\frac{\mathrm{Var}_{\widehat{P}_{s_h,a_h}}(\widehat{r}_h+\widehat{V}_{h+1})}{n_{s_h,a_h}}}+\frac{H}{n_{s_h,a_h}}\bigg]\;(\text{if} \;n_{s_h,a_h}>0);\;=\widetilde{O}(H)\;(\text{if} \;n_{s_h,a_h}=0).
\end{equation}
}and update $\widehat{Q}_h\leftarrow\widehat{Q}_h-\Gamma_h$. On one hand, $\sqrt{{\mathrm{Var}_{\widehat{P}_{s_h,a_h}}(\widehat{r}_h+\widehat{V}_{h+1})}/{n_{s_h,a_h}}}$ is a ``less pessimistic'' penalty than VPVI due to $\sqrt{\mathrm{Var}_{\widehat{P}}(\widehat{r}_h+\widehat{V}_{h+1})}\leq H$ and critically this design is more data-adaptive since it holds negative view towards the locations with high uncertainties and recommends the locations that we are confident about, as opposed to the online RL (which encourages exploration in the uncertain locations). Such principles are not reflected by the isotropic design in VPVI. On the other hand, it carries the extremely negative view towards fully agnostic locations $\widetilde{O}(H)$ which in turn causes the agent unlikely to choose them. We summarized the this \emph{adaptive pessimistic value iteration} (APVI) into the Algorithm~\ref{alg:APVI}, with $\widehat{P}_h,\widehat{r}_h$ defined in \eqref{eqn:mb_est}. APVI has the following guarantee. A sketch of the analysis is presented in Section~\ref{sec:proof_sketch} and Appendix~\ref{sec:proof_APVI} includes the full proof.

\begin{theorem}[Intrinsic offline RL bound]\label{thm:APVI}
	Under the Assumption~\ref{assum:single_concen}, denote $\bar{d}_m:=\min_{h\in[H]}\{d^\mu_h(s_h,a_h):d^\mu_h(s_h,a_h)>0\}$. For any $0<\delta<1$, there exists absolute constants $c_0,C'>0$, such that when $n>c_0 \cdot 1/\bar{d}_m\cdot\iota$ ($\iota=\log(HSA/\delta)$), with probability $1-\delta$, the output policy $\widehat{\pi}$ of APVI (Algorithm~\ref{alg:APVI}) satisfies ($\widetilde{O}$ hides log factor and higher order terms)
	\begin{equation}\label{eqn:APVI}
	0\leq v^\star-v^{\widehat{\pi}}\leq C'\sum_{h=1}^H\sum_{(s_h,a_h)\in\mathcal{C}_h}d^{\pi^\star}_h(s_h,a_h)\cdot\sqrt{\frac{\mathrm{Var}_{P_{s_h,a_h}}(r_h+V^\star_{h+1})\cdot\iota}{ n\cdot d^\mu_h{(s_h,a_h)}}}+\widetilde{O}\left(\frac{H^3}{n\cdot \bar{d}_m}\right)
	\end{equation}
\end{theorem}
\begin{remark}
	APVI (Algorithm~\ref{alg:APVI}) can also be called \textbf{LCBVI-Bernstein} as it creates the offline counterpart of UCBVI in \cite{azar2017minimax}. However, to highlight that the resulting bound fully adapts to the specific system structure, we use the word ``adaptive'' instead.
\end{remark}

APVI makes significant improvements in a lot of aspects. First and foremost, the dominate term is fully expressed by the system quantities that admits no explicit dependence on $H,S,A$. To the best of our knowledge, this is the first offline RL bound that concretely depicts the interrelations within the problem when the problem instance is a tuple $(M,\pi^\star,\mu)$: an MDP $M$ (coupled with the optimal policy $\pi^\star$) with the data rolling from an offline logging policy $\mu$. As we will discuss later, this result indicates (nearly) all the optimal worst-case non-adaptive bounds (and clearly also the VPVI) under their respective regimes / assumptions. Thus, \eqref{eqn:APVI} is generic. More interestingly, Theorem~\ref{thm:APVI} caters to the specific MDP structures and adaptively yields improved sample complexities (\emph{e.g.} faster convergence in deterministic systems) that existing works cannot imply. Such features are crucial as it helps us to understand what type of problems are harder / easier than others, and even more, in a \emph{quantitative} way. Last but not least, to illustrate this bound exhibits the intrinsic nature of offline RL, we prove a \emph{per-instance dependent} information-theoretical lower bound that shares a similar formulation. The proof of Theorem~\ref{thm:adaptive_lower_bound} can be found in Appendix~\ref{sec:proof_lower_bound}.

\begin{theorem}[Instance-dependent information theoretical offline lower bound]\label{thm:adaptive_lower_bound}
	Denote $\mathcal{G}:=\{(\mu,M): \exists \pi^\star\;s.t. \; d^\mu_h(s,a)>0\;\text{if}\;d^{\pi^\star}_h(s,a)>0\}$. Fix an instance $\mathcal{P}=(\mu,M)\in\mathcal{G}$. Let $\mathcal{D}$ consists of $n$ episodes and define $\xi=\sup_{h,s_h,a_h,s_{h+1}, d^\mu_h(s_h,a_h)\cdot \Var_{P_{s_h,a_h}}(V_{h+1}^\star)>0}\frac{P_h(s_{h+1}|s_h,a_h)\left(V_{h+1}^\star(s_{h+1})-\E_{P_{s_h,a_h}}[V_{h+1}^\star]\right)}{\sqrt{2\cdot  d^\mu_h(s_h,a_h)\cdot \Var_{P_{s_h,a_h}}(V_{h+1}^\star)}}$. Let $\widehat{\pi}$ to be the output of any algorithm.  Define the {local non-asymptotic minimax risk} as 
	\begin{equation}\label{eqn:local_risk}
	\mathfrak{R}_{n}(\mathcal{P}):=\sup_{\mathcal{P}'\in\mathcal{G}}\inf_{\widehat{\pi}}\max_{\mathcal{Q}\in\{\mathcal{P},\mathcal{P}'\}}\sqrt{n}\cdot\E_{\mathcal{Q}}\left[v^\star(\mathcal{Q})-v^{\widehat{\pi}}\right]
	\end{equation}
	where $v^\star(\mathcal{Q})$ denotes the optimal value under the instance $\mathcal{Q}$. Then there exists universal constants $c_0,p,C>0$, such that if $n\geq c_0{H^6\xi^4}/{(\sum_{h=1}^H \sum_{s_h,a_h}d^{\pi^\star}_h(s_h,a_h)\sqrt{\frac{\Var_{P_{s_h,a_h}}(V_{h+1}^\star)}{ \zeta \cdot d^\mu_h(s_h,a_h)}})^2}$, with constant probability $p>0$, 
	Then we have (here $\zeta =H/\bar{d}_m$):
	\begin{equation}\label{eqn:lower}
	\mathfrak{R}_{n}(\mathcal{P})\geq C\cdot {\sum_{h=1}^H\sum_{(s_h,a_h)\in\mathcal{C}_h}d^{\pi^\star}_h(s_h,a_h)\cdot\sqrt{\frac{\mathrm{Var}_{P_{s_h,a_h}}(r_h+V^\star_{h+1})}{\zeta \cdot d^\mu_h{(s_h,a_h)}}}},
	\end{equation}
	where $\mathcal{P}=(\mu,M)$ and $M=(\mathcal{S}, \mathcal{A}, P, r, H, d_1)$.
\end{theorem}

The interpretation of Theorem~\ref{thm:adaptive_lower_bound} is: for any instance $\mathcal{P}$, learning requires \eqref{eqn:lower} (divided by $1/\sqrt{n}$) for any algorithm. Note this notion is significantly stronger than the previous minimax offline lower bounds \citep{yin2021near,rashidinejad2021bridging,xie2021policy,jin2020pessimism} (where they only select a particular family of hard problems), therefore, their lower bounds in general do not hold for individual instances. 

The quantity \eqref{eqn:intrinsic} nearly-matches the per-instance lower bound \eqref{eqn:lower} (they deviate by a factor of $\zeta=H/\bar{d}_m$ due to the technical reason) and, in addition, we provide a matching minimax lower bound in Appendix~\ref{sec:minimax_lower}. These results certify Theorem~\ref{thm:APVI} is not only adaptive but also near-optimal. Hence, we call the quantity $\sum_{h=1}^H\sum_{(s_h,a_h)\in\mathcal{C}_h}d^{\pi^\star}_h(s_h,a_h)\cdot\sqrt{\frac{\mathrm{Var}_{P_{s_h,a_h}}(r_h+V^\star_{h+1})}{ n\cdot d^\mu_h{(s_h,a_h)}}}$ \emph{intrinsic offline reinforcement learning bound}. In the sequel, we provide thorough discussions to explain the intrinsic bound embraces the fundamental challenges in offline RL and the strong adaptivity. The detailed technical derivations that are missing in Section~\ref{subsec:one}-\ref{subsec:three} are deferred to Appendix~\ref{sec:missing_dev}.

\begin{figure}[H]
	\centering     
	\includegraphics[width=115mm]{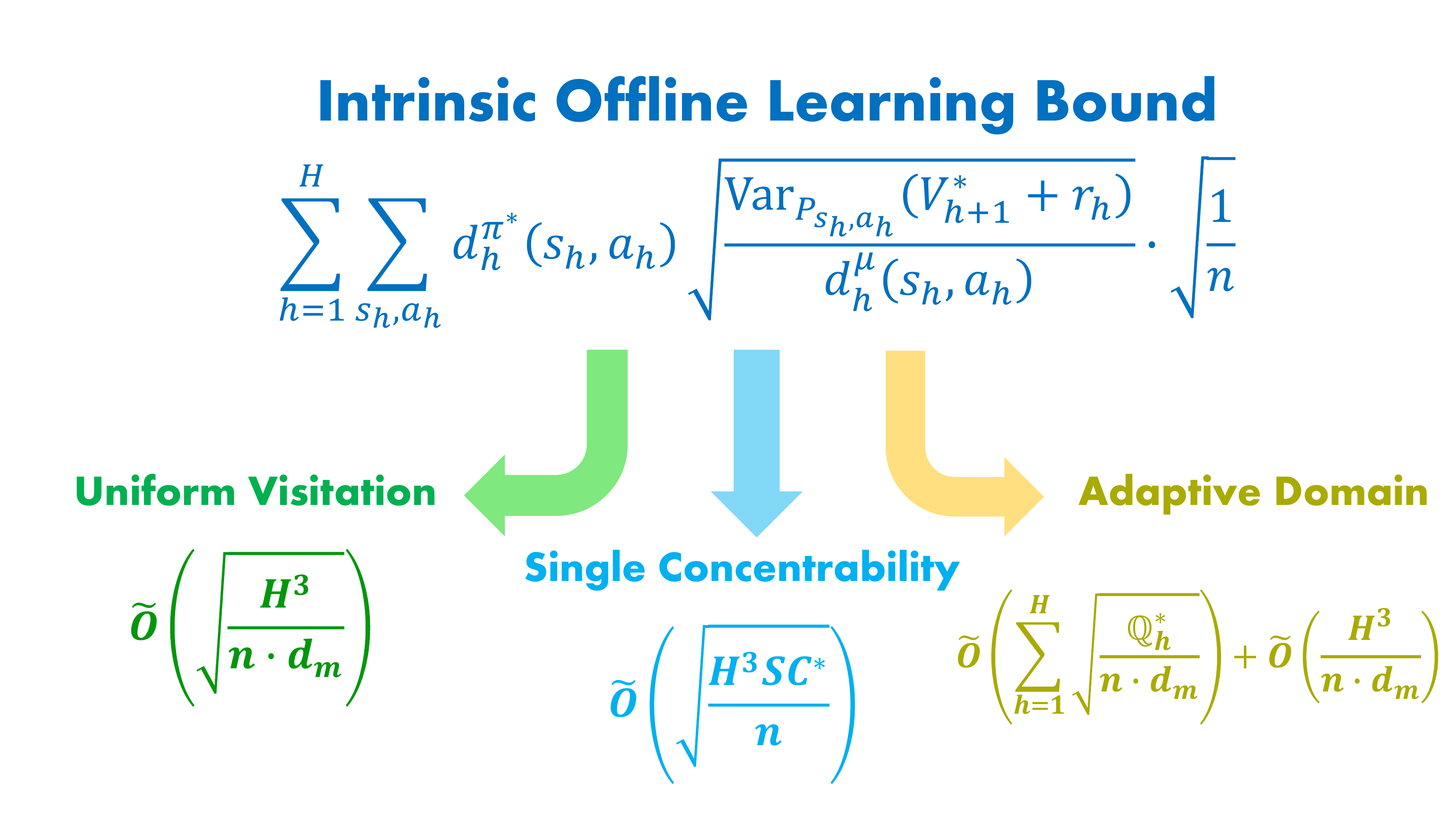}
	\caption{A visualization on how intrinsic learning bound subsumes existing best-known results: uniform visitation, single concentrability (partial coverage) and adaptive domain.}
	\label{fig:main}
\end{figure}

\subsection{Optimality under Uniform data-coverage assumption}\label{subsec:one}

Under the uniform exploration Assumption~\ref{assum:uniform} with parameter $d_m:=\min_{h,s_h,a_h} d_h^\mu (s_h,a_h) > 0$, \cite{yin2021near} analyzes the model-based plug-in approach and obtains the optimal sample complexity $\widetilde{O}(H^3/d_m\epsilon^2)$ and shows $\Omega(H^3/d_m\epsilon^2)$ is also the lower bound. Indeed, this rate can be directly implied by the intrinsic RL bound via \emph{Cauchy inequality} and \emph{the Sum of Total Variance} (Lemma~\ref{lem:H3toH2}):\footnote{Here $\odot$ denotes element-wise multiplication. Also note under \ref{assum:uniform}, our $\bar{d}_m=d_m$.}
{\small
\begin{equation}\label{eqn:inter_derivation}
\begin{aligned}
&\sum_{h=1}^H\langle d^{\pi^\star}_h(\cdot),\sqrt{\frac{\mathrm{Var}_{P_{(\cdot)}}(r_h+V^\star_{h+1})}{ n\cdot d^\mu_h{(\cdot)}}}\rangle = \sum_{h=1}^H\langle \sqrt{d^{\pi^\star}_h(\cdot)},\sqrt{\frac{d^{\pi^\star}_h(\cdot)\odot\mathrm{Var}_{P_{(\cdot)}}(r_h+V^\star_{h+1})}{ n\cdot d_m}}\rangle\\
\leq &\sum_{h=1}^H\norm{ \sqrt{d^{\pi^\star}_h(\cdot)}}_2\norm{\sqrt{\frac{d^{\pi^\star}_h(\cdot)\odot\mathrm{Var}_{P_{(\cdot)}}(r_h+V^\star_{h+1})}{ n\cdot d_m}}}_2\leq \sqrt{\frac{H\cdot \mathrm{Var}_{\pi^\star}(\sum_{h=1}^{H} r_{h})}{n\cdot d_m}}\leq \sqrt{\frac{H^3}{n\cdot d_m}}
\end{aligned}
\end{equation}
}which translates to $\widetilde{O}(H^3/d_m\epsilon^2)$  complexity. Our result maintains the optimal worst-case guarantee when $\mu$ has the uniform data-coverage:

\begin{proposition}
	Under Assumption~\ref{assum:uniform} and apply Theorem~\ref{thm:APVI}, APVI achieves the sample complexity of minimax-rate $\widetilde{O}(H^3/d_m\epsilon^2)$ (Theorem~4.1 and Theorem~G.2 in \cite{yin2021near}).
\end{proposition}

\begin{remark}\label{remark:time-variant}
	We believe if the MDP is time-invariant, then by a modified construction of $\widehat{P}$, $\widehat{r}$ in \eqref{eqn:mb_est} our result will imply the minimax-rate of $\widetilde{O}(H^2/d_m\epsilon^2)$ as achieved in \cite{yin2021nearoptimal}. We include this discussion in Appendix~\ref{sec:missing_dev}.
\end{remark}

\subsection{Bounded sum of total rewards and the Horizon-Free case}\label{subsec:btr}
There is another thread of studies that follow the bounded sum of total rewards assumption: \emph{i.e.} $r_h\geq 0$, $\sum_{h=1}^H r_h\in [0,1]$ \citep{krishnamurthy2016pac,jiang2017contextual,zhang2020reinforcement}. Such a setting is much weaker than the uniform bounded instantaneous reward condition, as explained in \cite{jiang2018open}. In offline RL, \cite{ren2021nearly} derives the nearly horizon-free worst case bound $\widetilde{O}(\sqrt{1/nd_m})$ for the time-invariant MDPs, under the Assumption~\ref{assum:uniform}. As a comparison, our Theorem~\ref{thm:APVI} achieves the following guarantee for the time-varying (non-stationary) MDPs.

\begin{proposition}
	Assume $r_h\geq 0$, $\sum_{h=1}^H r_h\leq 1$. Then in the time-varying case AVPI (Theorem~\ref{thm:APVI}) outputs a policy $\widehat{\pi}$ such that the suboptimality gap $v^\star-v^{\widehat{\pi}}$ is bounded by $\widetilde{O}(\sqrt{H/nd_m})$ with high probability under the Assumption~\ref{assum:uniform}. 
\end{proposition}   

The derivation is straightforward by using $\mathrm{Var}_{\pi^\star}(\sum_{h=1}^{H} r_{h})\leq 1$ in \eqref{eqn:inter_derivation}. This proposition is interesting since it indicates when the MDP is non-stationary, $\widetilde{O}(H/d_m\epsilon^2)$ is required in the worst case even under $\sum_{h=1}^H r_h\leq 1$.\footnote{Suppose in this case we can achieve $\widetilde{O}(1/d_m\epsilon^2)$ just like \cite{ren2021nearly}, then by a rescaling we obtain the $\widetilde{O}(H^2/d_m\epsilon^2)$ under the usual $0\leq r_h\leq 1$ assumption which violates the $\Omega(H^3/d_m\epsilon^2)$ lower bound.} The extra $H$ factor resembles the challenge that we have $H$ transitions ($P_1,\ldots,P_H$) to learn, as opposed to the bandit-type $1/d_m\epsilon^2$ result due to there is only one $P$ throughout (time-invariant). This reveals that one hardness in solving the MDP is in proportion to the number of different transition kernels within the MDP. Such a finding could help researchers understand the special settings like \emph{low switching cost in transitions} \citep{bai2019provably} or \emph{non-stationarity} \citep{cheung2020reinforcement}.

\subsection{Optimality with Single Concentrability}\label{subsec:two}

In the finite horizon discounted setting, \cite{rashidinejad2021bridging} proposes the single policy concentrability assumption which is defined as $C^\star:=\max_{h,s,a}\frac{d^{\pi^\star}_h(s,a)}{d^\mu_h(s,a)}<\infty$ in the current episodic non-stationary MDP setting. As discussed in Appendix~\ref{sec:dis_VPVI}, their lower bound translates to $\Omega(\sqrt{\frac{H^3SC^\star}{n}})$ and their VI-LCB algorithm yields $\widetilde{O}(\sqrt{\frac{H^5SC^\star}{n}})$ suboptimality gap in $H$-horizon case. Since single policy concentrability is strictly weaker than its uniform version (Assumption~\ref{assum:concen}), we only discuss this set up. In particular, we have the following implication from our Theorem~\ref{thm:APVI} (whose derivation can be found in Appendix~\ref{sec:missing_dev}):

\begin{proposition}
	Let $\pi^\star$ be a deterministic policy such that $C^\star:=\max_{h,s,a}\frac{d^{\pi^\star}_h(s,a)}{d^\mu_h(s,a)}<\infty$. Then by Theorem~\ref{thm:APVI}, with high probability the output policy of APVI satisfies the suboptimality gap $\widetilde{O}(\sqrt{\frac{H^3SC^\star}{n}})$ in the time-varying (non-stationary) MDPs. 
\end{proposition}
This can computed similar to \eqref{eqn:inter_derivation} except we use $\frac{d^{\pi^\star}_h(s,a)}{d^\mu_h(s,a)}\leq C^\star$. Our implication improves the VI-LCB by the factor $H^2$ (in terms of sample complexity) and is optimal (recover the concurrent \cite{xie2021policy}). Qualitatively, single concentrability is the same as Assumption~\ref{assum:single_concen}, but the use of $C^\star$ makes the bound highly problem independent and limits the adaptivity. Problem dependent bound is a more interesting domain as it tailors to each MDP separately. We discuss it now.

\subsection{Problem dependent domain}\label{subsec:three}

We define the \emph{pre-step environmental norm} (the finite horizon counterpart of \cite{maillard2014hard}) as: $\mathbb{Q}^\star_h=\max_{s_h,a_h}\mathrm{Var}_{P_{s_h,a_h}}(r_h+V^\star_{h+1})$ for all $h\in[H]$, and relax the total sum of rewards to be bounded by any arbitrary value $\mathcal{B}$ (\emph{i.e.} $\sum_{h=1}^H r_h\leq \mathcal{B}$), then Theorem~\ref{thm:APVI} implies:
\begin{proposition}\label{prop}
	Under Assumption~\ref{assum:uniform}, with high probability, subopmality of AVPI is bounded by 
	{
	\[
	\min\left\{\widetilde{O}\big(\sum_{h=1}^H\sqrt{\frac{\mathbb{Q}^\star_h}{n\bar{d}_m}}\big),\widetilde{O}\big(\sqrt{\frac{H\cdot \mathcal{B}^2}{n\bar{d}_m}}\big)\right\}+\widetilde{O}(\frac{H^3}{n\bar{d}_m}).
	\]}  
\end{proposition} 
Such a result mirrors the online version of the tight problem-dependent bound \cite{zanette2019tighter} but with a more general \emph{pre-step environmental norm} for the non-stationary MDPs.\footnote{\cite{zanette2019tighter} uses the maximal version by maximizing over $h$.} For the problem instances with either small $\mathcal{B}$ or small $\mathbb{Q}_h^\star$, our result yields much better performances, as discussed in the following.

\textbf{Deterministic systems.} For many practical applications of interest, the systems are equipped with low stochasticity, \emph{e.g.} robotics, or even deterministic dynamics, \emph{e.g.} the game of GO. In those scenarios, the agent needs less experience for each state-action therefore the learning procedure could be much faster. In particular, when the system is fully deterministic (in both transitions and rewards) then $\mathbb{Q}^\star_h=0$ for all $h$. This enables a faster convergence rate of order $\frac{H^3}{n\bar{d}_m}$ and significantly improves over the existing non-adaptive results that have order $\frac{1}{\sqrt{n}}$. The convergence rate $\frac{1}{n}$ matches \cite{wen2013efficient} by translating their constant (in $T$) regret into the PAC bound.

\textbf{Partially deterministic systems.} Practical worlds are complicated and we could sometimes have a mixture model which contains both deterministic and stochastic steps. In those scenarios, the main complexity is decided by the number of stochastic stages: suppose there are $t$ stochastic $P_h,r_h$'s and $H-t$ deterministic $P_{h'},r_{h'}$'s, then completing the offline learning guarantees {\small$t\cdot\sqrt{{\max Q^\star_h}/{n \bar{d}_m}}$} suboptimality gap, which could be much smaller than {\small$H\cdot\sqrt{{\max Q^\star_h}/{n \bar{d}_m}}$} when $t\ll H$. 

\textbf{Fast mixing domains.} Consider a class of highly mixing non-stationary MDPs (a variant of \cite{zanette2018problem}) that satisfies the transition $P_h(\cdot|s_h,a_h):=\nu_h(\cdot)$ depends on neither the state $s_h$ nor the action $a_h$. Define $\bar{s}_{t} := \arg \max V_{t}^{\star}(s)$ and $\underline{s}_{t} := \arg \max V_{t}^{\star}(s)$. Also, denote $\mathrm{rng}V^\star_h$ to be the range of $V^\star_h$.
In such cases, Bellman optimality equations have the form
{
\[
V_{h}^{\star}\left(\bar{s}_{h}\right)=\max _{a}\left(r_h\left(\bar{s}_{h}, a\right)+\nu_h^{\top} V_{h+1}^{\star}\right),\;\;V_{h}^{\star}\left(\underline{s}_{h}\right)=\max _{a}\left(r_h\left(\underline{s}_{h}, a\right)+\nu_h^{\top} V_{h+1}^{\star}\right),
\]
}which yields $\mathrm{rng}V^\star_h=V_{h}^{\star}\left(\bar{s}_{h}\right)-V_{h}^{\star}\left(\underline{s}_{h}\right)=\max_ar_h\left(\bar{s}_{h}, a\right)-\min_ar_h\left(\underline{s}_{h}, a\right)\leq 1$, and this in turn gives $\mathbb{Q}_h^\star\leq 1+(\mathrm{rng}V^\star_h)^2=2$. As a result, the suboptimality is bounded by $\widetilde{O}(\sqrt{H^2/nd_m})$ in the worst case. This result reveals, although this is a family of stochastic non-stationary MDPs, but it is only as hard as the family of stationary MDPs in the minimax sense ($\Omega(H^2/d_m\epsilon^2)$).

\textbf{Tabular contextual bandits.} Our result also implies $\widetilde{O}(\sum_{x_1,a_1}d^{\pi^\star}_1(x_1,a_1)\sqrt{\frac{\mathrm{Var}{(r_1)}}{n\cdot d^\mu_1(x_1,a_1)}})$ gap for the \emph{offline tabular contextual bandit} problem and improves to $\widetilde{O}(1/nd_m)$ when the reward is deterministic. In either cases, the result is optimal and this is due to: when $r_1$ is deterministic, the agent only needs one sample at every location (see \cite{bubeck2012regret} for a survey).

\section{Towards Assumption-Free Offline RL}\label{sec:assumption_free}
While assumption~\ref{assum:single_concen} is (arguably) the weakest assumption for correctly learning the optimal value, for the real-world applications even this might not be guaranteed. Can we still learn something meaningful? In this section, we consider this most general setting where the behavior policy $\mu$ can be arbitrary. In this case, $\mu$ might not cover any optimal policy $\pi^\star$ (\emph{i.e.} there might be high reward location $(s,a)$ that $\mu$ can never visit, \emph{e.g.} in the extreme case where a clumsy doctor only uses one treatment all the time), and, irrelevant to the number of episode $n$, a constant suboptimality gap needs to be suffered. To tackle this problem, we create a fictitious augmented MDP $M^\dagger$ that can help characterize the discrepancy of the values between the original MDP ${M}$ and the estimated MDP $\widehat{M}^\dagger$. In particular, $M^\dagger$ is negative towards agnostic state-actions $s_h,a_h$ by setting $r^\dagger_h=0 $ and transitions to an absorbing state $s^\dagger_{h+1}$. 

\textbf{Pessimistic augmented MDP.} $M^\dagger$ is defined with one extra state $s_h^\dagger$ for all $h\in\{2,\ldots,H+1\}$ with the augmented state space $\mathcal{S}^\dagger=\mathcal{S}\cup\{s^\dagger_h\}$. The transition and the reward are defined as follows: 
{\small
	\[
	P^{\dagger}_h(\cdot \mid s_h, a_h)=\left\{\begin{array}{ll}
	P_h(\cdot \mid s_h, a_h), \;n_{s_h,a_h}>0, \\
	\delta_{s^{\dagger}_{h+1}}, \; s_h=s_h^{\dagger} \text { or } n_{s_h,a_h}=0.
	\end{array} \;\; r^{\dagger}( s_h, a_h)=\left\{\begin{array}{ll}
	r(s_h, a_h), \; n_{s_h,a_h}>0, \\
	0, \; s_h=s^{\dagger}_{h} \text { or } n_{s_h,a_h}=0.
	\end{array}\right.\right.
	\]}here $\delta_s$ is the Dirac measure and we denote $V^{\dagger \pi}_h$ and $v^{\dagger\pi}$ to be the values under $M^\dagger$. $\widehat{M}^\dagger$ is the empirical counterpart of $M^\dagger$ with $\widehat{P}$, $\widehat{r}$ (the same as \eqref{eqn:mb_est}) replacing $P$, $r$. By Algorithm~\ref{alg:APVI}, we have

\begin{theorem}[Assumption-free offline reinforcement learning]\label{thm:AFRL}
	Let us make no assumption for $\mu$ and still denote $\bar{d}_m:=\min_{h\in[H]}\{d^\mu_h(s_h,a_h):d^\mu_h(s_h,a_h)>0\}$. For any $0<\delta<1$, there exists absolute constants $c_0,C'>0$, such that when $n>c_0 \cdot 1/\bar{d}_m\cdot\iota$ ($\iota=\log(HSA/\delta)$), with probability $1-\delta$, the output policy $\widehat{\pi}$ of APVI satisfies (recall $\mathcal{C}_h:=\{(s_h,a_h):d^\mu_h(s_h,a_h)>0\}$)
	{\small
	\begin{equation}\label{eqn:AFRL}
	 v^\star-v^{\widehat{\pi}}\leq \sum_{h=2}^{H+1}d^{\dagger\pi^\star}_h(s^\dagger_h)+C'\sum_{h=1}^H\sum_{(s_h,a_h)\in\mathcal{C}_h}d^{\dagger\pi^\star}_h(s_h,a_h)\cdot\sqrt{\frac{\mathrm{Var}_{P^\dagger_{s_h,a_h}}(r^\dagger_h+V^{\dagger\pi^\star}_{h+1})\cdot\iota}{ n\cdot d^\mu_h{(s_h,a_h)}}}+\widetilde{O}\left(\frac{H^3}{n\bar{d}_m}\right),
	\end{equation}
}where $d^{\dagger\pi^\star}_h(s_h,a_h)\leq d^{\pi^\star}_h(s_h,a_h),V^{\dagger\pi^\star}_h(s_h)\leq V^\star_h(s_h)$ for all $s_h,a_h\in\mathcal{S}\times\mathcal{A}$, and for all $h\in[H]$, $d^{\dagger\pi^\star}_h(s^\dagger_h)=\sum_{t=1}^{h-1}\sum_{(s_t,a_t)\in\mathcal{S}\times\mathcal{A}\backslash\mathcal{C}_t}d^{\dagger\pi^\star}_t(s_t,a_t)$. The proof is in Appendix~\ref{sec:proof_af}.
\end{theorem}

\textbf{Take-aways of Theorem~\ref{thm:AFRL}.} In $M^\dagger$, there is no agnostic location any more since the original unknown spaces now all have \emph{known} deterministic transitions to $s^\dagger$ in $M^\dagger$. At a price, the algorithm has to suffer the constant suboptimality $\sum_{h=2}^{H+1}d^{\dagger\pi^\star}_h(s^\dagger_h)$ due to no data in the region. The quantity $\sum_{h=2}^{H+1}d^{\dagger\pi^\star}_h(s^\dagger_h)$ helps characterize the hardness when nothing is assumed about $\mu$: it is always less than $H$ (cannot suffer more than $H$ suboptimality); under Assumption~\ref{assum:uniform}, it is $0$ since $M^\dagger=M$ with high probability (by Chernoff bound) and this causes $\mathcal{S}\times\mathcal{A}\backslash \mathcal{C}_h=\emptyset$; under Assumption~\ref{assum:single_concen}, it is also $0$ and \ref{thm:AFRL} reduces to Theorem~\ref{thm:APVI} (see Appendix~\ref{sec:proof_APVI}).

\subsection{Assumption Free vs Without Great Coverage (Partial Coverage)}

Recently there is a surge of studies that aim at weakening the assumptions of provable offline / batch RL. Those learning bounds are derived (mostly) under the insufficient data coverage assumptions. One type of works consider the assumption \emph{without great coverage} (or partial coverage): \cite{chang2021mitigating,uehara2021pessimistic} assume $\max_{s,a}d^{\pi_e}(s,a)/\mu(s,a)<\infty$ where $\pi_e$ is either an expert policy or a policy of great quality and they further compete against with this policy $\pi_e$. Those assumptions are similar to \ref{assum:single_concen} and therefore are stronger than the assumption-free RL we considered in \ref{thm:AFRL}.

In addition, there are other studies that apply to the case where $\mu$ can be arbitrary: \cite{liu2020provably} considers the behavior policy with insufficient coverage probability $\epsilon_\zeta$ (see their Definition~1), and they end up with the constant suboptimality gap $\frac{V_{\max}\epsilon_\zeta}{1-\gamma}$ (their Theorem~1), when the insufficient coverage probability $\epsilon_\zeta>0$, this gap has order $(1-\gamma)^{-2}$, which is larger in order than the biggest possible suboptimality gap $(1-\gamma)^{-1}$ therefore unable to characterize the essential statistical gap over the region that can never be visited by the behavior policy (and this happens similarly in \cite{kidambi2020morel}, see their Theorem~1); \cite{jin2020pessimism} derive the nice assumption-free result via regularization and their bound can incur $O(H^2)$ constant gap when there is at least one $(s_h,a_h)$ cannot be obtained by $\mu$ for all $h\in[H]$ (\emph{i.e.} replacing $nd^\mu_h(s_h,a_h)$ by $1$ in \eqref{eqn:VPVI}). The concurrent work \cite{xie2021bellman} provides a better characterization (and they call it \emph{off-support error}) with roughly $\frac{1}{1-\gamma}  \sum_{(s, a) \in \mathcal{S} \times \mathcal{A}}\left(d_{\pi} \backslash \nu\right)(s, a)\left[\Delta f_{\pi}(s, a)-\left(\mathcal{T}^{\pi} \Delta f_{\pi}\right)(s, a)\right]$, however, in the worst case $\Delta f_{\pi}(s, a)-\left(\mathcal{T}^{\pi} \Delta f_{\pi}\right)(s, a)$ might be large (which depends on the quality (assumption) of the function approximation class). 

In contrast, our $\sum_{h=2}^{H+1}d^{\dagger\pi^\star}_h(s^\dagger_h)$ quantity (with $d^{\dagger\pi^\star}_h(s^\dagger_h)=\sum_{t=1}^{h-1}\sum_{(s_t,a_t)\in\mathcal{S}\times\mathcal{A}\backslash\mathcal{C}_t}d^{\dagger\pi^\star}_t(s_t,a_t)\leq 1$) describes the ``must-suffer'' gap in a more precise way by absorbing all the agnostic probabilities into $s^\dagger$ and it is always bounded between $0$ and $H$. It reduces to $0$ when $\pi^\star$ is covered. The gap is always of order $H$ (as opposed to $O(H^2)$).

\subsection{The statistical limits for Offline Learning and OPE in tabular RL}

\begin{table*}[h]\label{table1}
	\centering\resizebox{\columnwidth}{!}{
		\begin{tabular}{ |c|c|c|c| } 
			\hline
			Task & Dominate Bound & Type  \\
			\hline 
			Offline policy learning  & $\sum_{h=1}^H\sum_{s_h,a_h}d^{\pi^\star}_h(s_h,a_h)\sqrt{\frac{\mathrm{Var}_{P_{s_h,a_h}}{(V^\star_{h+1}+r_h)}}{d^\mu_h(s_h,a_h)}}\sqrt{\frac{1}{n}}$ & Instance-dependent (Theorem~\ref{thm:APVI},\ref{thm:adaptive_lower_bound})  \\ 
			\hline
			OPE $(|v^\pi-\hat{v}^\pi|)$ & $\sqrt{\frac{1}{n} \sum_{h=0}^{H} \sum_{s_{h}, a_{h}} \frac{d_{h}^{\pi}\left(s_{h},a_h\right)^{2}}{d_{h}^{\mu}\left(s_{h},a_h\right)}  \operatorname{Var}_{P_{s_h,a_h}}\left(V_{h+1}^{\pi}+r_{h}\right) }$ & Upper bound, Cramer-Rao lower bound  \\ 
			\hline
		\end{tabular}
	}
	\caption{Showing the statistical optimalities for offline policy learning ($v^\star-v^{\hat{\pi}}$) and offline policy evaluation (OPE) ($|v^\pi-\hat{v}^\pi|$) for the non-stationary tabular MDPs. The upper bound of OPE comes from \cite{yin2020asymptotically,duan2020minimax} and the Cramer-Rao lower bound comes from \cite{jiang2016doubly}.}
\end{table*}

Table~\ref{table1} shows the statistical optimality for \emph{offline policy learning} and \emph{offline policy evaluation} (OPE) in the non-stationary tabular MDPs. By Cauchy-Schwartz inequality, it can be checked that the rate between the two bounds (roughly) deviate by a factor of $H$ (in terms of sample complexity), and this reveals that offline learning is inherently harder than OPE from the statistical aspect.

%% file: sections/proof_overview.tex
We only briefly sketch the key proving ideas in Section~\ref{sec:intrinsic}. We provide an extended proof overview in Appendix~\ref{sec:e_p_o}.  
Our analysis of the intrinsic learning bound in Section~\ref{sec:intrinsic} leverage the key design feature of APVI that $\widehat{V}_{h+1}$ only depends on the transition data from time $h+1$ to $H$ while $\widehat{P}_h$ only uses transition pairs at time $h$. This enables concentration inequalities due the \emph{conditional} independence. To cater for the data-adaptive bonus \eqref{eqn:pessimistic_pen}, we use \emph{Empirical} Bernstein inequality to get {\small$(\widehat{P}_h-P_h)\widehat{V}_{h+1}\lesssim \sqrt{{\mathrm{Var}_{\hat{P}}(\widehat{V}_{h+1})}/{n_{s_h,a_h}}}$}. Especially, to recover the {\small$\sqrt{\mathrm{Var}_P(V^\star_{h+1})}$} structure to we use a self-bounding reduction as follows. First, {\small$\sqrt{{\mathrm{Var}_{\hat{P}}(\widehat{V}_{h+1})}}-\sqrt{\mathrm{Var}_P(\widehat{V}_{h+1})}\lesssim H/\sqrt{n\bar{d}_m}$} and {\small$\sqrt{\mathrm{Var}_P(\widehat{V}_{h+1})}-\sqrt{\mathrm{Var}_P({V}^\star_{h+1})}\leq ||\widehat{V}_{h+1}-V^\star_{h+1}||_\infty$}. Next, we use \eqref{eqn:VPVI} as the intermediate step to crude bounding  {\small$||\widehat{V}_{h+1}-V^\star_{h+1}||_\infty\lesssim H^2/\sqrt{n\bar{d}_m}$} (where ``the use of \eqref{eqn:VPVI}'' is the more intricate self-bounding Lemma~\ref{lem:self_bound} in the actual proof) and this yields the desired structure of {\small$\sqrt{\mathrm{Var}_P({V}^\star_{h+1})}+H^2/\sqrt{n\bar{d}_m}$}. Lastly, we can combine this with \emph{the extended value difference lemma} in \cite{cai2020provably} to bound $V_1^\star-\widehat{V}_1$ and leverage the pessimistic design for bounding $\widehat{V}_1-V_1^{\widehat{\pi}}$.

For the per-instance lower bound, similar to \cite{khamaru2020temporal}, we reduce the problem from $\mathfrak{R}_{n}(\mathcal{P})$ to the two point testing problem and construct a problem-dependent local instance {\small$P'_h(s_{h+1}|s_h,a_h)=P_h(s_{h+1}|s_h,a_h)+\frac{P_h(s_{h+1}|s_h,a_h)\left(V_{h+1}^\star(s_{h+1})-\E_{P_{s_h,a_h}}[V_{h+1}^\star]\right)}{8\sqrt{\zeta\cdot n_{s_h,a_h}\cdot \Var_{P_{s_h,a_h}}(V_{h+1}^\star)}}$}. The design with the subtraction of ``the baseline'' $\E_{P_{s_h,a_h}}[V_{h+1}^\star]$ is the key to make sure $P'$ center around the instance $P$.


%% file: sections/proofs.tex

 \begin{algorithm}[H]
	\caption{Vanilla Pessimistic Value Iteration}
	\label{alg:VPVI}
	\small{
		\begin{algorithmic}[1]
			\STATE {\bfseries Input:} Offline dataset $\mathcal{D}=\{(s_h^\tau,a_h^\tau,r_h^\tau,s_{h+1}^\tau)\}_{\tau,h=1}^{n,H}$. Absolute Constant $C$, failure probability $\delta$.
			
			\STATE {\bfseries Initialization:} Set $\widehat{V}_{H+1}(\cdot)\leftarrow 0$.

			\FOR{time $h=H,H-1,\ldots,1$}
			\STATE Set $\widehat{Q}_h(\cdot,\cdot)\leftarrow {\widehat{r}_h(\cdot,\cdot)}+(\widehat{P}_{h}\cdot \widehat{V}_{h+1})(\cdot,\cdot)$
			\STATE $\forall s_h,a_h$, set $\Gamma_h(s_h,a_h)=\frac{CH\log(HSA/\delta)}{\sqrt{n_{s_h,a_h}}}$ if $n_{s_h,a_h}\geq 1$, o.w. set to $\frac{CH\log(HSA/\delta)}{1}$.
			\STATE Set $\widehat{Q}^p_h(\cdot,\cdot)\leftarrow \widehat{Q}_h(\cdot,\cdot)-\Gamma_h(\cdot,\cdot)$.\COMMENT{Pessmistic update}
			\STATE Set $\overline{Q}_h(\cdot,\cdot)\leftarrow \min\{\widehat{Q}^p_h(\cdot,\cdot),H-h+1\}^{+}$.
			\STATE Select $\widehat{\pi}_h(\cdot|s_h)\leftarrow \argmax_{\pi_h}\langle \overline{Q}_h(s_h,\cdot),\pi_h(\cdot|s_h)\rangle$, $\forall s_h$.
			\STATE Set $\widehat{V}_h(s_h)\leftarrow\langle \overline{Q}_h(s_h,\cdot), \widehat{\pi}_h(\cdot|s_h) \rangle$,  $\forall s_h$. 
			\ENDFOR
			
			\STATE {\bfseries Output: $\{\widehat{\pi}_h\}$. } 
			
		\end{algorithmic}
	}
\end{algorithm}

\section{Extended Proof Overview and Some Notations}\label{sec:e_p_o}

Our analysis of the intrinsic learning bound in Section~\ref{sec:intrinsic} leverage the key design feature of APVI that $\widehat{V}_{h+1}$ only depends on the transition data from time $h+1$ to $H$ while $\widehat{P}_h$ only uses transition pairs at time $h$. This enables concentration inequalities due the \emph{conditional} independence.\footnote{This trick is also leveraged in \cite{yin2021near}, but they consider the empirical optimal value $\widehat{V}^{\widehat{\pi}^\star}$ instead.} To cater for the data-adaptive bonus \eqref{eqn:pessimistic_pen}, we need to use \emph{Empirical} Bernstein inequality to get {\small$(\widehat{P}_h-P_h)\widehat{V}_{h+1}\lesssim \sqrt{{\mathrm{Var}_{\hat{P}}(\widehat{V}_{h+1})}/{n_{s_h,a_h}}}$}. Especially, to recover the {\small$\sqrt{\mathrm{Var}_P(V^\star_{h+1})}$} structure to we use a self-bounding reduction as follows. First, {\small$\sqrt{{\mathrm{Var}_{\hat{P}}(\widehat{V}_{h+1})}}-\sqrt{\mathrm{Var}_P(\widehat{V}_{h+1})}\lesssim H/\sqrt{n\bar{d}_m}$} and {\small$\sqrt{\mathrm{Var}_P(\widehat{V}_{h+1})}-\sqrt{\mathrm{Var}_P({V}^\star_{h+1})}\leq ||\widehat{V}_{h+1}-V^\star_{h+1}||_\infty$}. Next, we use \eqref{eqn:VPVI} as the intermediate step to crude bounding  {\small$||\widehat{V}_{h+1}-V^\star_{h+1}||_\infty\lesssim H^2/\sqrt{n\bar{d}_m}$} (where ``the use of \eqref{eqn:VPVI}'' is the more intricate self-bounding Lemma~\ref{lem:self_bound} in the actual proof) and this yields the desired structure of {\small$\sqrt{\mathrm{Var}_P({V}^\star_{h+1})}+H^2/\sqrt{n\bar{d}_m}$}. Lastly, we can combine this with \emph{the extended value difference lemma} in \cite{cai2020provably} to bound $V_1^\star-\widehat{V}_1$ and leverage the pessimistic design for bounding $\widehat{V}_1-V_1^{\widehat{\pi}}$.

For the per-instance lower bound, we reduce the problem from $\mathfrak{R}_{n}(\mathcal{P})$ to the two point testing problem and construct a problem-dependent local instance {\small$P'_h(s_{h+1}|s_h,a_h)=P_h(s_{h+1}|s_h,a_h)+\frac{P_h(s_{h+1}|s_h,a_h)\left(V_{h+1}^\star(s_{h+1})-\E_{P_{s_h,a_h}}[V_{h+1}^\star]\right)}{2\sqrt{n_{s_h,a_h}\cdot \Var_{P_{s_h,a_h}}(V_{h+1}^\star)}}$} to make sure the reduction work. Comparing to the lower bound in \cite{jin2020pessimism}, our result is per-instance and it is able recover the $\sqrt{\mathrm{Var}_P(V^\star)}$ term within those hard instances.

For the assumption-free offline RL, the use of \emph{pessimistic augmented MDP} help characterize the constant gap (due to the agnostic locations) via the following conclusion (Lemma~\ref{thm:pess_discrepancy}): 
\[
v^\pi-\sum_{h=1}^H\sum_{t=1}^{h-1}\sum_{(s_t,a_t)\in\mathcal{S}\times\mathcal{A}\backslash \mathcal{C}_h}d^\pi_t(s_t,a_t)\leq v^\pi-\sum_{h=1}^Hd^{\dagger\pi}_h(s^\dagger_h)\leq v^{\dagger\pi}\leq v^\pi.
\]
Especially, the mass of the absorbing state $s^\dagger_h$ have the expression
\[
d^{\dagger\pi}_h(s^\dagger_h)=\sum_{t=1}^{h-1}\sum_{(s_t,a_t)\in\mathcal{S}\times\mathcal{A}\backslash\mathcal{C}_t}d^{\dagger\pi}_t(s_t,a_t)
\]
which absorbs all the first time exit probabilities $d^{\dagger\pi}_t(s_t,a_t)$ under $M^\dagger$, see Section~\ref{sec:inter} for detailed explanations.

We use the following notations throughout the entire appendix. First recall $\mathcal{C}_h=\{(s_h,a_h):d^\mu_h(s_h,a_h)>0\}$ and $\bar{d}_m:=\min_{h\in[H],(s_h,a_h)\in\mathcal{C}_h}\{d^\mu_h(s_h,a_h)\}$. Also, $\iota=\log(HSA/\delta)$. Next, for any $V\in\R^S$, denote $\mathcal{T}_h(V)(s,a):=r_h(s,a)+(P_h\cdot V)(s,a)$ $\forall s,a\in\mathcal{S},\mathcal{A}$ be the Bellman update operator.

\section{Proof of VPVI (Theorem~\ref{thm:VPVI})}\label{sec:VPVI_proof}

We begin with the following helpful lemma.

\begin{lemma}\label{lem:sufficient_sample} For any $0<\delta<1$, there exists an absolute constant $c_1$ such that when total episode $n>c_1 \cdot 1/\bar{d}_m\cdot \log(HSA/\delta)$, then with probability $1-\delta$, $\forall h\in[H]$
	\[
	n_{s_h,a_h}\geq n\cdot d^\mu_h(s_h,a_h)/2,\quad\forall \; (s_h,a_h)\in\mathcal{C}_h.
	\]
	Furthermore, we denote 
	\begin{equation}\label{eqn:good_event}
	\mathcal{E}:=\{n_{s_h,a_h}\geq n\cdot d^\mu_h(s_h,a_h)/2,\;\forall \; (s_h,a_h)\in\mathcal{C}_h,\;h\in[H].\}
	\end{equation}
	then equivalently $P(\mathcal{E})>1-\delta$.
	
	In addition, we denote 
	\begin{equation}\label{eqn:good_event_1}
	\mathcal{E}':=\{n_{s_h,a_h}\leq \frac{3}{2} n\cdot d^\mu_h(s_h,a_h),\;\forall \; (s_h,a_h)\in\mathcal{C}_h,\;h\in[H].\}
	\end{equation}
	then similarly $P(\mathcal{E}')>1-\delta$.
\end{lemma}

\begin{proof}[Proof of Lemma~\ref{lem:sufficient_sample}]
	Define $E:=\{\exists h, (s_h,a_h)\in\mathcal{C}_h \;\text{s.t.}\; n_{s_h,a_h} <  n d_h^\mu(s_h,a_h)/2 \}$. Then combining the first part of multiplicative Chernoff bound (Lemma~\ref{lem:chernoff_multiplicative} in the Appendix) and a union bound, we obtain
	\begin{align*}
	\P[E] &\leq \sum_{h}\sum_{(s_h,a_h)\in\mathcal{C}_h} \P[n_{s_h,a_h} < n d^{\mu}_h(s_h,a_h)/2] \\
	&\leq HSA\cdot e^{-\frac{ n\cdot d_m}{8}}:=\delta
	\end{align*}
	solving this for $n$ then provides the stated result.
	
	For $\mathcal{E}'$ we can similarly use the second part of Lemma~\ref{lem:chernoff_multiplicative} to prove.
\end{proof}

Now in Lemma~\ref{lem:decompose_difference}, take $\pi=\pi^\star$, $\widehat{Q}_h=\overline{Q}_h$ and $\widehat{\pi}=\widehat{\pi}$ in Algorithm~\ref{alg:VPVI}, we have 
\begin{equation}\label{eqn:weak_sub_decomp}
V_1^{\pi^\star}(s)-V_1^{\widehat{\pi}}(s)\leq \sum_{h=1}^H\E_{\pi^\star}\left[\xi_h(s_h,a_h)\mid s_1=s\right]-\sum_{h=1}^H\E_{\widehat{\pi}}\left[\xi_h(s_h,a_h)\mid s_1=s\right]
\end{equation}
here $\xi_h(s,a)=(\mathcal{T}_h\widehat{V}_{h+1})(s,a)-\overline{Q}_h(s,a)$. This is true since by the definition of $\widehat{\pi}$ in Algorithm~\ref{alg:VPVI} $\langle\overline{Q}_{h}\left(s_{h}, \cdot\right), \pi_{h}\left(\cdot | s_{h}\right)-\widehat{\pi}_{h}\left(\cdot | s_{h}\right)\rangle\leq 0$ almost surely. Next we prove the asymmetric bound for $\xi_h$, which is the key lemma for the proof.

\begin{lemma}\label{lem:bellman_diff}
	Denote $\xi_h(s,a)=(\mathcal{T}_h\widehat{V}_{h+1})(s,a)-\overline{Q}_h(s,a)$, where $\widehat{V}_{h+1}$ and  $\overline{Q}_h$ are the quantities in Algorithm~\ref{alg:VPVI} and $\mathcal{T}_h(V):=r_h+P_h\cdot V$ for any $V$. Then with probability $1-\delta$, then for any $h,s_h,a_h$ such that $d^\mu_h(s_h,a_h)>0$, we have ($C'$ is an absolute constant)
	\[
	0\leq \xi_h(s_h,a_h)=(\mathcal{T}_h\widehat{V}_{h+1})(s_h,a_h)-\overline{Q}_h(s_h,a_h)\leq C'\cdot \sqrt{\frac{H^2\log(HSA/\delta)}{n\cdot d^\mu_h(s_h,a_h)}}.
	\]
\end{lemma}

\begin{proof}[Proof of Lemma~\ref{lem:bellman_diff}]
	Let us first consider the case where $n_{s_h,a_h}\geq 1$ for all $(s_h,a_h)\in\mathcal{C}_h$. In this case, by Hoeffding's inequality and a union bound, w.p. $1-\delta$, since $0\leq r_h\leq 1$,
	\begin{equation}\label{eqn:r}
	|\widehat{r}_h(s_h,a_h)-r_h(s_h,a_h)|\leq 2\sqrt{\frac{\log(HSA/\delta)}{n_{s_h,a_h}}} \;\forall (s_h,a_h)\in\mathcal{C}_h, h\in[H].
	\end{equation}
	Next, recall $\widehat{\pi}_{h+1}$ in Algorithm~\ref{alg:VPVI} is computed backwardly therefore only depends on sample tuple from time $h+1$ to $H$. Aa a result $\widehat{V}_{h+1}=\langle \overline{Q}_{h+1}, \widehat{\pi}_{h+1} \rangle$ also only depends on the sample tuple from time $h+1$ to $H$. On the other side, by our construction $\widehat{P}_h$ only depends on the transition pairs from $h$ to $h+1$. Therefore $\widehat{V}_{h+1}$ and $\widehat{P}_h$ are \emph{Conditionally} independent (This trick is also use in \cite{yin2021near}) so by Hoeffding's inequality again\footnote{It is worth mentioning if sub-policy $\widehat{\pi}_{h+1:t}$ depends on the data from all time steps $1,2,\ldots,H$, then $\widehat{V}_{h+1}$ and $\widehat{P}_h$ are no longer conditionally independent and Hoeffding's inequality cannot be applied. } (note $||\widehat{V}_h||_\infty\leq ||\overline{Q}_h||\leq H$ by VPVI)
	\begin{equation}\label{eqn:v}
	\left|\left((\widehat{P}_h-P_h)\widehat{V}_{h+1}\right)(s_h,a_h)\right|\leq 2\sqrt{\frac{H^2\cdot\log(HSA/\delta)}{n_{s_h,a_h}}}, \;\;\forall (s_h,a_h)\in\mathcal{C}_h.
	\end{equation}
	Now apply Lemma~\ref{lem:sufficient_sample}, we have with high probability the event $\mathcal{E}$ \eqref{eqn:good_event} is true, combining this with \eqref{eqn:r}, \eqref{eqn:v} and rescaling the constants we obtain with probability $1-\delta$, for all $h\in[H]$,
	\begin{equation}\label{eqn:r_v}
	\begin{aligned}
	&|\widehat{r}_h(s_h,a_h)-r_h(s_h,a_h)|\leq C\sqrt{\frac{\log(HSA/\delta)}{6n \cdot d^\mu_h{(s_h,a_h)}}} \\
	&\left|\left((\widehat{P}_h-P_h)\widehat{V}_{h+1}\right)(s_h,a_h)\right|\leq C\sqrt{\frac{H^2\cdot\log(HSA/\delta)}{6n \cdot d^\mu_h{(s_h,a_h)}}}, \;\;\forall (s_h,a_h)\in\mathcal{C}_h.
	\end{aligned}
	\end{equation}
	Now we are ready to prove the Lemma.
	
	\textbf{Step1:} we prove $\xi_h(s_h,a_h)\geq 0$ for all $(s_h,a_h)\in\mathcal{C}_h$, $h\in[H]$ with probability $1-\delta$. 
	
	We can condition on $\mathcal{E}'$ and \eqref{eqn:r_v} is true since our lemma is high probability version.
	Indeed, if $\widehat{Q}^p_h(s_h,a_h)<0$, then $\overline{Q}_h(s_h,a_h)=0$. In this case, $\xi_h(s_h,a_h)=(\mathcal{T}_h\widehat{V}_{h+1})(s_h,a_h)\geq 0$. If $\widehat{Q}^p_h(s_h,a_h)\geq 0$, then by definition $\overline{Q}_h(s_h,a_h)=\min\{\widehat{Q}^p_h(s_h,a_h),H-h+1\}^+\leq \widehat{Q}^p_h(s_h,a_h)$ and this implies
	\begin{align*}
	\xi_h(s_h,a_h)\geq& (\mathcal{T}_h\widehat{V}_{h+1})(s_h,a_h)-\widehat{Q}^p_h(s_h,a_h)\\
	=&(r_h-\widehat{r}_h)(s_h,a_h)+(P_h-\widehat{P}_h)\widehat{V}_{h+1}(s_h,a_h)+\Gamma_h(s_h,a_h)\\
	\geq &-{2}C\sqrt{\frac{H^2\cdot\log(HSA/\delta)}{6n \cdot d^\mu_h{(s_h,a_h)}}}+\Gamma_h(s_h,a_h)\\
	\geq &-C\sqrt{\frac{2H^2\cdot\log(HSA/\delta)}{3n \cdot d^\mu_h{(s_h,a_h)}}}+C\sqrt{\frac{H^2\cdot\log(HSA/\delta)}{3/2\cdot n \cdot d^\mu_h{(s_h,a_h)}}}=0
	\end{align*}
	where the second inequality uses \eqref{eqn:r_v} and the third inequality uses $\mathcal{E}'$.
	
	\textbf{Step2:} we prove $\xi_h(s_h,a_h)\leq C'\cdot \sqrt{\frac{H^2\log(HSA/\delta)}{n\cdot d^\mu_h(s_h,a_h)}}$ for all $h\in[H],(s_h,a_h)\in\mathcal{C}_h$ with probability $1-\delta$.
	
	First, since the construction $\widehat{V}_h \leq H-h+1$ for all $h\in[H]$, this implies
	\[
	\widehat{Q}^p_h=\widehat{Q}_h-\Gamma_h\leq \widehat{Q}_h= \widehat{r}_h+(\widehat{P}_{h}\widehat{V}_{h+1})\leq 1+(H-h)=H-h+1
	\]
	which uses $ \widehat{r}_h\leq 1$ almost surely and $\widehat{P}_{h}$ is row-stochastic. Due to this, we have the equivalent definition 
	\[
	\overline{Q}_h:=\min\{\widehat{Q}^p_h,H-h+1\}^+=\max\{\widehat{Q}^p_h,0\}\geq \widehat{Q}^p_h.
	\]
	Therefore
	
	{\small
	\begin{align*}
	\xi_h(s_h,a_h)=&(\mathcal{T}_h\widehat{V}_{h+1})(s_h,a_h)-\overline{Q}_h(s_h,a_h)\leq (\mathcal{T}_h\widehat{V}_{h+1})(s_h,a_h)-\widehat{Q}^p_h(s_h,a_h)\\
	=&(\mathcal{T}_h\widehat{V}_{h+1})(s_h,a_h)-\widehat{Q}_h(s_h,a_h)+\Gamma_h(s_h,a_h)\\
	=&(r_h-\widehat{r}_h)(s_h,a_h)+(P_h-\widehat{P}_h)\widehat{V}_{h+1}(s_h,a_h)+\Gamma_h(s_h,a_h)\\
	\leq &{2}C\sqrt{\frac{H^2\cdot\log(HSA/\delta)}{6n \cdot d^\mu_h{(s_h,a_h)}}}+\Gamma_h(s_h,a_h)\\
	\leq &C\sqrt{\frac{2H^2\cdot\log(HSA/\delta)}{3n \cdot d^\mu_h{(s_h,a_h)}}}+C\sqrt{\frac{2H^2\cdot\log(HSA/\delta)}{ n \cdot d^\mu_h{(s_h,a_h)}}}\\
	=&(\sqrt{\frac{2}{3}}+\sqrt{2})C\sqrt{\frac{H^2\cdot\log(HSA/\delta)}{ n \cdot d^\mu_h{(s_h,a_h)}}}:=C'\sqrt{\frac{H^2\cdot\log(HSA/\delta)}{ n \cdot d^\mu_h{(s_h,a_h)}}}
	\end{align*}
}where the first inequality uses \eqref{eqn:r_v} and the second one uses $P(\mathcal{E})\geq 1-\delta$ \eqref{eqn:good_event}.
	
	Combining Step 1 and Step 2 we finish the proof.
\end{proof}

Now we can finish proving the Theorem~\ref{thm:VPVI}. 

\begin{proof}[Proof of Theorem~\ref{thm:VPVI}]
	
	Indeed, applying Lemma~\ref{lem:bellman_diff} to \eqref{eqn:weak_sub_decomp} and average over initial distribution $s_1$, we obtain with probability $1-\delta$
	{\small
	\begin{align*}
	v^{\pi^\star}-v^{\widehat{\pi}}\leq& \sum_{h=1}^H\E_{\pi^\star}\left[\xi_h(s_h,a_h)\right]-\sum_{h=1}^H\E_{\widehat{\pi}}\left[\xi_h(s_h,a_h)\right]\\
	\leq & \sum_{h=1}^H\E_{\pi^\star}\left[\xi_h(s_h,a_h)\right]-\sum_{h=1}^H\E_{\widehat{\pi}}\left[0\right]\\
	\leq & C'H\sum_{h=1}^H\E_{\pi^\star}\left[\sqrt{\frac{\log(HSA/\delta)}{ n \cdot d^\mu_h{(s_h,a_h)}}}\right]-0\\
	=&C'H\sum_{h=1}^H\sum_{(s_h,a_h)\in\mathcal{C}_h}d^{\pi^\star}_h(s_h,a_h)\cdot\sqrt{\frac{\log(HSA/\delta)}{  d^\mu_h{(s_h,a_h)}}}\cdot \sqrt{\frac{1}{n}}
	\end{align*}}
	Note the second inequality is valid since by Line~5 of Algorithm~\ref{alg:VPVI} the Q-value at locations with $n_{s_h,a_h}=0$ are heavily penalized with $O(H)$, hence the greedy $\widehat{\pi}$ will search at locations where $n_{s_h,a_h}>0$ (which implies $d^\mu_h(s_h,a_h)>0$).
	The third inequality is valid since $d^{\pi^\star}_h(s_h,a_h)>0$ only if $d^\mu_h(s_h,a_h)>0$. Therefore the expectation over $\pi^\star$, instead of summing over all $(s_h,a_h)\in\mathcal{S}\times\mathcal{A}$, is a sum over $(s_h,a_h)$ s.t. $d^\mu_h(s_h,a_h)>0$. This completes the proof.
	
\end{proof}

\section{Discussion: the lower bound for single policy concentrability}\label{sec:dis_VPVI}

To be rigorous, here we provide some detailed explanations of \cite{rashidinejad2021bridging}. In particular, we can mirror their construction to obtain the $\Omega(\sqrt{\frac{H^3SC^\star}{n}})$ lower bound in the non-stationary finite horizon episodic setting. Indeed, their construction relies on the family with MDPs consisting of $S/4$ replicas of sub-MDPs with states $s_{0},s_1, s_{\oplus}, s_{\ominus}$. There is an additional state $s_{-1}$ and in total there are $S+1$ states. Here $s_{0},s_{\oplus}, s_{\ominus}$ all have only $1$ action $a_1$ and $s_1$ has two actions $a_1,a_2$ with transition $\mathbb{P}\left(s_{\oplus}^{j} \mid s_{1}^{j}, a_{1}\right)=\mathbb{P}\left(s_{\ominus}^{j} \mid s_{1}^{j}, a_{1}\right)=1 / 2$, $\mathbb{P}\left(s_{\oplus}^{j} \mid s_{1}^{j}, a_{2}\right)=1 / 2+v_{j} \delta$ and $\mathbb{P}\left(s_{\ominus}^{j} \mid s_{1}^{j}, a_{2}\right)=1 / 2-v_{j} \delta$. $v_{j} \in\{-1,+1\}$ is the design choice w.r.t $j$-th replica and $\delta \in[0,1 / 4]$. $s_{-1}$ transition to itself with probability $1$. The rewards for all of the states are $0$ except $s_{\oplus}^{j}$ has reward $1$ (See their Figure~5). In such a case, if $v_j = 1$, the optimal action at $s^j_1$ is $a_2$, otherwise, the optimal one is $a_1$. We can roughly create
\[
\begin{aligned}
d^{\star}\left(s_{0}^{j}\right) &=O(\frac{1}{S}), \quad d^{\star}\left(s_{1}^{j}\right)=O(\frac{1}{H S}) \\
d^{\star}\left(s_{\oplus}^{j}\right) &=\frac{\left(\frac{1}{2} \mathbf{1}\left\{v_{j}=-1\right\}+\left(\frac{1}{2}+\delta\right) \mathbf{1}\left\{v_{j}=1\right\}\right)\cdot H}{2} \cdot d^{\star}\left(s_{1}^{j}\right), \\
d^{\star}\left(s_{\ominus}^{j}\right) &=\frac{\left(\frac{1}{2} \mathbf{1}\left\{v_{j}=1\right\}+\left(\frac{1}{2}-\delta\right) \mathbf{1}\left\{v_{j}=-1\right\}\right)\cdot H}{2} \cdot d^{\star}\left(s_{1}^{j}\right), \quad d^{\star}\left(s_{-1}\right)=0
\end{aligned}
\]
and the behavior policy as 
\[
\begin{aligned}
\mu_{0}\left(s_{0}^{j}\right) &=\frac{d^{\star}\left(s_{0}^{j}\right)}{C^{\star}}, \quad \mu_{0}\left(s_{1}^{j}, a_{2}\right)=\frac{d^{\star}\left(s_{1}^{j}\right)}{C^{\star}}, \quad \mu_{0}\left(s_{1}^{j}, a_{1}\right)=d^{\star}\left(s_{1}^{j}\right) \cdot\left(1-\frac{1}{C^{\star}}\right) \\
\mu_{0}\left(s_{\oplus}^{j}\right) &= O(\frac{H}{C^{\star}} \cdot d^{\star}\left(s_{1}^{j}\right)), \quad \mu_{0}\left(s_{\ominus}^{j}\right)=O(\frac{H}{C^{\star}}) \cdot d^{\star}\left(s_{1}^{j}\right) \\
\mu_{0}\left(s_{-1}\right) &=1-\sum_{j}\left(\mu_{0}\left(s_{0}^{j}\right)+\mu_{0}\left(s_{1}^{j}\right)+\mu_{0}\left(s_{\oplus}^{j}\right)+\mu_{0}\left(s_{\ominus}^{j}\right)\right)
\end{aligned}
\]
By Fano's inequality, we can obtain: as long as 
$
O(\frac{n \delta^{2}}{HS C^{\star}}) \leq 1
$, then it holds $\inf _{\hat{\pi}} \sup _{P} \mathbb{E}\left[|v^\star-v^{\widehat{\pi}}|\right] \gtrsim H\delta$. One can set $\delta=O(\sqrt{\frac{HSC^\star}{n}})$ to obtain the result.

\section{Proof of Assumption-Free Offline Reinforcement Learning (Theorem~\ref{thm:AFRL})}\label{sec:proof_af}

Due to the assumption-free setting, the behavior policy $\mu$ is on longer guaranteed to trace any optimal policy $\pi^\star$. Therefore, in order to characterize the gap for the state-action agnostic space, we design the \emph{pessimistic augmented MDP} $M^\dagger$ to reformulate the system so that the stat-actions that are agnostic to the behavior policy are subsumed into new state $s^\dagger$. Indeed, it comes from its optimistic counterpart which has a long history (\emph{e.g.} RMAX exploration \cite{brafman2002r,jung2010gaussian}). Recently, \cite{liu2019off,kidambi2020morel,buckman2020importance} leverage this idea for continuous offline policy optimization, but their use either does not follow the assumption-free regime (see Assumption~1 of \cite{liu2019off}) or is more empirically orientated \citep{buckman2020importance,kidambi2020morel}. We find this helps to characterize the statistical gap when no assumption is made in offline RL, which provides a formal understanding of the hardness in distributional mismatches.

\subsection{Pessimistic augmented MDP} Let us define $M^\dagger$ use one extra state $s_h^\dagger$ for all $h\in\{2,\ldots,H\}$ with augmented state space $\mathcal{S}^\dagger=\mathcal{S}\cup\{s^\dagger_h\}$ and the transition and reward is defined as follows: (recall $\mathcal{C}_h:=\{(s_h,a_h):d^\mu_h(s_h,a_h)>0\}$)
{\small
\[
P^{\dagger}_h(\cdot \mid s_h, a_h)=\left\{\begin{array}{ll}
P_h(\cdot \mid s_h, a_h) & s_h, a_h \in \mathcal{C}_h, \\
\delta_{s^{\dagger}_{h+1}} & s_h=s_h^{\dagger} \text { or } s_h, a_h \notin \mathcal{C}_h,
\end{array} \;\; r^{\dagger}( s_h, a_h)=\left\{\begin{array}{ll}
r(s_h, a_h) & s_h, a_h \in \mathcal{C}_h\\
0 & s_h=s^{\dagger}_{h} \text { or } s_h, a_h \notin \mathcal{C}_h
\end{array}\right.\right.
\]
}and we further define for any $\pi$
\begin{equation}\label{eqn:value_pMDP}
V^{\dagger \pi}_h(s)=\E^\dagger_\pi\left[\sum_{t=h}^H r_t^\dagger\middle| s_h=s\right], v^{\dagger\pi}=\E^\dagger_\pi\left[\sum_{t=1}^H r_t^\dagger\right]\;\forall h\in[H].
\end{equation}
Furthermore, denote $\mathcal{K}_h:=\{(s_h,a_h):n_{s_h,a_h}>0\}$, we also create a fictitious version $\widetilde{M}^\dagger$ with:
{\small
\begin{equation}\label{eqn:wt_M}
\widetilde{P}^{\dagger}_h(\cdot \mid s_h, a_h)=\left\{\begin{array}{ll}
P_h(\cdot \mid s_h, a_h) & s_h, a_h \in \mathcal{K}_h, \\
\delta_{s^{\dagger}_{h+1}} & s_h=s_h^{\dagger} \text { or } s_h, a_h \notin \mathcal{K}_h,
\end{array} \;\; \widetilde{r}^{\dagger}( s_h, a_h)=\left\{\begin{array}{ll}
{r}(s_h, a_h) & s_h, a_h \in \mathcal{K}_h\\
0 & s_h=s^{\dagger}_{h} \text { or } s_h, a_h \notin \mathcal{C}_h
\end{array}\right.\right.
\end{equation}
}and the value functions under $\widetilde{M}^\dagger$ is similarly defined. Note in Section~\ref{sec:assumption_free}, we call \eqref{eqn:wt_M} $M^\dagger$. However, it does not really matter since $\widetilde{M}^\dagger=M^\dagger$ with high probability, as stated in the following.

\begin{lemma}\label{lem:tilde_equal_non}
	 For any $0<\delta<1$, there exists absolute constant $c$ s.t. when $n\geq c\cdot 1/\bar{d}_m \cdot\log(HSA/\delta)$,
	$$\P(\widetilde{M}^\dagger=M^\dagger)\geq 1-\delta.$$
\end{lemma}

\begin{proof}
	Note $\{\widetilde{M}^\dagger\neq M^\dagger\}\subset \{\exists\; d^\mu_h(s_h,a_h)>0\;and\;n_{s_h,a_h}=0 \}$. Similar to Lemma~\ref{lem:sufficient_sample}, this happens with probability less than $\delta$ under the condition of $n$.
\end{proof}

We have the following theorem to characterize the difference between the augmented MDP $M^\dagger$ and the original MDP $M$.
\begin{theorem}\label{thm:pess_discrepancy}
	Denote $M^\dagger=\{\mathcal{S},\mathcal{A}, H,r^\dagger,P^\dagger,d_1\}$ and for any $\pi$ denote $V^{\dagger\pi}_h$ be the value under $M^\dagger$. Then
	\begin{equation}
	v^\pi-\sum_{h=2}^{H+1}\sum_{t=1}^{h-1}\sum_{(s_t,a_t)\in\mathcal{S}\times\mathcal{A}\backslash \mathcal{C}_h}d^\pi_t(s_t,a_t)\leq v^\pi-\sum_{h=2}^{H+1}d^{\dagger\pi}_h(s^\dagger_h)\leq v^{\dagger\pi}\leq v^\pi
	\end{equation}
\end{theorem}

Before proving Theorem~\ref{thm:pess_discrepancy}, we first prove the following helper Lemmas~\ref{lem:geq_dagger}, \ref{lem:d_s_dagger}.
\begin{lemma}\label{lem:geq_dagger}
	$\forall h\in[H], (s_h,a_h)\in\mathcal{S}\times\mathcal{A}$, $d^\pi_h(s_h,a_h)\geq d^{\dagger\pi}_h(s_h,a_h)$.
\end{lemma}
\begin{proof}[Proof of Lemma~\ref{lem:geq_dagger}]
	There are two cases for $(s_h,a_h)\in\mathcal{S}\times\mathcal{A}$: either $(s_h,a_h)\in\mathcal{C}_h$ or $(s_h,a_h)\notin\mathcal{C}_h$.
	
	\textbf{Step1:} by the definition of $P^\dagger_h$, it directly holds: for all $s_{h+1}\in\mathcal{S}$ and $(s_h,a_h)\in\mathcal{S}\times\mathcal{A}$, $P^\dagger_h(s_{h+1}|s_h,a_h)\leq P_h(s_{h+1}|s_h,a_h)$.
	
	\textbf{Step2:} we prove the argument by induction. It is clear when $h=1$ $d^\pi_1(s_1,a_1)= d^{\dagger\pi}_1(s_1,a_1)$ (since there is no $s^\dagger_1$). Then for any $(s_h,a_h)\in\mathcal{S}\times\mathcal{A}$,
	\begin{align*}
	&d^\pi_{h+1}(s_{h+1},a_{h+1})=\sum_{s_h,a_h\in\mathcal{S}\times\mathcal{A}}P^\pi(s_{h+1},a_{h+1}|s_{h},a_{h})d^\pi_{h}(s_{h},a_{h})\\
	=&\sum_{s_h,a_h\in\mathcal{S}\times\mathcal{A}}\pi(a_{h+1}|s_{h+1})P^\pi_h(s_{h+1}|s_{h},a_{h})d^\pi_{h}(s_{h},a_{h})\\
	\geq &\sum_{s_h,a_h\in\mathcal{S}\times\mathcal{A}}\pi(a_{h+1}|s_{h+1})P^{\dagger\pi}_h(s_{h+1}|s_{h},a_{h})d^\pi_{h}(s_{h},a_{h})\\
	\geq &\sum_{s_h,a_h\in\mathcal{S}\times\mathcal{A}}\pi(a_{h+1}|s_{h+1})P^{\dagger\pi}_h(s_{h+1}|s_{h},a_{h})d^{\dagger\pi}_{h}(s_{h},a_{h})\\
	= &\sum_{s_h,a_h\in\mathcal{S}\times\mathcal{A},s_{h}=s_h^\dagger}\pi(a_{h+1}|s_{h+1})P^{\dagger\pi}_h(s_{h+1}|s_{h},a_{h})d^{\dagger\pi}_{h}(s_{h},a_{h})=d^{\dagger\pi}_{h+1}(s_{h+1},a_{h+1}).\\
	\end{align*}
	where the first inequality uses Step1, the second inequality uses induction assumption and the second to last equal sign uses $P^{\dagger\pi}_h(s_{h+1}|s_{h}^\dagger,a_{h})=0$ for $s_{h+1}\in\mathcal{S}$. By induction we conclude the proof for this lemma. 
	
\end{proof}
Next we prove the second lemma that measures $d^{\dagger\pi}_h(s^\dagger_h)$. 
\begin{lemma}\label{lem:d_s_dagger}
	For all $h\in[2,H+1]$, $d^{\dagger\pi}_h(s^\dagger_h)=\sum_{t=1}^{h-1}\sum_{(s_t,a_t)\in\mathcal{S}\times\mathcal{A}\backslash\mathcal{C}_t}d^{\dagger\pi}_t(s_t,a_t)$.
\end{lemma}
\begin{proof}[Proof of Lemma~\ref{lem:d_s_dagger}]
	Indeed,
	\begin{align*}
	&d^{\dagger\pi}_{h+1}(s^\dagger_{h+1})=\sum_{a_{h+1}}d^{\dagger\pi}_{h+1}(s^\dagger_{h+1},a_{h+1})\\
	=&\sum_{a_{h+1}}\sum_{(s_h,a_h)\notin\mathcal{C}_h,s_h=s^\dagger_h}P^{\dagger}(s^\dagger_{h+1},a_{h+1}\mid s_h,a_h)d^{\dagger\pi}_h(s_h,a_h)\\
	=&\sum_{a_{h+1}}\left(\sum_{(s_h,a_h)\notin\mathcal{C}_h}P^{\dagger}(s^\dagger_{h+1},a_{h+1}\mid s_h,a_h)d^{\dagger\pi}_h(s_h,a_h)+\sum_{a_h}P^{\dagger}(s^\dagger_{h+1},a_{h+1}\mid s_h^\dagger,a_h)d^{\dagger\pi}_h(s_h^\dagger,a_h)\right)\\
	=&\sum_{a_{h+1}}\left(\sum_{(s_h,a_h)\notin\mathcal{C}_h}P^{\dagger}(s^\dagger_{h+1},a_{h+1}\mid s_h,a_h)d^{\dagger\pi}_h(s_h,a_h)+\sum_{a_h}\pi(a_{h+1}\mid s_{h+1}^\dagger)d^{\dagger\pi}_h(s_h^\dagger,a_h)\right)\\
	=&\sum_{a_{h+1}}\left(\sum_{(s_h,a_h)\notin\mathcal{C}_h}P^{\dagger}(s^\dagger_{h+1},a_{h+1}\mid s_h,a_h)d^{\dagger\pi}_h(s_h,a_h)\right)+d^{\dagger\pi}_h(s_h^\dagger)\\
	=&\sum_{a_{h+1}}\left(\sum_{(s_h,a_h)\notin\mathcal{C}_h}\pi(a_{h+1}\mid s^\dagger_{h+1})d^{\dagger\pi}_h(s_h,a_h)\right)+d^{\dagger\pi}_h(s_h^\dagger)=\sum_{(s_h,a_h)\notin\mathcal{C}_h}d^{\dagger\pi}_h(s_h,a_h)+d^{\dagger\pi}_h(s_h^\dagger).\\
	\end{align*}
	Apply the above recursively we obtain the result. 
\end{proof}

Now we are ready to prove Theorem~\ref{thm:pess_discrepancy}. 

\begin{proof}[Proof of Theorem~\ref{thm:pess_discrepancy}]
	\textbf{Step1:} we first show $v^{\dagger\pi}\leq v^\pi$.
	
	Consider the stopping time $T=\inf\{t:s.t.\;(s_t,a_t)\notin\mathcal{C}_h\}\land H$. Then $1\leq T\leq H$.
	{\small
	\begin{align*}
	v^\pi=&E_{\pi}\left[\sum_{h=1}^{H}  r\left(s_{h}, a_{h}\right)\right]=E_{\pi}\left[\sum_{h=1}^{T-1}  r\left(s_{h}, a_{h}\right)+\sum_{h=T}^{H} r\left(s_{h}, a_{h}\right)\right]\\
	=&E_{\pi}^\dagger\left[\sum_{h=1}^{T-1}  r\left(s_{h}, a_{h}\right)\right]+E_{\pi}\left[\sum_{h=T}^{H} r\left(s_{h}, a_{h}\right)\right]
	\geq E_{\pi}^\dagger\left[\sum_{h=1}^{T-1}  r\left(s_{h}, a_{h}\right)\right]+E_{\pi}\left[\sum_{h=T}^{H} 0\right]\\
	=& E_{\pi}^\dagger\left[\sum_{h=1}^{T-1}  r\left(s_{h}, a_{h}\right)\right]+E_{\pi}^\dagger\left[\sum_{h=T}^{H} 0\right]
	= E_{\pi}^\dagger\left[\sum_{h=1}^{T-1}  r\left(s_{h}, a_{h}\right)\right]+E_{\pi}^\dagger\left[\sum_{h=T}^{H} r(s_h,a_h)\right]=v^{\dagger\pi},\\
	\end{align*}}
	where the third and the fourth equal signs use the distribution of $T$ is identical under either $M$ or $M^\dagger$ by construction. The fifth equal sign uses the definition of pessimistic reward.
	
	\textbf{Step2:} Next we show 
	{\small
	\begin{equation}\label{eqn:dagger_bound}
	v^\pi\leq  v^{\dagger\pi}+\sum_{h=2}^{H+1}d^{\dagger\pi}_h(s^\dagger_h)\leq v^{\dagger\pi}+\sum_{h=2}^{H+1}\sum_{t=1}^{h-1}\sum_{(s_t,a_t)\in\mathcal{S}\times\mathcal{A}\backslash \mathcal{C}_t}d^\pi_t(s_t,a_t).
	\end{equation}
    }Indeed,
	{\small
	\begin{align*}
	&v^\pi=\sum_{h=1}^H\sum_{(s_h,a_h)\in\mathcal{S}\times\mathcal{A}}d^\pi_h(s_h,a_h)r(s_h,a_h)\\
	=&\sum_{h=1}^H\sum_{(s_h,a_h)\in\mathcal{S}\times\mathcal{A}}\left(d^\pi_h(s_h,a_h)-d^{\dagger\pi}_h(s_h,a_h)\right)r(s_h,a_h)+\sum_{h=1}^H\sum_{(s_h,a_h)\in\mathcal{S}\times\mathcal{A}}d^{\dagger\pi}_h(s_h,a_h)r(s_h,a_h)\\
	\leq&\sum_{h=1}^H\sum_{(s_h,a_h)\in\mathcal{S}\times\mathcal{A}}\left(d^\pi_h(s_h,a_h)-d^{\dagger\pi}_h(s_h,a_h)\right)\cdot 1+\sum_{h=1}^H\sum_{(s_h,a_h)\in\mathcal{S}\times\mathcal{A}}d^{\dagger\pi}_h(s_h,a_h)r(s_h,a_h)\\
	= &\sum_{h=1}^H\left(1-\sum_{(s_h,a_h)\in\mathcal{S}\times\mathcal{A}}d^{\dagger\pi}_h(s_h,a_h)\right)+\sum_{h=1}^H\sum_{(s_h,a_h)\in\mathcal{S}\times\mathcal{A}}d^{\dagger\pi}_h(s_h,a_h)r(s_h,a_h)\\
	=&\sum_{h=2}^H d^{\dagger\pi}_h(s_h^\dagger)+\sum_{h=1}^H\sum_{(s_h,a_h)\in\mathcal{S}\times\mathcal{A}}d^{\dagger\pi}_h(s_h,a_h)r(s_h,a_h)\\
	=&\sum_{h=2}^H d^{\dagger\pi}_h(s_h^\dagger)+\sum_{h=1}^H\sum_{(s_h,a_h)\in\mathcal{S}\times\mathcal{A}}d^{\dagger\pi}_h(s_h,a_h)
	\left(r(s_h,a_h)-r^\dagger(s_h,a_h)\right)
	+\sum_{h=1}^H\sum_{(s_h,a_h)\in\mathcal{S}\times\mathcal{A}}d^{\dagger\pi}_h(s_h,a_h)
	r^\dagger(s_h,a_h)\\
	=&\sum_{h=2}^H d^{\dagger\pi}_h(s_h^\dagger)+\sum_{h=1}^H\sum_{(s_h,a_h)\notin\mathcal{C}_h}d^{\dagger\pi}_h(s_h,a_h)
	\left(r(s_h,a_h)-r^\dagger(s_h,a_h)\right)
	+\sum_{h=1}^H\sum_{(s_h,a_h)\in\mathcal{S}\times\mathcal{A}}d^{\dagger\pi}_h(s_h,a_h)
	r^\dagger(s_h,a_h)\\
	\end{align*}
	\begin{align*}
	=&\sum_{h=2}^H d^{\dagger\pi}_h(s_h^\dagger)+\sum_{h=1}^H\sum_{(s_h,a_h)\notin\mathcal{C}_h}d^{\dagger\pi}_h(s_h,a_h)
	\left(r(s_h,a_h)-r^\dagger(s_h,a_h)\right)
	+v^{\dagger\pi}\\
	\leq &\sum_{h=2}^H d^{\dagger\pi}_h(s_h^\dagger)+\sum_{h=1}^H\sum_{(s_h,a_h)\notin\mathcal{C}_h}d^{\dagger\pi}_h(s_h,a_h)\cdot 1
	+v^{\dagger\pi}=\sum_{h=2}^{H+1} d^{\dagger\pi}_h(s_h^\dagger)+v^{\dagger\pi}\\
	\end{align*}
	}The first inequality is due to Lemma~\ref{lem:geq_dagger}. The fourth equal sign uses $d^{\dagger}_1(s^\dagger_1)=0$. The sixth equal sign is due to $r(s_h,a_h)=r^\dagger(s_h,a_h)$ when $(s_h,a_h)\in \mathcal{C}_h$. The seventh equal sign is due to $r^\dagger(s^\dagger_h,a_h)=0$. The last equal sign uses Lemma~\ref{lem:d_s_dagger}. The right inequality in \eqref{eqn:dagger_bound} uses Lemma~\ref{lem:geq_dagger}. Step 1 and Step 2 conclude the proof of Theorem~\ref{thm:pess_discrepancy}.
\end{proof}

\subsubsection{Strong adaptive assumption-free bound}\label{sec:af_bound}

Now we are ready to launch the \emph{assumption-free} AVPI (Algorithm~\ref{alg:APVI}) with the following model-based construction $\widehat{M}^\dagger$ (recall $\mathcal{K}_h:=\{(s_h,a_h):n_{s_h,a_h}>0\}$):
{\small
\[
\widehat{P}^{\dagger}_h(\cdot \mid s_h, a_h)=\left\{\begin{array}{ll}
\widehat{P}_h(\cdot \mid s_h, a_h) & s_h, a_h \in \mathcal{K}_h, \\
\delta_{s^{\dagger}_{h+1}} & s_h=s_h^{\dagger} \text { or } s_h, a_h \notin \mathcal{K}_h,
\end{array} \;\; \widehat{r}^{\dagger}( s_h, a_h)=\left\{\begin{array}{ll}
\widehat{r}(s_h, a_h) & s_h, a_h \in \mathcal{S}\times\mathcal{A} \\
0 & s_h=s^{\dagger}_{h} \text { or } s_h, a_h \notin \mathcal{C}_h
\end{array}\right.\right.
\]
}where $\widehat{P},\widehat{r}$ is defined as
\begin{equation}
\widehat{P}_h(s'|s_h,a_h)=\frac{\sum_{\tau=1}^n\mathbf{1}[(s^{\tau}_{h+1},a^{\tau}_h,s^{\tau}_h)=(s^\prime,s_h,a_h)]}{n_{s_h,a_h}},\; \widehat{r}_h(s_h,a_h)=\frac{\sum_{\tau=1}^n\mathbf{1}[(a^{\tau}_h,s^{\tau}_h)=(s_h,a_h)]\cdot r_h^\tau}{n_{s_h,a_h}},
\end{equation}

The benefit of using $\widetilde{M}^\dagger$ \eqref{eqn:wt_M} is that in $\widetilde{M}^\dagger$ there is no agnostic location even no assumption is made.  The $\widehat{M}^\dagger$ creates a empirical estimate for $\widetilde{M}^\dagger$. In this case, the pessimistic bonus is designed as 
 \[
 \Gamma_h(s_h,a_h)=2\sqrt{\frac{\mathrm{Var}_{\widehat{P}^\dagger_{s_h,a_h}}(\widehat{r}^\dagger_h+\widehat{V}_{h+1})\cdot\iota}{n_{s_h,a_h}}}+\frac{14H\cdot\iota}{3n_{s_h,a_h}}
 \]
  if $n_{s_h,a_h}\in\mathcal{K}_h$ and $0$ otherwise (here $\widehat{V}_{h+1}$ is computed backwardly from the next time step in Algorithm~\ref{alg:APVI}). Now let us start the proof. First of all, let us assume $\widetilde{M}^\dagger=M^\dagger$ for the moment so we can get rid of the tilde expression for notation convenience. We will formally recover the result for $M^\dagger$ at the end by Lemma~\ref{lem:tilde_equal_non}. 
  
  In particular, while we always use $\pi^\star$ to denote the optimal policy in the \emph{Original} MDP, we augment it in the $M^\dagger$($\widetilde{M}^\dagger$) arbitrarily and abuse the notation as:
  \begin{equation}\label{eqn:optimal_policy}
 \pi^\star( \cdot |s_h)=\left\{\begin{array}{ll}
  \pi^\star( \cdot |s_h) & s_h \in \mathcal{S} \\
   arbitrary\;distribution& s_h=s^{\dagger}_{h} 
  \end{array}\right.
  \end{equation}
  
  and always use $\widehat{\pi}$ to denote the output of Algorithm~\ref{alg:APVI}. We rely on the following lemma that characterize the suboptimality gap.

\begin{lemma}\label{lem:sub_gap}

Recall $\pi^\star$ in \eqref{eqn:optimal_policy} and define $(\mathcal{T}^\dagger_{h} V)(\cdot,\cdot):=r_h^\dagger(\cdot,\cdot)+(P_h^\dagger V)(\cdot,\cdot)$ for any $V\in\R^{S+1}$. Note $\widehat{\pi}$, $\overline{Q}_h$, $\widehat{V}_h$ are defined in Algorithm~\ref{alg:APVI} and denote $\xi^\dagger_h(s,a)=(\mathcal{T}^\dagger_h\widehat{V}_{h+1})(s,a)-\overline{Q}_h(s,a)$.  
\begin{equation}\label{eqn:sub_decomp}
V_1^{\dagger\pi^\star}(s)-V_1^{\dagger\widehat{\pi}}(s)\leq \sum_{h=1}^H\E^\dagger_{\pi^\star}\left[\xi^\dagger_h(s_h,a_h)\mid s_1=s\right]-\sum_{h=1}^H\E^\dagger_{\widehat{\pi}}\left[\xi^\dagger_h(s_h,a_h)\mid s_1=s\right].
\end{equation}
where $V_1^{\dagger\pi}$ is defined in \eqref{eqn:value_pMDP}. Furthermore, \eqref{eqn:sub_decomp} holds for all $V_h^{\dagger\pi^\star}(s)-V_h^{\dagger\widehat{\pi}}(s)$.
\end{lemma}

\begin{proof}[Proof of Lemma~\ref{lem:sub_gap}]
Apply Lemma~\ref{lem:decompose_difference} with $\mathcal{T}_h=\mathcal{T}^\dagger_h$, $\pi=\pi^\star$, $\widehat{Q}_h=\overline{Q}_h$ and $\widehat{\pi}=\widehat{\pi}$ in Algorithm~\ref{alg:APVI}, we can obtain the result since by the definition of $\widehat{\pi}$ in Algorithm~\ref{alg:APVI} $\langle\overline{Q}_{h}\left(s_{h}, \cdot\right), \pi_{h}\left(\cdot | s_{h}\right)-\widehat{\pi}_{h}\left(\cdot | s_{h}\right)\rangle\leq 0$ almost surely for any $\pi$. The proof for $V_h^{\dagger\pi^\star}(s)-V_h^{\dagger\widehat{\pi}}(s)$ is identical.
\end{proof}

Next we prove the adaptive asymmetric bound for $\xi^\dagger_h$, which is the key for recover the structure of intrinsic bound.

\begin{lemma}\label{lem:bellman_diff_tight}
	Denote $\xi^\dagger_h(s,a)=(\mathcal{T}^\dagger_h\widehat{V}_{h+1})(s,a)-\overline{Q}_h(s,a)$, where $\widehat{V}_{h+1}$ and  $\overline{Q}_h$ are the quantities in Algorithm~\ref{alg:APVI} and $\mathcal{T}^\dagger_h(V):=r^\dagger_h+P^\dagger_h\cdot V$ for any $V\in\R^{S+1}$. Then with probability $1-\delta$, then for any $h,s_h,a_h$ such that $n_{s_h,a_h}>0$, we have 
	\begin{align*}
	0\leq &\xi^\dagger_h(s_h,a_h)=(\mathcal{T}^\dagger_h\widehat{V}_{h+1})(s_h,a_h)-\overline{Q}_h(s_h,a_h)\\
	\leq &4\sqrt{\frac{\mathrm{Var}_{\widehat{P}^\dagger_{s_h,a_h}}(\widehat{r}^\dagger_h+\widehat{V}_{h+1})\cdot\log(HSA/\delta)}{n_{s_h,a_h}}}+\frac{28H\cdot\log(HSA/\delta)}{3n_{s_h,a_h}}
	\end{align*}
\end{lemma}

\begin{proof}[Proof of Lemma~\ref{lem:bellman_diff_tight}]
	Recall we are under $M^\dagger$ ($\widehat{M}^\dagger$). For all $(s_h,a_h)\in\mathcal{K}_h$, by Empirical Bernstein inequality (Lemma~\ref{lem:empirical_bernstein_ineq}) and a union bound\footnote{Here note even though $|\mathcal{S}^\dagger|=S+1$, for state $s^\dagger_h$ we always have $n_{s^\dagger_h,a_h}=0$ for any $a_h$. Therefore apply the union bound only provides $HSA$ in th log term instead of $H(S+1)A$.}, w.p. $1-\delta$, since $0\leq r^\dagger_h\leq 1$,
	\begin{equation}\label{eqn:r_bern}
	|\widehat{r}^\dagger_h(s_h,a_h)-r^\dagger_h(s_h,a_h)|\leq \sqrt{\frac{2\mathrm{Var}_{\widehat{P}^\dagger}(\widehat{r}^\dagger_h)\log(HSA/\delta)}{n_{s_h,a_h}}} +\frac{7\log(HSA/\delta)}{3n_{s_h,a_h}}\;\forall (s_h,a_h)\in\mathcal{K}_h, h\in[H].
	\end{equation}
	Next, recall $\widehat{\pi}_{h+1}$ in Algorithm~\ref{alg:APVI} is computed backwardly therefore only depends on sample tuple from time $h+1$ to $H$. Aa a result $\widehat{V}_{h+1}=\langle \overline{Q}_{h+1}, \widehat{\pi}_{h+1} \rangle$ also only depends on the sample tuple from time $h+1$ to $H$. On the other side, by our construction $\widehat{P}^\dagger_h$ only depends on the transition pairs from $h$ to $h+1$. Therefore $\widehat{V}_{h+1}$ and $\widehat{P}^\dagger_h$ are \emph{Conditionally} independent (This trick is also use in \cite{yin2021near}) so by Empirical Bernstein inequality again\footnote{It is worth mentioning if sub-policy $\widehat{\pi}_{h+1:t}$ depends on the data from all time steps $1,2,\ldots,H$, then $\widehat{V}_{h+1}$ and $\widehat{P}_h$ are no longer conditionally independent and Hoeffding's inequality cannot be applied. }  and a union bound (note $||\widehat{V}_h||_\infty\leq ||\overline{Q}_h||\leq H$ by APVI) for all $(s_h,a_h)\in\mathcal{K}_h$, w.p. $1-\delta$,
	\begin{equation}\label{eqn:v_bern}
	\left|\left((\widehat{P}^\dagger_h-P^\dagger_h)\widehat{V}_{h+1}\right)(s_h,a_h)\right|\leq \sqrt{\frac{2\mathrm{Var}_{\widehat{P}^\dagger_{s_h,a_h}}(\widehat{V}_{h+1})\cdot\log(HSA/\delta)}{n_{s_h,a_h}}}+\frac{7H\cdot\log(HSA/\delta)}{3n_{s_h,a_h}}.
	\end{equation}

	Now we are ready to prove the Lemma.
	
	\textbf{Step1:} we prove $\xi_h(s_h,a_h)\geq 0$ for all $(s_h,a_h)\in\mathcal{K}_h$, $h\in[H]$ with probability $1-\delta$. 
	
	Indeed, if $\widehat{Q}^p_h(s_h,a_h)<0$, then $\overline{Q}_h(s_h,a_h)=0$. In this case, $\xi_h(s_h,a_h)=(\mathcal{T}_h\widehat{V}_{h+1})(s_h,a_h)\geq 0$ (note $\widehat{V}_{h}\geq 0$ by the definition). If $\widehat{Q}^p_h(s_h,a_h)\geq 0$, then by definition $\overline{Q}_h(s_h,a_h)=\min\{\widehat{Q}^p_h(s_h,a_h),H-h+1\}^+\leq \widehat{Q}^p_h(s_h,a_h)$ and this implies
	\begin{align*}
	&\xi^\dagger_h(s_h,a_h)\geq (\mathcal{T}^\dagger_h\widehat{V}_{h+1})(s_h,a_h)-\widehat{Q}^p_h(s_h,a_h)\\
	=&(r^\dagger_h-\widehat{r}^\dagger_h)(s_h,a_h)+(P^\dagger_h-\widehat{P}^\dagger_h)\widehat{V}_{h+1}(s_h,a_h)+\Gamma_h(s_h,a_h)\\
	\geq &-2\sqrt{\frac{\mathrm{Var}_{\widehat{P}^\dagger_{s_h,a_h}}(\widehat{r}^\dagger_h+\widehat{V}_{h+1})\cdot\log(HSA/\delta)}{n_{s_h,a_h}}}-\frac{14H\cdot\log(HSA/\delta)}{3n_{s_h,a_h}}+\Gamma_h(s_h,a_h)=0\\
	\end{align*}
	where the inequality uses \eqref{eqn:r_bern}, \eqref{eqn:v_bern} and $\sqrt{a}+\sqrt{b}\leq \sqrt{2(a+b)}$ and $r_h$ and $s_{h+1}$ are conditionally independent given $s_h,a_h$.
	The last equal sign uses Line~6 of Algorithm~\ref{alg:APVI}.
	
	\textbf{Step2:} we prove $\xi^\dagger_h(s_h,a_h)
	\leq 4\sqrt{\frac{\mathrm{Var}_{\widehat{P}^\dagger_{s_h,a_h}}(\widehat{r}^\dagger_h+\widehat{V}_{h+1})\cdot\log(HSA/\delta)}{n_{s_h,a_h}}}+\frac{28H\cdot\log(HSA/\delta)}{3n_{s_h,a_h}}$ for all $h\in[H],(s_h,a_h)\in\mathcal{K}_h$ with probability $1-\delta$.
	
	First, since by construction $\widehat{V}_h \leq H-h+1$ for all $h\in[H]$, this implies
	\[
	\widehat{Q}^p_h=\widehat{Q}_h-\Gamma_h\leq \widehat{Q}_h= \widehat{r}^\dagger_h+(\widehat{P}^\dagger_{h}\widehat{V}_{h+1})\leq 1+(H-h)=H-h+1
	\]
	which uses $ \widehat{r}^\dagger_h\leq 1$ almost surely and $\widehat{P}^\dagger_{h}$ is row-stochastic. Due to this, we have the equivalent definition 
	\[
	\overline{Q}_h:=\min\{\widehat{Q}^p_h,H-h+1\}^+=\max\{\widehat{Q}^p_h,0\}\geq \widehat{Q}^p_h.
	\]
	Therefore
	\begin{align*}
	&\xi^\dagger_h(s_h,a_h)=(\mathcal{T}^\dagger_h\widehat{V}_{h+1})(s_h,a_h)-\overline{Q}_h(s_h,a_h)\leq (\mathcal{T}^\dagger_h\widehat{V}_{h+1})(s_h,a_h)-\widehat{Q}^p_h(s_h,a_h)\\
	=&(\mathcal{T}^\dagger_h\widehat{V}_{h+1})(s_h,a_h)-\widehat{Q}_h(s_h,a_h)+\Gamma_h(s_h,a_h)\\
	=&(r^\dagger_h-\widehat{r}^\dagger_h)(s_h,a_h)+(P^\dagger_h-\widehat{P}^\dagger_h)\widehat{V}_{h+1}(s_h,a_h)+\Gamma_h(s_h,a_h)\\
	\leq &2\sqrt{\frac{\mathrm{Var}_{\widehat{P}^\dagger_{s_h,a_h}}(\widehat{r}^\dagger_h+\widehat{V}_{h+1})\cdot\log(HSA/\delta)}{n_{s_h,a_h}}}+\frac{14H\cdot\log(HSA/\delta)}{3n_{s_h,a_h}}+\Gamma_h(s_h,a_h)\\
	=&4\sqrt{\frac{\mathrm{Var}_{\widehat{P}^\dagger_{s_h,a_h}}(\widehat{r}^\dagger_h+\widehat{V}_{h+1})\cdot\log(HSA/\delta)}{n_{s_h,a_h}}}+\frac{28H\cdot\log(HSA/\delta)}{3n_{s_h,a_h}}.
	\end{align*}
	Combining Step 1 and Step 2 we finish the proof.
\end{proof}

\subsubsection{Proof of Theorem~\ref{thm:AFRL}}

Now we are ready to prove the Theorem~\ref{thm:AFRL}. 

	First of all, by Lemma~\ref{lem:sub_gap} and Lemma~\ref{lem:bellman_diff_tight}, for all $t\in[H]$, $s\in\mathcal{S}$ (excluding $s^\dagger$) w.p. $1-\delta$
	\begin{equation}\label{eqn:expression_bound}
	\begin{aligned}
&V_t^{\dagger\pi^\star}(s)-V_t^{\dagger\widehat{\pi}}(s)\leq \sum_{h=t}^H\E^\dagger_{\pi^\star}\left[\xi^\dagger_h(s_h,a_h)\mid s_t=s\right]-\sum_{h=t}^H\E^\dagger_{\widehat{\pi}}\left[\xi^\dagger_h(s_h,a_h)\mid s_t=s\right]\\
\leq&\sum_{h=t}^H\E^\dagger_{\pi^\star}\left[\xi^\dagger_h(s_h,a_h)\mid s_t=s\right]-0\\
\leq&\sum_{h=t}^H\E^\dagger_{\pi^\star}\left[4\sqrt{\frac{\mathrm{Var}_{\widehat{P}^\dagger_{s_h,a_h}}(\widehat{r}^\dagger_h+\widehat{V}_{h+1})\cdot\iota}{n_{s_h,a_h}}}+\frac{28H\cdot\iota}{3n_{s_h,a_h}}\mid s_t=s\right]\\
\leq&\sum_{h=t}^H\E^\dagger_{\pi^\star}\left[4\sqrt{\frac{2\mathrm{Var}_{\widehat{P}^\dagger_{s_h,a_h}}(\widehat{r}^\dagger_h+\widehat{V}_{h+1})\cdot\iota}{nd^\mu_h(s_h,a_h)}}+\frac{56H\cdot\iota}{3nd^\mu_h(s_h,a_h)}\mid s_t=s\right]\\
	\end{aligned}
	\end{equation}
	here recall the expectation is only taken over $s_h,a_h$. Note by the Pessimistic MDP $\widetilde{M}^\dagger$ ($\widehat{M}^\dagger$), for all $(s_h,a_h)\notin \mathcal{K}_h$ and $s^\dagger_h$, the pessimistic reward leads to $Q^{\dagger\pi}(s_h,a_h),V^{\dagger\pi}(s^\dagger_h)=0$ for any $\pi$, therefore Lemma~\ref{lem:bellman_diff_tight} can be applied. Moreover, the last inequality is by Lemma~\ref{lem:sufficient_sample}.

	\begin{lemma}[self-bounding]\label{lem:self_bound}
	 We prove, for all $t\in[H]$, w.p. $1-\delta$, for all $s\in\mathcal{S}$ (excluding $s^\dagger$),
	\[
	\left|V_t^{\dagger\pi^\star}(s)-\widehat{V}_t(s)\right|\leq \frac{8\sqrt{2\iota}H^2}{\sqrt{n\cdot\bar{d}_m}} +\frac{112H^2\cdot \iota}{3n\cdot \bar{d}_m}.
	\]
	where $\bar{d}_m$ is defined in Theorem~\ref{thm:AFRL}.
	\end{lemma}

	\begin{remark}
		The self-bounding lemma essentially provides a crude high probability bound for $|V_t^{\dagger\pi^\star}-\widehat{V}_t|$ (or $|V_t^{\dagger\pi^\star}-V_t^{\dagger\widehat{\pi}}|$) with suboptimal order $\widetilde{O}(\frac{H^2}{\sqrt{n\bar{d}_m}})$ and we can use it to further bound the higher order term in the main result.
	\end{remark}

\begin{proof}[Proof of Lemma~\ref{lem:self_bound}]
	Indeed, by \eqref{eqn:expression_bound}, since $\mathrm{Var}_{\widehat{P}^\dagger_{s_h,a_h}}(\widehat{r}^\dagger_h+\widehat{V}_{h+1})\leq H^2$, we have w.p. $1-\delta$,
	\begin{equation}\label{eqn:inter3}
	\left|V_t^{\dagger\pi^\star}(s)-V_t^{\dagger\widehat{\pi}}(s)\right|\leq \frac{4\sqrt{2\iota}H^2}{\sqrt{n\cdot\bar{d}_m}} +\frac{56H^2\cdot \iota}{3n\cdot \bar{d}_m}
	\end{equation}
	for all $t\in[H]$. Next, when apply Lemma~\ref{lem:decompose_difference} to Lemma~\ref{lem:sub_gap}, by \eqref{eqn:inter1} and \eqref{eqn:inter2} we essentially obtain
	\begin{align*}
	V_t^{\dagger\pi^\star}(s)-\widehat{V}_t(s)=&\sum_{h=t}^H\E^\dagger_{\pi^\star}\left[\xi^\dagger_h(s_h,a_h)\mid s_t=s\right]
	+\sum_{h=t}^{H} \mathbb{E}^\dagger_{\pi^\star}\left[\langle\widehat{Q}_{h}\left(s_{h}, \cdot\right), \pi^\star_{h}\left(\cdot | s_{h}\right)-\widehat{\pi}_{h}\left(\cdot | s_{h}\right)\rangle \mid s_{t}=s\right]\\
	\leq&\frac{4\sqrt{2\iota}H^2}{\sqrt{n\cdot\bar{d}_m}} +\frac{56H^2\cdot \iota}{3n\cdot \bar{d}_m}+0
	\end{align*}
	and
	\[
	\widehat{V}_t(s)-V_t^{\dagger\widehat{\pi}}(s)=-\sum_{h=t}^H\E^\dagger_{\widehat{\pi}}\left[\xi^\dagger_h(s_h,a_h)\mid s_t=s\right]\geq 0.
	\]
	Combing those two with \eqref{eqn:inter3} we obtain the result.

\end{proof}

\begin{lemma}\label{lem:var_change}
	For all $(a_h,a_h)\in\mathcal{K}_h$ and any $||V||_\infty\leq H$, w.p. $1-\delta$,
	\[
	\sqrt{\mathrm{Var}_{\widehat{P}^\dagger_{s_h,a_h}}(V)}\leq 6H \sqrt{\frac{\iota}{n\cdot d^\mu_h(s_h,a_h)}}+\sqrt{\mathrm{Var}_{{P}^\dagger_{s_h,a_h}}(V)}.
	\] 
\end{lemma}
\begin{proof}

	This is a direct application of Lemma~\ref{lem:sqrt_var_diff} with a union bound. Specifically, we apply $\frac{n-1}{n}\leq 1$.
	
\end{proof}

Now by Lemma~\ref{lem:self_bound} and Lemma~\ref{lem:var_change}, for all $(s_h,a_h)\in\mathcal{K}_h$, w.p. $1-\delta$,
\begin{align*}
&\sqrt{\mathrm{Var}_{\widehat{P}^\dagger_{s_h,a_h}}(\widehat{r}^\dagger_h+\widehat{V}_{h+1})}\leq \sqrt{\mathrm{Var}_{{P}^\dagger_{s_h,a_h}}(\widehat{r}^\dagger_h+\widehat{V}_{h+1})}+6H \sqrt{\frac{\iota}{n\cdot d^\mu_h(s_h,a_h)}}\\
\leq &\sqrt{\mathrm{Var}_{{P}^\dagger_{s_h,a_h}}({r}^\dagger_h+{V}^{\dagger\pi^\star}_{h+1})}+\norm{(\widehat{r}^\dagger_h+\widehat{V}_{h+1})-({r}^\dagger_h+{V}^{\dagger\pi^\star}_{h+1})}_{\infty,s\in\mathcal{S}}+6H \sqrt{\frac{\iota}{n\cdot d^\mu_h(s_h,a_h)}}\\
\leq&\sqrt{\mathrm{Var}_{{P}^\dagger_{s_h,a_h}}({r}^\dagger_h+{V}^{\dagger\pi^\star}_{h+1})}+\frac{10\sqrt{2\iota}H^2}{\sqrt{n\cdot\bar{d}_m}} +\frac{112H^2\cdot \iota}{3n\cdot \bar{d}_m}+6H \sqrt{\frac{\iota}{n\cdot d^\mu_h(s_h,a_h)}}\\
\end{align*}
Therefore plug this into \eqref{eqn:expression_bound}, and average over $s_1$, we finally get, w.p. $1-\delta$,
	\begin{align*}
	&v^{\dagger\pi^\star}-v^{\dagger\widehat{\pi}}\leq 
	\sum_{h=1}^H\E^\dagger_{\pi^\star}\left[4\sqrt{\frac{2\mathrm{Var}_{\widehat{P}^\dagger_{s_h,a_h}}(\widehat{r}^\dagger_h+\widehat{V}_{h+1})\cdot\iota}{nd^\mu_h(s_h,a_h)}}+\frac{56H\cdot\iota}{3nd^\mu_h(s_h,a_h)}\mid s_1=s\right]\\
	\leq& C'\sum_{h=1}^H\E^\dagger_{\pi^\star}\left[\sqrt{\frac{\mathrm{Var}_{{P}^\dagger_{s_h,a_h}}({r}^\dagger_h+{V}^{\dagger\pi^\star}_{h+1})\cdot\iota}{nd^\mu_h(s_h,a_h)}}\right]+\widetilde{O}(\frac{H^3}{n\cdot\bar{d}_m})\\
	=&C'\sum_{h=1}^H\sum_{(s_h,a_h)\in\mathcal{K}_h}d^{\dagger\pi^\star}(s_h,a_h)\sqrt{\frac{\mathrm{Var}_{{P}^\dagger_{s_h,a_h}}({r}^\dagger_h+{V}^{\dagger\pi^\star}_{h+1})\cdot\iota}{nd^\mu_h(s_h,a_h)}}+\widetilde{O}(\frac{H^3}{n\cdot\bar{d}_m})\\
\end{align*}
here $\widetilde{O}$ absorbs log factor and even higher orders. 

Note throughout the section we assume $\widetilde{M}^\dagger={M}^\dagger$. Now be Lemma~\ref{lem:tilde_equal_non}, we can replace the $\mathcal{K}_h$ in above by $\mathcal{C}_h$ so the result holds in high probability.

Lastly, we end up with w.p. $1-\delta$
\begin{equation}\label{eqn:af_final}
\begin{aligned}
0\leq &v^{\pi^\star}-v^{\widehat{\pi}}\leq \sum_{h=2}^{H+1} d^{\dagger\pi^\star}_h(s^\dagger_h)+ v^{\dagger\pi^\star} -v^{\widehat{\pi}}\leq\sum_{h=2}^{H+1} d^{\dagger\pi^\star}_h(s^\dagger_h)+ v^{\dagger\pi^\star} -v^{\dagger\widehat{\pi}} \\
\leq&\sum_{h=2}^{H+1} d^{\dagger\pi^\star}_h(s^\dagger_h)+C'\sum_{h=1}^H\sum_{(s_h,a_h)\in\mathcal{C}_h}d^{\dagger\pi^\star}_h(s_h,a_h)\sqrt{\frac{\mathrm{Var}_{{P}^\dagger_{s_h,a_h}}({r}^\dagger_h+{V}^{\dagger\pi^\star}_{h+1})\cdot\iota}{nd^\mu_h(s_h,a_h)}}+\widetilde{O}(\frac{H^3}{n\cdot\bar{d}_m})\\
\end{aligned}
\end{equation}
where the first inequality uses Lemma~\ref{thm:pess_discrepancy} with $\pi=\pi^\star$ and the second one uses Lemma~\ref{thm:pess_discrepancy} with $\pi=\widehat{\pi}$. This concludes the proof of Theorem~\ref{thm:AFRL}. The rest of the results are coming from Lemma~\ref{lem:geq_dagger},\ref{lem:d_s_dagger}.

\begin{remark}
	We mention the summation of the main term in \eqref{eqn:af_final} does not include $s_h^\dagger$ since $V^{\dagger\pi}_{h}(s^\dagger_h)=0$ for any $\pi$ due to the pessimistic MDP design. In particular, this state contributes nothing to neither $v^{\dagger\pi^\star}$ nor $v^{\dagger\widehat{\pi}}$.
\end{remark}

\subsection{Interpretation of Theorem~\ref{thm:AFRL}}\label{sec:inter} 

The constant (in $n$) gap, which is incurred by the behavior agnostic space $\bigcup_{h=1}^H \{(s_h,a_h): d^\mu_h(s_h,a_h)=0\}$, is bounded by 
\[
\sum_{h=2}^{H+1}d^{\dagger\pi^\star}_h(s^\dagger_h)=\sum_{h=2}^{H+1}\sum_{t=1}^{h-1}\sum_{(s_t,a_t)\in\mathcal{S}\times\mathcal{A}\backslash\mathcal{C}_t}d^{\dagger\pi^\star}_t(s_t,a_t)\leq \sum_{h=2}^{H+1}\sum_{t=1}^{h-1}\sum_{(s_t,a_t)\in\mathcal{S}\times\mathcal{A}\backslash\mathcal{C}_t}d^{\pi^\star}_t(s_t,a_t),
\]
Note for quantity $d^{\dagger\pi^\star}_t(s_t,a_t)$ (where $(s_t,a_t)\in\mathcal{S}\times\mathcal{A}\backslash \mathcal{C}_t$), it is equivalently defined as 
\[
d^{\dagger\pi^\star}_t(s_t,a_t)=\P_{M^\dagger}\left[ S_t,A_t=s_t,a_t\middle| (S_{t-1},A_{t-1})\in\mathcal{C}_{t-1},\ldots,(S_{1},A_{1})\in\mathcal{C}_{1} \right]
\]
is probability for the first time the trajectory exits the reachable regions and enters $(s_t,a_t)\notin\mathcal{C}_t$. Therefore, $d^{\dagger\pi^\star}_t(s_t,a_t)$ is much smaller than $d^{\pi^\star}_t(s_t,a_t)$ for $s_t,a_t\notin\mathcal{C}_h$ (since $d^{\pi^\star}_t(s_t,a_t)$ includes the probability that trajectory $s_t,a_t$). Such a feature is reflected by the quantity that express the gap using the mass of the absorbing state: $\sum_{h=2}^{H+1}d^{\dagger\pi^\star}_h(s^\dagger_h)(=\sum_{h=2}^{H+1}\sum_{t=1}^{h-1}\sum_{(s_t,a_t)\in\mathcal{S}\times\mathcal{A}\backslash\mathcal{C}_t}d^{\dagger\pi^\star}_t(s_t,a_t))$. Especially, this gap can vary between $0$ and $H$, depending on the exploratory ability of $\mu$. Also, different from AVPI, the \emph{assumption-free} AVPI set $0$ penalty at locations where $n_{s_t,a_t}=0$. The interpretation is: the locations with $n_{s_t,a_t}=0$ in $M^\dagger$ are the fully aware locations (with deterministic transition to $s^\dagger$ and reward $0$ by design) therefore we are certain about the behaviors in those places.

\section{Proof of Theorem~\ref{thm:APVI}}\label{sec:proof_APVI}
Indeed, Theorem~\ref{thm:APVI} can be implied by Theorem~\ref{thm:AFRL} as a special case.
\begin{proof}[Proof of Theorem~\ref{thm:APVI}]
Under Assumption~\ref{assum:single_concen}, $d^{\pi^\star}_h(s_h,a_h)=0$ if $d^\mu_h(s_h,a_h)=0$. In this case, 
\begin{align*}
0\leq& \sum_{h=2}^{H+1} d^{\dagger\pi^\star}_h(s^\dagger_h)=\sum_{h=2}^{H+1}\sum_{t=1}^{h-1}\sum_{(s_t,a_t)\in\mathcal{S}\times\mathcal{A}\backslash\mathcal{C}_t}d^{\dagger\pi^\star}_t(s_t,a_t)\leq \sum_{h=2}^{H+1}\sum_{t=1}^{h-1}\sum_{(s_t,a_t)\in\mathcal{S}\times\mathcal{A}\backslash\mathcal{C}_t}d^{\pi^\star}_t(s_t,a_t)\\
=&\sum_{h=2}^{H+1}\sum_{t=1}^{h-1}\sum_{(s_t,a_t):d^\mu_t(s_t,a_t)=0}d^{\pi^\star}_t(s_t,a_t)=0
\end{align*}
due to Lemma~\ref{lem:geq_dagger},\ref{lem:d_s_dagger}. Therefore, the gap $\sum_{h=1}^Hd^{\dagger\pi^\star}_h(s^\dagger_h)$ vanishes when Assumption~\ref{assum:single_concen} is true. Also, in this case $M^\dagger$ can be replaced by a $M'$, where $M'$ is the sub-MDP induced by $\mu$. \emph{i.e.}, $M'=\bigcup_{h=1}^H \mathcal{S}_h\times\mathcal{A}_h$ with $\mathcal{S}_h\times\mathcal{A}_h=\mathcal{C}_h$.\footnote{In this sub-MDP, each state might have different number of actions!} The transitions and the rewards remain the same in $M^\dagger$.

Since there is certain $\pi^\star$ that is fully covered by $\mu$, for such $\pi^\star$ we have $V_h^{\pi^\star}|_{M}=V_h^{\pi^\star}|_{M'}$ for all $h\in[H]$. Also, in $M'$, $\mu$ can explore all the locations, therefore the probability transition to $s^\dagger_h$ is $0$. Hence, all the $d^\dagger,P^\dagger,r^\dagger,V^\dagger$ in Theorem~\ref{thm:APVI} are replaced by its original version.
	
\end{proof}

\begin{remark}
	Note even though the proof can essentially leverage the reduction of the proving procedure of Theorem~\ref{thm:AFRL}, for clear presentation of the algorithm design we still include the locations with no observation and set the severe penalty $\tilde{O}(H)$. This is different from its assumption-free version with $0$ penalty (also see Section~\ref{sec:inter} for related discussions).
\end{remark}

\section{Proof of Theorem~\ref{thm:adaptive_lower_bound}: Instance-dependent Lower Bound}\label{sec:proof_lower_bound}

Global minimax lower bound holds uniformly over large classes of models but lacks the characterization of individual instances. The more appropriate characterization of instance dependence is the (non-asymptotic) local minimax bound, which is originated from the local minimax framework \cite{cai2004adaptation} and recently used in \cite{khamaru2020temporal,khamaru2021instance}. The proof essentially relies on the reduction to the testing between two value instances with respect to the Hellinger distance. Specifically, the choice of the alternative instance should characterize the MDP problem we are considering and we fix the MDP problem (together with the behavior policy $\mu$) as: $\mathcal{P}:=(\mu,M)$ where $M=(\mathcal{S}, \mathcal{A}, P, r, H, d_1)$. Recall the local risk is defined as:
\[
\mathfrak{R}_{n}(\mathcal{P}):=\sup_{\mathcal{P}'\in\mathcal{G}}\inf_{\widehat{\pi}}\max_{\mathcal{Q}\in\{\mathcal{P},\mathcal{P}'\}}\sqrt{n}\cdot\E_{\mathcal{Q}}\left[v^\star(\mathcal{Q})-v^{\widehat{\pi}}\right]
\]
and $\mathcal{G}:=\{(\mu,M): \exists\; \pi^\star\;s.t. \; d^\mu_h(s,a)>0\;\text{if}\;d^{\pi^\star}_h(s,a)>0\}$.

For the ease of exposition, we use the notation $\mathcal{P}=(\mu,P_{1:H},r)$ instead of $(\mu,M)$. We start by considering the following two classes of alternatives instances:
\begin{equation}\label{eqn:S1S2}
\mathcal{S}_1=\{\mathcal{P}'=(\mu',P_{1:H}',r')\mid \mu'=\mu,r'=r,\mathcal{P}'\in\mathcal{G}\},\quad \mathcal{S}_2=\{\mathcal{P}'=(\mu',P_{1:H}',r')\mid \mu'=\mu,P_{1:H}'=P_{1:H},\mathcal{P}'\in\mathcal{G}\}.
\end{equation}
and define the restricted local risks w.r.t. $\mathcal{S}_i$:
\begin{equation}
\mathfrak{R}_{n}(\mathcal{P},\mathcal{S}_i):=\sup_{\mathcal{P}'\in\mathcal{S}_i}\inf_{\widehat{\pi}}\max_{\mathcal{Q}\in\{\mathcal{P},\mathcal{P}'\}}\sqrt{n}\cdot\E_{\mathcal{Q}}\left[v^\star(\mathcal{Q})-v^{\widehat{\pi}}\right],\quad i=1,2.
\end{equation}
Then it suffices to prove the following lemma:
\begin{lemma}\label{lem:lower_key}
	There exists an universal constant $C>0$ such that:
	\begin{align*}
	\mathfrak{R}_{n}(\mathcal{P},\mathcal{S}_1)&\geq C\cdot {\sum_{h=1}^H\sum_{(s_h,a_h)\in\mathcal{C}_h}d^{\pi^\star}_h(s_h,a_h)\cdot\sqrt{\frac{\mathrm{Var}_{P_{s_h,a_h}}(V^\star_{h+1})}{ \zeta\cdot d^\mu_h{(s_h,a_h)}}}},\\
	\mathfrak{R}_{n}(\mathcal{P},\mathcal{S}_2)&\geq C\cdot {\sum_{h=1}^H\sum_{(s_h,a_h)\in\mathcal{C}_h}d^{\pi^\star}_h(s_h,a_h)\cdot\sqrt{\frac{\mathrm{Var}_{s_h,a_h}(r_h)}{ \zeta\cdot d^\mu_h{(s_h,a_h)}}}}.\\
	\end{align*}
\end{lemma}

Given Lemma~\ref{lem:lower_key}, we can directly prove Theorem~\ref{thm:adaptive_lower_bound} as follows.

\begin{proof}[Proof of Theorem~\ref{thm:adaptive_lower_bound}]
	Given Lemma~\ref{lem:lower_key}, we directly have 
	\begin{align*}
	\mathfrak{R}_{n}(\mathcal{P})&\geq \max\{\mathfrak{R}_{n}(\mathcal{P},\mathcal{S}_1),\mathfrak{R}_{n}(\mathcal{P},\mathcal{S}_2)\}\\
	&\geq \frac{1}{2}\left(\mathfrak{R}_{n}(\mathcal{P},\mathcal{S}_1)+\mathfrak{R}_{n}(\mathcal{P},\mathcal{S}_2)\right)\\
	&\geq \frac{C}{2}\cdot \sum_{h=1}^H\sum_{(s_h,a_h)\in\mathcal{C}_h}d^{\pi^\star}_h(s_h,a_h)\cdot\left(\sqrt{\frac{\mathrm{Var}_{P_{s_h,a_h}}(V^\star_{h+1})}{\zeta\cdot  d^\mu_h{(s_h,a_h)}}}+\sqrt{\frac{\mathrm{Var}_{{s_h,a_h}}(r_{h})}{ \zeta\cdot d^\mu_h{(s_h,a_h)}}}\right)\\
	&\geq \frac{C}{2}\cdot \sum_{h=1}^H\sum_{(s_h,a_h)\in\mathcal{C}_h}d^{\pi^\star}_h(s_h,a_h)\cdot\sqrt{\frac{\mathrm{Var}_{P_{s_h,a_h}}(r_h+V^\star_{h+1})}{\zeta\cdot  d^\mu_h{(s_h,a_h)}}}\\
	\end{align*} 
	where the last inequality uses $\sqrt{a}+\sqrt{b}\geq \sqrt{a+b}$ for all $a,b\geq 0$ and $\mathrm{Var}_{P_{s_h,a_h}}(V^\star_{h+1})+\mathrm{Var}_{{s_h,a_h}}(r_h)=\mathrm{Var}_{P_{s_h,a_h}}(r_h+V^\star_{h+1})$ since $V^\star_{h+1}$ and $r_h$ are conditionally independent given $s_h,a_h$.
\end{proof}

For the rest of the section, we prove Lemma~\ref{lem:lower_key}.

\subsection{Reduction to two-point optimal-value estimations}

We first need the following lemma, which converts local learning risk to the following $\mathfrak{M}_n$ via the reduction from estimation to testing.

\begin{lemma}\label{lem:risk_bound}
	Define 
	\begin{align*}
	\mathfrak{M}_n(\mathcal{P},\mathcal{S}_1):&=\sup_{\mathcal{P}'\in\mathcal{S}_1}\left\{\sqrt{n}\cdot |v^\star(\mathcal{P})-v^\star(\mathcal{P}')|\middle| d_{Hel}(P^n,P'^n)^2\leq 0.4\right\}\\
	\mathfrak{M}_n(\mathcal{P},\mathcal{S}_2):&=\sup_{\mathcal{P}'\in\mathcal{S}_2}\left\{\sqrt{n}\cdot |v^\star(\mathcal{P})-v^\star(\mathcal{P}')|\middle| d_{Hel}(p_r^n,p'^n_r)\leq 0.4\right\}.
	\end{align*}
	then we have 
	\[
	\mathfrak{R}_n(\mathcal{P},\mathcal{S}_i)\geq \frac{1}{50}\mathfrak{M}_n(\mathcal{P},\mathcal{S}_i), \quad i=1,2.
	\]
\end{lemma}

\begin{proof}[Proof of Lemma~\ref{lem:risk_bound}]
	Indeed, denote $\mathcal{P}^n$ to be a product measure induced by $n$ trajectories from $\mathcal{P}$, then for any output $\hat{\pi}$ by the averaged risk we have:
	\begin{align*}
		\max_{\mathcal{Q}\in\{\mathcal{P},\mathcal{P}'\}}\mathbb{E}_\mathcal{Q}\left[|v^\star(\mathcal{Q})-v^{\hat{\pi}}|\right]&\geq \frac{1}{2}\left(\mathbb{E}_{\mathcal{P}^n}\left[|v^\star(\mathcal{P})-v^{\hat{\pi}}|\right]+\mathbb{E}_{\mathcal{P}^{'n}}\left[|v^\star(\mathcal{P}')-v^{\hat{\pi}}|\right]\right)\\
		&\geq \frac{1}{2}\delta\left[\mathcal{P}^n\left(|v^\star(\mathcal{P})-v^{\hat{\pi}}|\geq\delta\right)+\mathcal{P}'^n\left(|v^\star(\mathcal{P}')-v^{\hat{\pi}}|\geq\delta\right)\right],
	\end{align*}
	where the last inequality is by Markov inequality. Now choose $\delta = \frac{1}{2}\cdot |v^\star(\mathcal{P})-v^\star(\mathcal{P}')|$, we have $|v^\star(\mathcal{P})-v^{\hat{\pi}}|\leq \delta$ implies $|v^\star(\mathcal{P}')-v^{\hat{\pi}}|\geq \delta$, therefore above 
	\begin{equation}\label{eqn:Hel_final}
	\begin{aligned}
	&= \frac{1}{2}\delta\cdot\left[1-\mathcal{P}^n\left(|v^\star(\mathcal{P})-v^{\hat{\pi}}|<\delta\right)+\mathcal{P}'^n\left(|v^\star(\mathcal{P}')-v^{\hat{\pi}}|\geq\delta\right)\right]\\
	&\geq \frac{1}{2}\delta\cdot\left[1-\mathcal{P}^n\left(|v^\star(\mathcal{P}')-v^{\hat{\pi}}|\geq\delta\right)+\mathcal{P}'^n\left(|v^\star(\mathcal{P}')-v^{\hat{\pi}}|\geq\delta\right)\right]\\
	&\geq \frac{1}{2}\delta\cdot \left[1-\norm{\mathcal{P}^n-\mathcal{P}'^n}_{\mathrm{TV}}\right]\geq \frac{1}{2}\delta\cdot\left[1-\sqrt{2} \cdot d_{Hel}(\mathcal{P}^n,\mathcal{P}'^n)\right]
	\end{aligned}
	\end{equation}
	
	Then plug in the condition for $d_{Hel}(\mathcal{P}^n,\mathcal{P}'^n)$ we obtain the result. The proof for the second result is similar.
	
\end{proof}

\subsection{Instance-dependent lower bound}

Now we complete the proof by the following lemma. Combing Lemma~\ref{lem:risk_bound_second} and Lemma~\ref{lem:risk_bound}, we finish the proof of Lemma~\ref{lem:lower_key}.
\begin{lemma}\label{lem:risk_bound_second}
There exists an universal constant $C>0$ such that:
\begin{align*}
\mathfrak{M}_{n}(\mathcal{P},\mathcal{S}_1)&\geq C\cdot {\sum_{h=1}^H\sum_{(s_h,a_h)\in\mathcal{C}_h}d^{\pi^\star}_h(s_h,a_h)\cdot\sqrt{\frac{\mathrm{Var}_{P_{s_h,a_h}}(V^\star_{h+1})}{ \zeta\cdot d^\mu_h{(s_h,a_h)}}}},\\
\mathfrak{M}_{n}(\mathcal{P},\mathcal{S}_2)&\geq C\cdot {\sum_{h=1}^H\sum_{(s_h,a_h)\in\mathcal{C}_h}d^{\pi^\star}_h(s_h,a_h)\cdot\sqrt{\frac{\mathrm{Var}_{s_h,a_h}(r_h)}{ \zeta\cdot d^\mu_h{(s_h,a_h)}}}}.\\
\end{align*}
\end{lemma}

\begin{proof}[Proof of Lemma~\ref{lem:risk_bound_second}]
So far we haven't leveraged the specific structure of instance $\mathcal{P}=(\mu,P_{1:H},r)$. Now we define $P'$ as follows ($\forall h,s_h,a_h$):\footnote{In below, it suffices to only consider the instance where $n_{s_h,a_h}\cdot\mathrm{Var}_{P_{s_h,a_h}}(V^\star_{h+1})>0$ since, 1. when $\mathrm{Var}_{P_{s_h,a_h}}(V^\star_{h+1})=0$, the numerator is also $0$ therefore by convention by we can define ratio to be $0$; 2. if $n_{s_h,a_h}=0$, then with high probability $d^\mu_h(s_h,a_h)=0$, in this case the transition $P(\cdot|s_h,a_h)$ does not matter since $d^{\pi^\star}_h(s_h,a_h)=0$ by theorem condition.}
\begin{equation}\label{eqn:definition_of_P_prime}
P'_h(s_{h+1}|s_h,a_h)=P_h(s_{h+1}|s_h,a_h)+\frac{P_h(s_{h+1}|s_h,a_h)\left(V_{h+1}^\star(s_{h+1})-\E_{P_{s_h,a_h}}[V_{h+1}^\star]\right)}{8\sqrt{\zeta\cdot n_{s_h,a_h}\cdot \Var_{P_{s_h,a_h}}(V_{h+1}^\star)}}
\end{equation}
where $\zeta = H/\bar{d}_m$ and the alternative instance as $\mathcal{P}'=(\mu,P'_{1:H},r)$. Denote $\bar{Q}^\star$ to be the optimal $Q$-values under $\mathcal{P}'$ and $Q^\star$ the optimal $Q$-values under $\mathcal{P}$. $\bar{\pi}^\star$ is the optimal policy under $\mathcal{P}'$ and ${\pi}^\star$ is the optimal policy under $\mathcal{P}$. The proof has two steps.

\textbf{Step1: } we show 
\[
2\cdot d_{Hel}(\mathcal{P}^n,\mathcal{P}'^n)^2\leq 0.8
\]

Define $\tau=(s_1,a_1,s_2,a_2,\ldots, s_H,a_H)\sim P(s_1,a_1,s_2,a_2,\ldots, s_H,a_H)$ to be the trajectories, then
{\small
\begin{equation}\label{eqn:hel_p}
\begin{aligned}
&d_{Hel}(\mathcal{P}^n,\mathcal{P}'^n)^2=1-\int_{\tau^n}\sqrt{f_{P'^n}(\tau^n)\cdot f_{P^n}(\tau^n)}d\tau^n
=1-\prod_{i=1}^n \int_{\tau}\sqrt{f_{P'}(\tau)\cdot f_{P}(\tau)}d\tau\\
=&1-\prod_{i=1}^n\int_{s_1}\sqrt{d_1^{P'}(s_1)d_1^P(s_1)}\left(\int_{a_1}\sqrt{\mu_{P'}(a_1|s_1)\mu_P(a_1|s_1)}\left(\int_{s_2}\sqrt{{P'}_1(s_2|s_1,a_1)P_1(s_2|s_1,a_1)}\ldots ds_2\right)da_1\right)ds_1\\
\leq &1-\prod_{i=1}^n\prod_{h=1}^H \min_{s,a} \left(\int_{s'}\sqrt{{P'}_1(s'|s,a)P_1(s'|s,a)}\ldots ds'\right)=1-\prod_{i=1}^n \prod_{h=1}^H(1-\max_{s,a} d_{Hel}(\mathcal{P}_{h,s,a},\mathcal{P}'_{h,s,a})^2)\\
\leq &1-\prod_{i=1}^n\prod_{h=1}^H\left(1-\frac{1}{2nH}\right)=1-\left(1-\frac{1}{2nH}\right)^{nH}\leq 1-\frac{1}{\sqrt{e}}\leq 0.4,
\end{aligned}
\end{equation}
}where the second inequality uses independence of trajectories and the third equation comes from the conditional probability rule. The first inequality comes from $\int_a \sqrt{\mu_{P'}(a|s)\mu_P(a|s)}da=\int_a \mu_{P}(a|s)da=1$ and the second inequality comes from item 2 of Lemma~\ref{lem:local_instance} via Definition~\ref{def:Hellinger}. This verifies $P'$ satisfies the condition of $\mathfrak{M}_n(\mathcal{P},\mathcal{S}_1)$.

\textbf{Step2:} we show for this instance $\mathcal{P}'$ we have 
\[
\sqrt{n}|v^\star(\mathcal{P})-v^\star(\mathcal{P}')|\geq C\cdot {\sum_{h=1}^H\sum_{(s_h,a_h)\in\mathcal{C}_h}d^{\pi^\star}_h(s_h,a_h)\cdot\sqrt{\frac{\mathrm{Var}_{P_{s_h,a_h}}(V^\star_{h+1})}{ \zeta\cdot d^\mu_h{(s_h,a_h)}}}}.
\]

Define $\xi=\sup_{h,s_h,a_h,s_{h+1}, d^\mu_h(s_h,a_h)\cdot \Var_{P_{s_h,a_h}}(V_{h+1}^\star)>0}\frac{P_h(s_{h+1}|s_h,a_h)\left(V_{h+1}^\star(s_{h+1})-\E_{P_{s_h,a_h}}[V_{h+1}^\star]\right)}{4\sqrt{\zeta\cdot  d^\mu_h(s_h,a_h)\cdot \Var_{P_{s_h,a_h}}(V_{h+1}^\star)}}$,
{\small
\begin{align*}
Q^\star_1-\bar{Q}_1^\star&=\left(r_1+P^{{\pi}^\star}_1{Q}_2^\star\right)-\left(r_1+P'^{\bar{\pi}^\star}_1\bar{Q}_2^\star\right)\\
&\leq \left(r_1+P^{{\pi}^\star}_1{Q}_2^\star\right)-\left(r_1+P'^{{\pi}^\star}_1\bar{Q}_2^\star\right)\\
&=P^{{\pi}^\star}_1\left({Q}_2^\star-\bar{Q}_2^\star\right)+\left(P^{{\pi}^\star}_1-P'^{{\pi}^\star}_1\right)\bar{Q}_2^\star\\
&\leq P^{{\pi}^\star}_1P^{{\pi}^\star}_2\left({Q}_3^\star-\bar{Q}_3^\star\right)+P^{{\pi}^\star}_1\left(P^{{\pi}^\star}_2-P'^{{\pi}^\star}_2\right)\bar{Q}_3^\star+\left(P^{{\pi}^\star}_1-P'^{{\pi}^\star}_1\right)\bar{Q}_2^\star\\
&\leq \ldots\\
&\leq \sum_{h=1}^H P^{{\pi}^\star}_{1:h-1}\left(P^{{\pi}^\star}_h-P'^{{\pi}^\star}_h\right)\bar{Q}_{h+1}^\star=\sum_{h=1}^H P^{{\pi}^\star}_{1:h-1}\left(P_h-P'_h\right)\bar{Q}_{h+1}^\star(\cdot,\pi^\star(\cdot))\\
&\leq H\cdot \sup_{h,s_h,a_h,s_{h+1}}|P_h-P'_h|(s_{h+1}|s_h,a_h)\cdot \norm{\bar{Q}_{h+1}^\star}_\infty\\
&\leq H^2 \cdot \sup_{h,s_h,a_h,s_{h+1}}\frac{P_h(s_{h+1}|s_h,a_h)\left(V_{h+1}^\star(s_{h+1})-\E_{P_{s_h,a_h}}[V_{h+1}^\star]\right)}{8\sqrt{\zeta\cdot n_{s_h,a_h}\cdot \Var_{P_{s_h,a_h}}(V_{h+1}^\star)}}\\
&\leq H^2 \cdot \sup_{h,s_h,a_h,s_{h+1}}\frac{P_h(s_{h+1}|s_h,a_h)\left(V_{h+1}^\star(s_{h+1})-\E_{P_{s_h,a_h}}[V_{h+1}^\star]\right)}{8\sqrt{  \zeta\cdot n d^\mu_h(s_h,a_h)\cdot \Var_{P_{s_h,a_h}}(V_{h+1}^\star)}}=H^2\xi\sqrt{\frac{1}{n}}\\
\end{align*}
}where the first inequality is by $\bar{\pi}^\star$ is the optimal policy for $\bar{Q}^\star$ and second inequality is by recursively applying the \textbf{element-wisely} inequality for $Q^\star_2-\bar{Q}^\star_2$ and the fact that $P_1,P_2$ has non-negative coordinates. The last inequality is by Lemma~\ref{lem:chernoff_multiplicative}. By a similar calculation, we also have  
\begin{align*}
\bar{Q}_1^\star-Q^\star_1&=\left(r_1+P'^{\bar{\pi}^\star}_1\bar{Q}_2^\star\right)-\left(r_1+P^{{\pi}^\star}_1{Q}_2^\star\right)\\
&\leq \ldots\\
&\leq \sum_{h=1}^H P'^{\bar{\pi}^\star}_{1:h-1}\left(P'_h-P_h\right){Q}_{h+1}^\star(\cdot,\bar{\pi}^\star(\cdot))\leq H^2\xi\sqrt{\frac{1}{n}}\\
\end{align*}
and combing the above two we obtain $\norm{\bar{Q}_1^\star-Q_1^\star}_\infty\leq H^2\xi\sqrt{\frac{1}{n}}$, and (by similar computation and the union bound by Lemma~\ref{lem:chernoff_multiplicative}) further holds true for all $\bar{Q}^\star_h,Q^\star_h$'s, with high probability
\begin{equation}\label{eqn:crude_diff}
\max_{h}\norm{\bar{Q}_1^\star-Q_1^\star}_\infty\leq H^2\xi\sqrt{\frac{1}{n}}
\end{equation}

Now by the calculation again,
\begin{equation}\label{eqn:decomp}
\begin{aligned}
\bar{Q}_1^\star-Q^\star_1&=r_1+P'^{\bar{\pi}^\star}_1\bar{Q}_2^\star-\left(r_1+P^{{\pi}^\star}_1{Q}_2^\star\right)\\
&\geq r_1+P'^{{\pi}^\star}_1\bar{Q}_2^\star-\left(r_1+P^{{\pi}^\star}_1{Q}_2^\star\right)\\
&=P'^{{\pi}^\star}_1\left(\bar{Q}_2^\star-{Q}_2^\star\right)+\left(P'^{{\pi}^\star}_1-P^{{\pi}^\star}_1\right){Q}_2^\star\\
&=\left(P'^{{\pi}^\star}_1-P^{{\pi}^\star}_1\right)\left(\bar{Q}_2^\star-{Q}_2^\star\right)+P^{{\pi}^\star}_1\left(\bar{Q}_2^\star-{Q}_2^\star\right)+\left(P'^{{\pi}^\star}_1-P^{{\pi}^\star}_1\right){Q}_2^\star\\
&\geq\left(P'^{{\pi}^\star}_1-P^{{\pi}^\star}_1\right)\left(\bar{Q}_2^\star-{Q}_2^\star\right)+P^{{\pi}^\star}_1\left(P'^{{\pi}^\star}_2-P^{{\pi}^\star}_2\right)\left(\bar{Q}_3^\star-{Q}_3^\star\right)\\
&+\left(P'^{{\pi}^\star}_1-P^{{\pi}^\star}_1\right){Q}_2^\star+P^{{\pi}^\star}_1\left(P'^{{\pi}^\star}_2-P^{{\pi}^\star}_2\right){Q}_3^\star\\
&+P^{{\pi}^\star}_1P^{{\pi}^\star}_2\left(\bar{Q}_3^\star-{Q}_3^\star\right)\\
&\geq\ldots\\
&\geq\sum_{h=1}^H P^{{\pi}^\star}_{1:h-1}\left(P'^{{\pi}^\star}_h-P^{{\pi}^\star}_h\right)\left(\bar{Q}_{h+1}^\star-{Q}_{h+1}^\star\right)+\sum_{h=1}^HP^{{\pi}^\star}_{1:h-1}\left(P'^{{\pi}^\star}_h-P^{{\pi}^\star}_h\right){Q}_{h+1}^\star\\
&=\sum_{h=1}^H P^{{\pi}^\star}_{1:h-1}\left(P'^{{\pi}^\star}_h-P^{{\pi}^\star}_h\right)\left(\bar{Q}_{h+1}^\star-{Q}_{h+1}^\star\right)+\sum_{h=1}^HP^{{\pi}^\star}_{1:h-1}\left(P'_h-P_h\right){V}_{h+1}^\star
\end{aligned}
\end{equation}
where the second inequality recursively applies $\bar{Q}_2^\star-Q^\star_2\geq P'^{{\pi}^\star}_2\left(\bar{Q}_3^\star-{Q}_3^\star\right)+\left(P'^{{\pi}^\star}_2-P^{{\pi}^\star}_2\right){Q}_3^\star$ and the above is equivalent to 
\[
\sum_{h=1}^H P^{{\pi}^\star}_{1:h-1}\left(P^{{\pi}^\star}_h-P'^{{\pi}^\star}_h\right)\left(\bar{Q}_{h+1}^\star-{Q}_{h+1}^\star\right)+\bar{Q}_1^\star-Q^\star_1\geq \sum_{h=1}^HP^{{\pi}^\star}_{1:h-1}\left(P'_h-P_h\right){V}_{h+1}^\star.
\]

Now by Lemma~\ref{lem:local_instance} item 3, $\sum_{h=1}^HP^{{\pi}^\star}_{1:h-1}\left(P'_h-P_h\right){V}_{h+1}^\star\geq 0$, therefore multiply the initial distribution on both sides and take the absolute value on the left hand side to get 
\begin{equation}\label{eqn:lower_inter}
\sum_{h=1}^H d^{{\pi}^\star}_{1:h-1}\left|P^{{\pi}^\star}_h-P'^{{\pi}^\star}_h\right|\left|\bar{Q}_{h+1}^\star-{Q}_{h+1}^\star\right|+|\bar{v}^\star-v^\star|\geq \sum_{h=1}^Hd^{{\pi}^\star}_{1:h-1}\left(P'_h-P_h\right){V}_{h+1}^\star
\end{equation}
On one hand, 
\[
\sum_{h=1}^H d^{{\pi}^\star}_{1:h-1}\left|P^{{\pi}^\star}_h-P'^{{\pi}^\star}_h\right|\left|\bar{Q}_{h+1}^\star-{Q}_{h+1}^\star\right|\leq H\cdot \xi\sqrt{\frac{1}{n}} \sup_h\norm{\bar{Q}^\star_h-Q^\star_h}_\infty\leq H^3\xi^2\frac{1}{n},
\]
One the other hand, 
\begin{align*}
&\sum_{h=1}^Hd^{{\pi}^\star}_{1:h-1}\left(P'_h-P_h\right){V}_{h+1}^\star=\sum_{h=1}^H \sum_{s_h,a_h}d^{\pi^\star}_h(s_h,a_h)\sum_{s_{h+1}}(P'_h-P_h)(s_{h+1}|s_h,a_h)V^\star_{h+1}(s_{h+1})\\
&=\sum_{h=1}^H \sum_{s_h,a_h}d^{\pi^\star}_h(s_h,a_h)\sum_{s_{h+1}}(P'_h-P_h)(s_{h+1}|s_h,a_h)\left(V_{h+1}^\star(s_{h+1})-\E_{P_{s_h,a_h}}[V_{h+1}^\star]\right)\\
&=\sum_{h=1}^H \sum_{s_h,a_h}d^{\pi^\star}_h(s_h,a_h)\sum_{s_{h+1}}\frac{P_h(s_{h+1}|s_h,a_h)\left(V_{h+1}^\star(s_{h+1})-\E_{P_{s_h,a_h}}[V_{h+1}^\star]\right)^2}{8\sqrt{\zeta\cdot n_{s_h,a_h}\cdot \Var_{P_{s_h,a_h}}(V_{h+1}^\star)}}\\
&=\sum_{h=1}^H \sum_{s_h,a_h}d^{\pi^\star}_h(s_h,a_h)\sqrt{\frac{\Var_{P_{s_h,a_h}}(V_{h+1}^\star)}{8^2\zeta\cdot n_{s_h,a_h}}}\geq \sum_{h=1}^H \sum_{s_h,a_h}d^{\pi^\star}_h(s_h,a_h)\sqrt{\frac{\Var_{P_{s_h,a_h}}(V_{h+1}^\star)}{96 \zeta\cdot n d^\mu_h(s_h,a_h)}}\\
\end{align*}
The last step is by Lemma~\ref{lem:chernoff_multiplicative}. Combing those two into \eqref{eqn:lower_inter}, we finally obtain
\begin{equation}\label{eqn:lower_final}
\begin{aligned}
|\bar{v}^\star-v^\star|&\geq \sum_{h=1}^H \sum_{s_h,a_h}d^{\pi^\star}_h(s_h,a_h)\sqrt{\frac{\Var_{P_{s_h,a_h}}(V_{h+1}^\star)}{96\zeta\cdot n d^\mu_h(s_h,a_h)}}-H^3\xi^2\frac{1}{n}\\
&\geq \frac{1}{2}\sum_{h=1}^H \sum_{s_h,a_h}d^{\pi^\star}_h(s_h,a_h)\sqrt{\frac{\Var_{P_{s_h,a_h}}(V_{h+1}^\star)}{96\zeta\cdot n d^\mu_h(s_h,a_h)}}
\end{aligned}
\end{equation}
as long as $n\geq \frac{4H^6\xi^4}{\left(\sum_{h=1}^H \sum_{s_h,a_h}d^{\pi^\star}_h(s_h,a_h)\sqrt{\frac{\Var_{P_{s_h,a_h}}(V_{h+1}^\star)}{96\zeta d^\mu_h(s_h,a_h)}}\right)^2}$.

By \eqref{eqn:lower_final}, we finish the proof of Step2.

\textbf{The result for the reward can be similarly derived in the following sense.} First, define the perturbed mean reward as: 
\begin{equation}
r'_h(s_h,a_h)=r_h(s_h,a_h)+\frac{\sigma_r}{2\sqrt{ \zeta\cdot n_{s_h,a_h}}},
\end{equation}
where $\sigma_r>0$ is a parameter and the realization of reward is sampled from normal $r|_{s_h,a_h}\sim\mathcal{N}(r_h(s_h,a_h), \sigma_r^2)$ and $r'|_{s_h,a_h}\sim\mathcal{N}(r'_h(s_h,a_h), \sigma_r^2)$. In this scenario, similar to \eqref{eqn:hel_p}, we have 
\begin{equation}
\begin{aligned}
&d_{Hel}(\mathcal{P}^n,\mathcal{P}'^n)^2=1-\int_{\tau^n}\sqrt{f_{P'^n}(\tau^n)\cdot f_{P^n}(\tau^n)}d\tau^n\\
\leq &1-\prod_{i=1}^n \prod_{h=1}^H(1-\max_{s,a} d_{Hel}(r_{h,s,a},{r}'_{h,s,a})^2)\\
\leq &1-\prod_{i=1}^n \prod_{h=1}^H(1-\max_{s,a} D_{KL}(r_{h,s,a},{r}'_{h,s,a})^2)\\
= & 1-\prod_{i=1}^n \prod_{h=1}^H(1-\max_{s,a} \frac{|r'_h(s,a)-r_h(s,a)|^2}{2\sigma^2_r})=1-\prod_{i=1}^n \prod_{h=1}^H(1-\max_{s_h,a_h} \frac{1}{8\zeta\cdot n_{s_h,a_h}})\\
\leq &1-\prod_{i=1}^n\prod_{h=1}^H\left(1-\frac{1}{2nH}\right)=1-\left(1-\frac{1}{2nH}\right)^{nH}\leq 1-\frac{1}{\sqrt{e}}\leq 0.4,
\end{aligned}
\end{equation}
and similar to  \eqref{eqn:decomp}
\begin{equation}
\begin{aligned}
\bar{Q}_1^\star-Q^\star_1&=r'_1+P^{\bar{\pi}^\star}_1\bar{Q}_2^\star-\left(r_1+P^{{\pi}^\star}_1{Q}_2^\star\right)\\
&\geq r'_1+P^{{\pi}^\star}_1\bar{Q}_2^\star-\left(r_1+P^{{\pi}^\star}_1{Q}_2^\star\right)\\
&=P^{{\pi}^\star}_1\left(\bar{Q}_2^\star-{Q}_2^\star\right)+\left(r'_1-r_1\right)\\
&\geq\ldots\\
&\geq\sum_{h=1}^H P^{{\pi}^\star}_{1:h-1}\left(r'_h-r_h\right)=\sum_{h=1}^H P^{{\pi}^\star}_{1:h-1}\sqrt{\frac{\sigma^2_r}{2\zeta\cdot n_{s_h,a_h}}}\\
\end{aligned}
\end{equation}
the last step is to average over $d_1$ and use Lemma~\ref{lem:chernoff_multiplicative}.
\end{proof}

\section{Minimax lower bound}\label{sec:minimax_lower}

\begin{theorem}[Adaptive minimax lower bound]\label{thm:minimax_lower_bound}
	Recall for each individual instance $(\mu,M)$, $\bar{d}_m:=\min_{h\in[H]}\{d^\mu_h(s_h,a_h):d^\mu_h(s_h,a_h)>0\}$ and $\mathcal{D}$ consists of $n$ episodes. Now consider a class of problem family $\mathcal{G}:=\{(\mu,M): \exists \pi^\star\;s.t. \; d^\mu_h(s,a)>0\;\text{if}\;d^{\pi^\star}_h(s,a)>0\}$. Let $\widehat{\pi}$ be the output of any algorithm on $\mathcal{D}$. Then there exists universal constants $c_0,p,C>0$, such that if $n\geq c_0 \cdot 1/\bar{d}_m\cdot\log(HSA/p)$, with constant probability $p>0$,
	\begin{equation}\label{eqn:minimax_lower}
	\inf_{\widehat{\pi}}\sup_{(\mu,M)\in\mathcal{G}}\frac{\E_{\mu,M}\left[v^\star-v^{\widehat{\pi}}\right]}{\sum_{h=1}^H\sum_{(s_h,a_h)\in\mathcal{C}_h}d^{\pi^\star}_h(s_h,a_h)\cdot\sqrt{\frac{\mathrm{Var}_{P_{s_h,a_h}}(r_h+V^\star_{h+1})}{ n\cdot d^\mu_h{(s_h,a_h)}}}}\geq C.
	\end{equation}
\end{theorem}

For completeness, we provide the proof of minimax lower bound. The proof uses the hard instance construction in \cite{jin2020pessimism}. 

\subsection{Proof of Theorem~\ref{thm:minimax_lower_bound}}

\textbf{Construction of the hard MDP instances.} We define a family of MDPs where each MDP instance within in the family has three states $s_1,s_+,s_-$ and $A$ actions $a_1,\ldots,a_A$. The initial state is always $s_1$. The transition kernel at time $1$ will transition $s_1$ to either $s_+$ or $s_-$ depending on three probabilities {$p_1,p_2,p_3(=\min\{p_1,p_2\})$} and the transition kernel at time $h\geq 2$ will deterministically transition back to itself, \emph{i.e.}
\begin{align*}
P_1(s_+|s_1,a_1)=p_1,\;\;P_1(s_1|s_1,a_1)=1-p_1,\\
P_1(s_+|s_1,a_2)=p_2,\;\;P_1(s_1|s_1,a_2)=1-p_2,\\
P_1(s_+|s_1,a_j)=p_3,\;\;P_1(s_1|s_1,a_j)=1-p_3,\;\;\forall j\geq 3,\\
P_h(s_+|s_+,a)=P_h(s_-|s_-,a)=1,\;\;\forall h\geq 2,a\in\mathcal{A}.
\end{align*}
The state $s_+$ always receives reward $1$ regardless of the action and the state $s_1$, $s_-$ will always have reward $0$. We denote such an instance as $M(p_1,p_2,p_3)$.

\textbf{Bounding the suboptimality gap.} By construction, the optimal policy at time $h=1$ will be $\pi^\star_1(s_1)=a_1$ if $p_1>p_2$ and $\pi^\star_1(s_1)=a_2$ if $p_1<p_2$. For optimal policy for $h\geq 2$ is arbitrary. Therefore, for those instances (denote $j^{*}=\underset{j \in\{1,2\}}{\arg \max } p_{j}$)
\begin{align*}
v^\star=0+P_1(s_+|s_1,\pi^\star_1(s_1))\cdot 1+\sum_{h=3}^H P_1(s_+|s_1,\pi^\star_1(s_1))\cdot \ldots\cdot P_{h-1}(s_+|s_+,\pi^\star_{h-1}(s_+))\cdot 1=p_{j^\star}\cdot (H-1)\\
\end{align*}
and
\begin{align*}
v^{\widehat{\pi}}=&0+\left[\sum_{j=1}^A p_j\cdot\widehat{\pi}_1(a_j|s_1)\right]\cdot 1+\sum_{h=3}^H\left[\sum_{j=1}^A p_j\cdot\widehat{\pi}_1(a_j|s_1)\right]\cdot\ldots\cdot P_{h-1}(s_+|s_+,\widehat{\pi}_{h-1}(s_+))\cdot 1\\
=&\left[\sum_{j=1}^A p_j\cdot\widehat{\pi}_1(a_j|s_1)\right]\cdot (H-1).
\end{align*}
Therefore the suboptimality gap 
\[
v^\star-v^{\widehat{\pi}}=p_{j^\star}\cdot (H-1)-\sum_{j=1}^A p_j\cdot\widehat{\pi}_1(a_j|s_1)\cdot (H-1).
\]
Let us further denote $p^\star=p_{j^\star}$ and denote the rest of the probabilities as $p$ (\emph{i.e.} if $p_1>p_2$, then $p_1=p^\star$, $p_2=\ldots=p_A=p$; if $p_1<p_2$, then $p_2=p^\star$, $p_1=p_3=\ldots=p_A=p$), then 
\begin{align*}
v^\star-v^{\widehat{\pi}}=&p_{j^\star}\cdot (H-1)-\sum_{j=1}^A p_j\cdot\widehat{\pi}_1(a_j|s_1)\cdot (H-1)\\
=&(H-1)\cdot\left(p^\star-p^\star\widehat{\pi}_1(a_1^\star|s_1)-\sum_{j=2}^A p_j\cdot\widehat{\pi}_1(a_j|s_1)\right)\\
=&(H-1)\cdot\left(1-\widehat{\pi}_1(a_1^\star|s_1)\right)(p^\star-p).
\end{align*}
Let {\small$\mathcal{D}=\left\{\left(s_{h}^{\tau}, a_{h}^{\tau}, r_{h}^{\tau}, s_{h+1}^{\tau}\right)\right\}_{\tau\in[n]}^{h \in[H]}$} is coming from $\mu$ where $\mu$ satisfies $(\mu,M)$ belongs to $\mathcal{G}:=\{(\mu,M): \exists \pi^\star\;s.t. \; d^\mu_h(s,a)>0\;\text{if}\;d^{\pi^\star}_h(s,a)>0\}$.
Define $n_j=\sum_{\tau=1}^n \mathbf{1}[a_1^\tau=a_j]$. Consider two MDPs $\mathcal{M}_1=M(p^\star,p,p)$ and $M_2=M(p,p^\star,p)$, then 
\begin{equation}\label{eqn:lower_1}
\begin{aligned}
&\sup_{l\in\{1,2\}}\sqrt{n_l}\cdot\E_{\mu,\mathcal{M}_l}\left[v^\star-v^{\widehat{\pi}}\right]\\
\geq&\frac{\sqrt{n_1n_2}}{\sqrt{n_1}+\sqrt{n_2}}\cdot\left(\E_{\mu,\mathcal{M}_1}\left[v^\star-v^{\widehat{\pi}}\right]+\E_{\mu,\mathcal{M}_2}\left[v^\star-v^{\widehat{\pi}}\right]\right)\\
=&\frac{\sqrt{n_1n_2}}{\sqrt{n_1}+\sqrt{n_2}}\cdot(p^\star-p)\cdot (H-1) \cdot\left(\E_{\mu,\mathcal{M}_1}[1-\widehat{\pi}_1(a_1|s_1)]+\E_{\mu,\mathcal{M}_2}[1-\widehat{\pi}_1(a_2|s_1)]\right)
\end{aligned}
\end{equation}
where the first inequality uses $\max \{x, y\} \geq a \cdot x+(1-a) y$ for any $a\in[0,1],\;x,y\geq 0$. Importantly, we choose the above $\mu$ to satisfy $\mu_1(a_1|s_1)>0$, $\mu_1(a_2|s_1)>0$. In this scenario, it satisfies $d^\mu_h(s,a)>0\;\text{if}\;d^{\pi^\star}_h(s,a)>0$ (since $\mu$ can reach both $a_1$ and $a_2$, hence $(\mu,M)\in\mathcal{G}$ for $M$ to be either $\mathcal{M}_1$ or $\mathcal{M}_2$) and by the condition $n\geq \widetilde{O}(1/\bar{d}_m)$, $n_1,n_2>0$ with high probability (depends on $p$) by Chernoff bound hence the above inequality apply. 

Now define the (randomized) test function:
\[
\psi(\widehat{\pi})=\mathbf{1}\{a\neq a_1\},\quad \text{where} \; a\sim\widehat{\pi}_1(\cdot|s_1),
\]
then 
\begin{equation}\label{eqn:lower_2}
\begin{aligned}
&\E_{\mu,\mathcal{M}_1}[1-\widehat{\pi}_1(a_1|s_1)]+\E_{\mu,\mathcal{M}_2}[1-\widehat{\pi}_1(a_2|s_1)]=\E_{\mu,\mathcal{M}_1}[\mathbf{1}\{\psi(\widehat{\pi})=1\}]+\E_{\mu,\mathcal{M}_2}[\mathbf{1}\{\psi(\widehat{\pi})=0\}]\\
&\geq 1-\mathrm{TV}(\P_{\mathcal{D}\sim(\mu,\mathcal{M}_1)},\P_{\mathcal{D}\sim(\mu,\mathcal{M}_2)})\geq 1-\sqrt{\mathrm{KL}(\P_{\mathcal{D}\sim(\mu,\mathcal{M}_1)}||\P_{\mathcal{D}\sim(\mu,\mathcal{M}_2)})/2}.
\end{aligned}
\end{equation}
Now we apply the following lemma in \cite{jin2020pessimism}.

\begin{lemma}[(C.17) in \cite{jin2020pessimism}]\label{lem:lower}
	Let $n_1,n_2\geq 4$ and $\min\{n_1/n_2,n_2/n_1\}>c$ for some absolute constant $c$. Then there exists some $p\neq p^\star\in[1/4,3/4]$ such that $p^\star-p=\frac{\sqrt{3}}{4\sqrt{2(n_1+n_2)}}$ and 
	\[
	\mathrm{KL}(\P_{\mathcal{D}\sim(\mu,\mathcal{M}_1)}||\P_{\mathcal{D}\sim(\mu,\mathcal{M}_2)})\leq 1/2.
	\]
\end{lemma}
Note the condition of Lemma~\ref{lem:lower} is satisfied with high probability by the condition $n\geq \widetilde{O}(1/\bar{d}_m)$ and the design that $\mu_1(a_1|s_1)>0$, $\mu_1(a_2|s_1)>0$. Hence, by \eqref{eqn:lower_1}, \eqref{eqn:lower_2} and Lemma~\ref{lem:lower} we have 
\begin{equation}\label{eqn:lower_inter_minimax}
\sup_{l\in\{1,2\}}\sqrt{n_l}\cdot\E_{\mu,\mathcal{M}_l}\left[v^\star-v^{\widehat{\pi}}\right]\geq C'(H-1),
\end{equation}
where $C'=\sqrt{3}/(\sqrt{c}+1)\sqrt{8c(c+1)}>0$. 

\subsubsection{The intrinsic quantity in the hard instances}\label{subsec:lower}

Note under the family $M(p_1,p_2,p_3)$, the optimal values
\[
V^\star_2(s_+)=H-1,\;V^\star_2(s_-)=0,\;
\]
also, since $P_h$ is deterministic for $h\geq 2$, therefore 
\[
\mathrm{Var}_{P_h}(r_h+V^\star_{h+1})=0,\quad \forall h\geq 2.
\] 

For convenience, let us assume the MDP is $M(p^\star,p,p)$. Then in this case
\begin{equation}\label{eqn:lower_inter_2}
\begin{aligned}
&\sum_{h=1}^H\sum_{(s_h,a_h)\in\mathcal{C}_h}d^{\pi^\star}_h(s_h,a_h)\cdot\sqrt{\frac{\mathrm{Var}_{P_{s_h,a_h}}(r_h+V^\star_{h+1})}{ n\cdot d^\mu_h{(s_h,a_h)}}}\\
= &\sum_{(s_1,a_1)\in\mathcal{C}_1}d^{\pi^\star}_1(s_1,a_1)\cdot\sqrt{\frac{\mathrm{Var}_{P_{s_1,a_1}}(r_1+V^\star_{2})}{ n\cdot d^\mu_1{(s_1,a_1)}}}\\
=&d^{\pi^\star}_1(s_1,a_1)\cdot\sqrt{\frac{\mathrm{Var}_{P_{s_1,a_1}}(r_1+V^\star_{2})}{ n\cdot d^\mu_1{(s_1,a_1)}}}+d^{\pi^\star}_1(s_1,a_2)\cdot\sqrt{\frac{\mathrm{Var}_{P_{s_1,a_2}}(r_1+V^\star_{2})}{ n\cdot d^\mu_1{(s_1,a_2)}}}\\
=&d^{\pi^\star}_1(s_1,a_1)\cdot\sqrt{\frac{\mathrm{Var}_{P_{s_1,a_1}}(r_1+V^\star_{2})}{ n\cdot d^\mu_1{(s_1,a_1)}}}+0\cdot\sqrt{\frac{\mathrm{Var}_{P_{s_1,a_2}}(r_1+V^\star_{2})}{ n\cdot d^\mu_1{(s_1,a_2)}}}\\
=&1\cdot\sqrt{\frac{\mathrm{Var}_{P_{s_1,a_1}}(r_1+V^\star_{2})}{ n\cdot d^\mu_1{(s_1,a_1)}}}\leq \sqrt{\frac{3\mathrm{Var}_{P_{s_1,a_1}}(r_1+V^\star_{2})}{2 n_{s_1,a_1}}}\\
=&\sqrt{\frac{3\mathrm{Var}_{P_{s_1,a_1}}(V^\star_{2})}{ 2n_{s_1,a_1}}}=\sqrt{\frac{3p^\star(1-p^\star)(H-1)^2}{2 n_{s_1,a_1}}}\leq \frac{H-1}{\sqrt{8n_{s_1,a_1}/3}}=\frac{H-1}{\sqrt{8n_{1}/3}}
\end{aligned}
\end{equation}
where the first inequality uses the Chernoff bound.

\textbf{Finish the proof.}
By \eqref{eqn:lower_inter_minimax} and \eqref{eqn:lower_inter_2}, we have with constant probability, for any arbitrary algorithm $\widehat{\pi}$,
\begin{align*}
&\sup_{l\in\{1,2\}}\frac{\E_{\mu,\mathcal{M}_l}[v^\star-v^{\widehat{\pi}}]}{\sum_{h=1}^H\sum_{(s_h,a_h)\in\mathcal{C}_h}d^{\pi^\star}_h(s_h,a_h)\cdot\sqrt{\frac{\mathrm{Var}_{P_{s_h,a_h}}(r_h+V^\star_{h+1})}{ n\cdot d^\mu_h{(s_h,a_h)}}}}\\
\geq&\sup_{l\in\{1,2\}}\frac{\E_{\mu,\mathcal{M}_l}[v^\star-v^{\widehat{\pi}}]}{\frac{H-1}{\sqrt{8n_{l}/3}}}\\
=&\sup_{l\in\{1,2\}}\sqrt{\frac{8}{3}}\cdot\frac{\sqrt{n_l}\cdot\E_{\mu,\mathcal{M}_l}[v^\star-v^{\widehat{\pi}}]}{H-1}\geq \sqrt{\frac{8}{3}}C':=C,\\
\end{align*}
this concludes the proof.

\begin{remark}\label{remark_lower}
	In the proofing procedure \ref{subsec:lower}, we can actually get rid of the hard instance construction of \cite{jin2020pessimism} by setting the hard instances at any time step $t$, concretely:
	\begin{itemize}
		\item From time $1$ to $t-1$, there is only one absorbing state $s_a$ with reward $0$;
		\item At time $t$, $s_a$ can transition to either $s_+$ or $s_-$ follow the same transition as above; from $t+1$ to $H$, $s_+$ and $s_-$ are absorbing states;
		\item 	$s_+$ has reward $1$ and $s_-$ has reward $0$.
	\end{itemize}
	Those instances still validate the intrinsic bound is required due to the fact that there is at least one stochastic transition kernel. This finding is interesting as it reveals the intrinsic bound is only ``hard'' for offline reinforcement learning when there are stochasticity in the dynamic. Under the deterministic family, those hard instances fail and we enter the faster convergence regime. 
	
\end{remark}




\section{Discussions and missing derivations in Section~\ref{sec:intrinsic}}\label{sec:missing_dev}

We omit the $\widetilde{O}$ notation in the derivations for the simplicity.

\subsection{Derivation in Section~\ref{subsec:one}}

When the uniform data-coverage is satisfied,
\begin{align*}
	v^\star-v^{\widehat{\pi}}\lesssim&\sum_{h=1}^H\sum_{(s_h,a_h)\in\mathcal{C}_h}d^{\pi^\star}_h(s_h,a_h)\cdot\sqrt{\frac{\mathrm{Var}_{P_{s_h,a_h}}(r_h+V^\star_{h+1})}{ n\cdot d^\mu_h{(s_h,a_h)}}}\\
	\leq &\sqrt{\frac{1}{n d_m}}\sum_{h=1}^H\sum_{(s_h,a_h)\in\mathcal{C}_h}d^{\pi^\star}_h(s_h,a_h)\cdot\sqrt{{\mathrm{Var}_{P_{s_h,a_h}}(r_h+V^\star_{h+1})}}\\
	\leq &\sqrt{\frac{1}{n d_m}}\sum_{h=1}^H\sum_{(s_h,a_h)\in\mathcal{S}\times\mathcal{A}}d^{\pi^\star}_h(s_h,a_h)\cdot\sqrt{{\mathrm{Var}_{P_{s_h,a_h}}(r_h+V^\star_{h+1})}}\\
	= &\sqrt{\frac{1}{n d_m}}\sum_{h=1}^H\sum_{(s_h,a_h)\in\mathcal{S}\times\mathcal{A}}\sqrt{d^{\pi^\star}_h(s_h,a_h)}\cdot\sqrt{{d^{\pi^\star}_h(s_h,a_h)\mathrm{Var}_{P_{s_h,a_h}}(r_h+V^\star_{h+1})}}\\
	\leq &\sqrt{\frac{1}{n d_m}}\sum_{h=1}^H\sqrt{\sum_{(s_h,a_h)\in\mathcal{S}\times\mathcal{A}}d^{\pi^\star}_h(s_h,a_h)}\cdot\sqrt{\sum_{(s_h,a_h)\in\mathcal{S}\times\mathcal{A}}{d^{\pi^\star}_h(s_h,a_h)\mathrm{Var}_{P_{s_h,a_h}}(r_h+V^\star_{h+1})}}\\
	= &\sqrt{\frac{1}{n d_m}}\sum_{h=1}^H\sqrt{\sum_{(s_h,a_h)\in\mathcal{S}\times\mathcal{A}}{d^{\pi^\star}_h(s_h,a_h)\mathrm{Var}_{P_{s_h,a_h}}(r_h+V^\star_{h+1})}}\\
	\leq &\sqrt{\frac{1}{n d_m}}\sqrt{\sum_{h=1}^H 1}\cdot\sqrt{\sum_{h=1}^H\sum_{(s_h,a_h)\in\mathcal{S}\times\mathcal{A}}{d^{\pi^\star}_h(s_h,a_h)\mathrm{Var}_{P_{s_h,a_h}}(r_h+V^\star_{h+1})}}\\
	\leq &\sqrt{\frac{1}{n d_m}} \sqrt{H}\cdot\sqrt{\mathrm{Var}_{\pi}\left[\sum_{t=1}^{H} r_{t}\right]}\leq \sqrt{\frac{H^3}{n d_m}}, 
\end{align*}
where we use the Cauchy inequality and Lemma~\ref{lem:H3toH2}.

\subsection{Uniform data-coverage in the time-invariant setting (Remark~\ref{remark:time-variant})}

In the time-invariant setting, $P$ is identical, therefore given data {\small$\mathcal{D}=\left\{\left(s_{h}^{\tau}, a_{h}^{\tau}, r_{h}^{\tau}, s_{h+1}^{\tau}\right)\right\}_{\tau\in[n]}^{h \in[H]}$}, we should modify $n_{s,a}:=\sum_{h=1}^H\sum_{\tau=1}^n\mathbf{1}[s_h^{\tau},a_h^{\tau}=s_h,a_h]$ and
{\small
	\[
	\widehat{P}(s'|s,a)=\frac{\sum_{h=1}^H\sum_{\tau=1}^n\mathbf{1}[(s^{\tau}_{h+1},a^{\tau}_h,s^{\tau}_h)=(s^\prime,s,a)]}{n_{s,a}},\; \widehat{r}(s,a)=\frac{\sum_{h=1}^H\sum_{\tau=1}^n\mathbf{1}[(a^{\tau}_h,s^{\tau}_h)=(s,a)]\cdot r_h^\tau}{n_{s,a}},
	\]
}if $n_{s_h,a_h}>0$ and $\widehat{P}(s'|s,a)={1}/{S},\widehat{r}(s,a)=0$ if $n_{s,a}=0$. Define $\bar{d}^\mu(s,a)=\frac{1}{H}\sum_{h=1}^Hd^\mu_h(s,a)$, then since in this case
\[
\E[n_{s,a}]=\sum_{h=1}^H\sum_{\tau=1}^nd^\mu_h(s_h,a_h)=nH\bar{d}^\mu(s,a),
\]
A similar algorithm should yield 
{\small
\[
\sqrt{\frac{1}{n H d_m}} \sqrt{H}\cdot\sqrt{\mathrm{Var}_{\pi}\left[\sum_{t=1}^{H} r_{t}\right]}\leq \sqrt{\frac{H^2}{n d_m}}.
\]
}Formalizing this result depends on decoupling the dependence between $\widehat{P}$ and $\widehat{V}_h$, which could be more tricky (see \cite{yin2021optimal,ren2021nearly} for two treatments under the uniform data coverage assumption). We leave this as the future work.

\subsection{Derivation in Section~\ref{subsec:btr}}

This follows from the derivation of Section~\ref{subsec:one} by bounding 
\[
v^\star-v^{\widehat{\pi}}\lesssim \sqrt{\frac{1}{n d_m}} \sqrt{H}\cdot\sqrt{\mathrm{Var}_{\pi}\left[\sum_{t=1}^{H} r_{t}\right]}\leq \sqrt{\frac{H}{n d_m}}.
\]

\subsection{Derivation in Section~\ref{subsec:two}}

Using the single concentrability coefficient $C^\star$, when $\pi^\star$ is deterministic,
{\small
\begin{align*}
v^\star-v^{\widehat{\pi}}\lesssim&\sum_{h=1}^H\sum_{(s_h,a_h)\in\mathcal{C}_h}d^{\pi^\star}_h(s_h,a_h)\cdot\sqrt{\frac{\mathrm{Var}_{P_{s_h,a_h}}(r_h+V^\star_{h+1})}{ n\cdot d^\mu_h{(s_h,a_h)}}}
\leq \sqrt{\frac{C^\star}{n }}\sum_{h=1}^H\sum_{(s_h,a_h)\in\mathcal{C}_h}\sqrt{{d^{\pi^\star}_h(s_h,a_h)\cdot \mathrm{Var}_{P_{s_h,a_h}}(r_h+V^\star_{h+1})}}\\
\leq &\sqrt{\frac{C^\star}{n }}\sum_{h=1}^H\sum_{(s_h,a_h)\in\mathcal{S}\times\mathcal{A}}\sqrt{{d^{\pi^\star}_h(s_h,a_h)\cdot \mathrm{Var}_{P_{s_h,a_h}}(r_h+V^\star_{h+1})}}\\
=&\sqrt{\frac{C^\star}{n }}\sum_{h=1}^H\sum_{s_h\in\mathcal{S}}\sqrt{{d^{\pi^\star}_h(s_h,\pi^\star_h(s_h))\cdot \mathrm{Var}_{P_{s_h,\pi^\star_h(s_h)}}(r_h+V^\star_{h+1})}}\\
\leq & \sqrt{\frac{C^\star}{n }}\sum_{h=1}^H\sqrt{\sum_{s_h\in\mathcal{S}}1}\sqrt{{\sum_{s_h\in\mathcal{S}}d^{\pi^\star}_h(s_h,\pi^\star_h(s_h))\cdot \mathrm{Var}_{P_{s_h,\pi^\star_h(s_h)}}(r_h+V^\star_{h+1})}}\\
\leq & \sqrt{\frac{SC^\star}{n }}\sum_{h=1}^H\sqrt{{\sum_{s_h\in\mathcal{S}}d^{\pi^\star}_h(s_h,\pi^\star_h(s_h))\cdot \mathrm{Var}_{P_{s_h,\pi^\star_h(s_h)}}(r_h+V^\star_{h+1})}}
\leq \sqrt{\frac{SC^\star}{n }} \sqrt{H}\cdot\sqrt{\mathrm{Var}_{\pi}\left[\sum_{t=1}^{H} r_{t}\right]}\leq \sqrt{\frac{H^3SC^\star}{n }}.
\end{align*}}
where we use the Cauchy inequality and Lemma~\ref{lem:H3toH2}. Also, from the discussion in Section~\ref{sec:dis_VPVI}, we know this is minimax rate optimal.

\subsection{Derivation in Section~\ref{subsec:three}}

The derivation of Proposition~\ref{prop} is similar to the previous cases except we use the bounds $\mathrm{Var}_{P_h}(V^\star_{h+1})\leq \mathbb{Q}^\star_h$ and $\sum_{h=1}^H r_h\leq \mathcal{B}$. The derivations for the deterministic system or the partially deterministic system are straightforward. For the fast mixing example, we leverage the fact that for any random variable $X$, $|X-\E[X]|\leq \mathrm{rng}(X)$, hence $\mathbb{Q}^\star\leq 1+(\mathrm{rng}V^\star)^2\leq 2$. 

Last but not least, we mention the \emph{per-step environmental norm} $\mathbb{Q}^\star_h:=\max_{s_h,a_h}\mathrm{Var}_{P_{s_h,a_h}}(V^\star_{h+1})$ is more general than its maximal version in \cite{zanette2019tighter} with $\mathbb{Q}^\star:=\max_{s_h,a_h,h}\mathrm{Var}_{P_{s_h,a_h}}(V^\star_{h+1})$. Improvement can be made for the $\mathbb{Q}^\star_h$ version, \emph{e.g.} for the partially deterministic systems, $t\sqrt{\mathbb{Q}^\star/n\bar{d}_m}$ vs $H\sqrt{\mathbb{Q}^\star/n\bar{d}_m}$. Even though \cite{zanette2019tighter} considers the time-invariant setting, \emph{i.e.} $P$ is identical, the quantity $\mathbb{Q}^\star_h:=\max_{s,a}\mathrm{Var}_{P_{s,a}}(V^\star_{h+1})$ can still be much smaller than $\mathbb{Q}^\star$, \emph{e.g.} when the range of $V^\star_t,\ldots,V^\star_H$ is relatively small and the range of $V^\star_1,\ldots,V^\star_{t-1}$ is relatively large. 

In this sense, beyond the current adaptive regret $\sqrt{\mathbb{Q}^\star SAT}$ \citep{zanette2019tighter}, the more adaptive regret should have a form like either
\[
\sqrt{\frac{\sum_{h=1}^H\mathbb{Q}^\star_h SAT}{H}}\quad \text{or}\quad \sum_{h=1}^H \frac{\sqrt{\mathbb{Q}_h^\star SAT}}{H}.
\]
This remains an open question in online RL.

\section{Assisting lemmas}

\begin{lemma}[Multiplicative Chernoff bound \cite{chernoff1952measure}]\label{lem:chernoff_multiplicative}
	Let $X$ be a Binomial random variable with parameter $p,n$. For any $1\geq\theta>0$, we have that 
	$$
	\mathbb{P}[X<(1-\theta) p n]<e^{-\frac{\theta^{2} p n}{2}} . \quad \text { and } \quad \mathbb{P}[X \geq(1+\theta) p n]<e^{-\frac{\theta^{2} p n}{3}}
	$$
\end{lemma}

\begin{lemma}[Hoeffding’s Inequality \cite{sridharan2002gentle}]\label{lem:hoeffding_ineq}
	Let $x_1,...,x_n$ be independent bounded random variables such that $\E[x_i]=0$ and $|x_i|\leq \xi_i$ with probability $1$. Then for any $\epsilon >0$ we have 
	$$
	\P\left( \frac{1}{n}\sum_{i=1}^nx_i\geq \epsilon\right) \leq e^{-\frac{2n^2\epsilon^2}{\sum_{i=1}^n\xi_i^2}}.
	$$
\end{lemma}

\begin{lemma}[Bernstein’s Inequality]\label{lem:bernstein_ineq}
	Let $x_1,...,x_n$ be independent bounded random variables such that $\E[x_i]=0$ and $|x_i|\leq \xi$ with probability $1$. Let $\sigma^2 = \frac{1}{n}\sum_{i=1}^n \mathrm{Var}[x_i]$, then with probability $1-\delta$ we have 
	$$
	\frac{1}{n}\sum_{i=1}^n x_i\leq \sqrt{\frac{2\sigma^2\cdot\log(1/\delta)}{n}}+\frac{2\xi}{3n}\log(1/\delta)
	$$
\end{lemma}

\begin{lemma}[Empirical Bernstein’s Inequality \citep{maurer2009empirical}]\label{lem:empirical_bernstein_ineq}
	Let $x_1,...,x_n$ be i.i.d random variables such that $|x_i|\leq \xi$ with probability $1$. Let $\bar{x}=\frac{1}{n}\sum_{i=1}^nx_i$ and $\widehat{V}_n=\frac{1}{n}\sum_{i=1}^n(x_i-\bar{x})^2$, then with probability $1-\delta$ we have 
	$$
	\left|\frac{1}{n}\sum_{i=1}^n x_i-\E[x]\right|\leq \sqrt{\frac{2\widehat{V}_n\cdot\log(2/\delta)}{n}}+\frac{7\xi}{3n}\log(2/\delta).
	$$
\end{lemma}

\begin{lemma}[Freedman's inequality \cite{tropp2011freedman}]\label{lem:freedman}
	Let $X$ be the martingale associated with a filter $\mathcal{F}$ (\textit{i.e.} $X_i=\E[X|\mathcal{F}_i]$) satisfying $|X_i-X_{i-1}|\leq M$ for $i=1,...,n$. Denote $W:=\sum_{i=1}^n\Var(X_i|\mathcal{F}_{i-1})$  then we have 
	\[
	\P(|X-\E[X]|\geq\epsilon,W\leq \sigma^2)\leq 2 e^{-\frac{\epsilon^2}{2(\sigma^2+M\epsilon/3)}}.
	\]
	Or in other words, with probability $1-\delta$,
	\[
	|X-\E[X]|\leq \sqrt{{8\sigma^2\cdot\log(1/\delta)}}+\frac{2M}{3}\cdot\log(1/\delta), \quad\text{Or} \quad W\geq \sigma^2.
	\]
\end{lemma}

\begin{lemma}[Sum of Total Variance, Lemma~3.4 of \cite{yin2020asymptotically}]\label{lem:H3toH2}

	\begin{align*}
	&\mathrm{Var}_{\pi}\left[\sum_{t=h}^{H} r_{t}\right]\\
	=&\sum_{t=h}^{H}\Bigg(\mathbb{E}_{\pi}\left[\operatorname{Var}\left[r_{t}+V_{t+1}^{\pi}\left(s_{t+1}\right) \mid s_{t}, a_{t}\right] \right]
	+\mathbb{E}_{\pi}\left[\operatorname{Var}\left[\mathbb{E}\left[r_{t}+V_{t+1}^{\pi}\left(s_{t+1}\right) \mid s_{t}, a_{t}\right] \middle| s_{t}\right] \right]\Bigg)
	\end{align*}
	here $s_t,a_t,r_t,\ldots$ is a random trajectory.
\end{lemma}

\begin{remark}
	The infinite horizon discounted setting counterpart is $(I-\gamma P^\pi)^{-1}\sigma_{V^\pi}\leq (1-\gamma)^{-3/2}$.
\end{remark}

\begin{lemma}[Empirical Bernstein Inequality]\label{lem:sqrt_var_diff}
	Let $n\geq 2$ and $V\in\R^S$ be any function with $||V||_\infty\leq H$, $P$ be any $S$-dimensional distribution and $\widehat{P}$ be its empirical version using $n$ samples. Then with probability $1-\delta$,
	\[
	\left|\sqrt{\mathrm{Var}_{\widehat{P}}(V)}-\sqrt{\frac{n-1}{n}\mathrm{Var}_{{P}}(V)}\right|\leq 2H\sqrt{\frac{\log(2/\delta)}{n-1}}.
	\]
\end{lemma}

\begin{proof}
	This is a directly application of Theorem~10 in \cite{maurer2009empirical}. Indeed, by direct translating Theorem~10 of \cite{maurer2009empirical}, 
	\[
	V_n(V)=\frac{1}{n(n-1)} \sum_{1 \leq i<j \leq n}\left(V(s_{i})-V(s_{j})\right)^{2}=\frac{1}{n}\sum_{i=1}^n(V(s_i)-\overline{V})^2=\mathrm{Var}_{\widehat{P}}(V).
	\]
	where $s_i\sim P$ are i.i.d random variables and 
	\begin{align*}
	\E[V_n]=&\E\left[\mathrm{Var}_{\widehat{P}}(V)\right]=\E\left[\E_{\widehat{P}}[V^2]-\left(\E_{\widehat{P}}\left[V\right]\right)^2\right]\\
	=&\E\left[\frac{1}{n}\sum_{i=1}^n V^2(s_i)\right]-\E\left[\left(\frac{1}{n}\sum_{i=1}^n V(s_i)\right)^2\right]\\
	=&\E\left[V^2\right]-\frac{1}{n^2}\E\left[\sum_{i=1}^n V^2(s_i)+2\sum_{1\leq i<j\leq n}V(s_i)V(s_j)\right]\\
	=&\E\left[V^2\right]-\frac{1}{n}\E\left[V^2\right]-2\frac{n(n-1)/2}{n^2}(\E[V])^2\\
	=&\frac{n-1}{n}\mathrm{Var}_{{P}}(V).
	\end{align*}
	Therefore by Theorem~10 of \cite{maurer2009empirical} we directly have the result.

\end{proof}

\subsection{Extend Value Difference}

The extended value difference lemma helps characterize the difference between the estimated value $\widehat{V}_1$ and the true value $V_1^\pi$, which was first summarized in \cite{cai2020provably} and also used in \cite{jin2020pessimism}.

\begin{lemma}[Extended Value Difference (Section~B.1 in \cite{cai2020provably})]\label{lem:evd}
	Let $\pi=\{\pi_h\}_{h=1}^H$ and $\pi'=\{\pi'_h\}_{h=1}^H$ be two arbitrary policies and let $\{\widehat{Q}_h\}_{h=1}^H$ be any given Q-functions. Then define $\widehat{V}_{h}(s):=\langle\widehat{Q}_{h}(s, \cdot), \pi_{h}(\cdot \mid s)\rangle$ for all $s\in\mathcal{S}$. Then for all $s\in\mathcal{S}$,
	
	\begin{equation}
	\begin{aligned}
	\widehat{V}_{1}(s)-V_{1}^{\pi^{\prime}}(s)=& \sum_{h=1}^{H} \mathbb{E}_{\pi^{\prime}}\left[\langle\widehat{Q}_{h}\left(s_{h}, \cdot\right), \pi_{h}\left(\cdot \mid s_{h}\right)-\pi_{h}^{\prime}\left(\cdot \mid s_{h}\right)\rangle \mid s_{1}=s\right] \\
	&+\sum_{h=1}^{H} \mathbb{E}_{\pi^{\prime}}\left[\widehat{Q}_{h}\left(s_{h}, a_{h}\right)-\left(\mathcal{T}_{h} \widehat{V}_{h+1}\right)\left(s_{h}, a_{h}\right) \mid s_{1}=s\right]
	\end{aligned}
	\end{equation}
	where $(\mathcal{T}_{h} V)(\cdot,\cdot):=r_h(\cdot,\cdot)+(P_hV)(\cdot,\cdot)$ for any $V\in\R^S$.
\end{lemma}

\begin{proof}
	
	Denote $\xi_h=\widehat{Q}_h-\mathcal{T}_h\widehat{V}_{h+1}$. For any $h\in[H]$, we have 
	\begin{align*}
	\widehat{V}_h-V^{\pi'}_h&=\langle \widehat{Q}_h,\pi_h\rangle-\langle {Q}^{\pi'}_h,\pi'_h\rangle\\
	&=\langle \widehat{Q}_h,\pi_h-\pi'_h\rangle+\langle \widehat{Q}_h-{Q}^{\pi'}_h,\pi'_h\rangle\\
	&=\langle \widehat{Q}_h,\pi_h-\pi'_h\rangle+\langle P_h(\widehat{V}_{h+1}-V^{\pi'}_{h+1})+\xi_h,\pi'_h\rangle\\
	&=\langle \widehat{Q}_h,\pi_h-\pi'_h\rangle+\langle P_h(\widehat{V}_{h+1}-V^{\pi'}_{h+1}),\pi'_h\rangle+
	\langle \xi_h,\pi'_h\rangle
	\end{align*}
	recursively apply the above for $\widehat{V}_{h+1}-V^{\pi'}_{h+1}$ and use the $\mathbb{E}_{\pi'}$ notation (instead of the inner product of $P_h,\pi'_h$) we can finish the prove of this lemma.
\end{proof}

The following lemma helps to characterize the gap between any two policies.

\begin{lemma}\label{lem:decompose_difference} 
	Let $\widehat{\pi}=\left\{\widehat{\pi}_{h}\right\}_{h=1}^{H}$ and $\widehat{Q}_h(\cdot,\cdot)$ be the arbitrary policy and Q-function and also $\widehat{V}_h(s)=\langle \widehat{Q}_h(s,\cdot),\widehat{\pi}_h(\cdot|s)\rangle$ $\forall s\in\mathcal{S}$.  and $\xi_h(s,a)=(\mathcal{T}_h\widehat{V}_{h+1})(s,a)-\widehat{Q}_h(s,a)$ element-wisely. Then for any arbitrary $\pi$, we have 
	\begin{align*}
	V_1^{\pi}(s)-V_1^{\widehat{\pi}}(s)=&\sum_{h=1}^H\E_{\pi}\left[\xi_h(s_h,a_h)\mid s_1=s\right]-\sum_{h=1}^H\E_{\widehat{\pi}}\left[\xi_h(s_h,a_h)\mid s_1=s\right]\\
	+&\sum_{h=1}^{H} \mathbb{E}_{\pi}\left[\langle\widehat{Q}_{h}\left(s_{h}, \cdot\right), \pi_{h}\left(\cdot | s_{h}\right)-\widehat{\pi}_{h}\left(\cdot | s_{h}\right)\rangle \mid s_{1}=x\right]
	\end{align*}
	where the expectation are taken over $s_h,a_h$.
\end{lemma}

\begin{proof}
	Note the gap can be rewritten as 
	\[
	V_1^{\pi}(s)-V_1^{\widehat{\pi}}(s)=V_1^{\pi}(s)-\widehat{V}_1(s)+\widehat{V}_1(s)-V_1^{\widehat{\pi}}(s).
	\]
	By Lemma~\ref{lem:evd} with $\pi=\widehat{\pi}$, $\pi'=\pi$, we directly have 
	\begin{equation}\label{eqn:inter1}
	V_1^{\pi}(s)-\widehat{V}_1(s)=\sum_{h=1}^H\E_{\pi}\left[\xi_h(s_h,a_h)\mid s_1=s\right]
	+\sum_{h=1}^{H} \mathbb{E}_{\pi}\left[\langle\widehat{Q}_{h}\left(s_{h}, \cdot\right), \pi_{h}\left(\cdot | s_{h}\right)-\widehat{\pi}_{h}\left(\cdot | s_{h}\right)\rangle \mid s_{1}=s\right]
	\end{equation}
	Next apply Lemma~\ref{lem:evd} again with $\pi=\pi'=\widehat{\pi}$, we directly have 
	\begin{equation}\label{eqn:inter2}
	\widehat{V}_1(s)-V_1^{\widehat{\pi}}(s)=-\sum_{h=1}^H\E_{\widehat{\pi}}\left[\xi_h(s_h,a_h)\mid s_1=s\right].
	\end{equation}
	Combine the above two results we prove the stated result.
\end{proof}

\subsection{Hellinger Distance}

\begin{definition}\label{def:Hellinger}
	Let $f,g$ are the two probability densities on the same probability space $\mathcal{X}$. Then the Hellinger distance between $f$ and $g$ is defined as:
	\[
	d_{Hel}^2(f,g)=\frac{1}{2}\int_{x\in\mathcal{X}}\left(\sqrt{f(x)}-\sqrt{g(x)}\right)^2dx=1-\int_{x\in\mathcal{X}}\sqrt{f(x)g(x)}dx.
	\]
	In particular, it holds that 
	\[
	\lvert\lvert f-g\rvert\rvert_{TV}\leq \sqrt{2}\cdot d_{Hel}(f,g).
	\]
	
\end{definition}

\subsection{Propoerty of Local Instance}
\begin{lemma}\label{lem:local_instance}
	
Recall the definition of $P'$ in \eqref{eqn:definition_of_P_prime} 
\[
P'_h(s_{h+1}|s_h,a_h)=P_h(s_{h+1}|s_h,a_h)+\frac{P_h(s_{h+1}|s_h,a_h)\left(V_{h+1}^\star(s_{h+1})-\E_{P_{s_h,a_h}}[V_{h+1}^\star]\right)}{8\sqrt{\zeta\cdot n_{s_h,a_h}\cdot \Var_{P_{s_h,a_h}}(V_{h+1}^\star)}},
\]
where $\zeta=  H/\bar{d}_m$ and $P$ is the original instance. Then :
\begin{itemize}
	\item with high probability, when $n\geq C\cdot \sup_{h,s_{h+1},s_h,a_h} \left(\frac{H}{1/P_h(s_{h+1}|s_h,a_h)-1}\right)^2\cdot\frac{1}{H\cdot\Var(V^\star_{h+1})}$, $P'$ is a valid probability distribution; 
	\item $d_{Hel}(P'_h(\cdot|s_h,a_h),P_h(\cdot|s_h,a_h))\leq\frac{1}{\sqrt{2nH}}$ for $n$ sufficiently large;
	\item Elementwisely, $(P'_h-P_h)V^\star_{h+1}\geq \mathbf{0}$ for all $h\in[H]$.
\end{itemize}

\end{lemma}

\begin{remark}\label{remark}
	WLOG, let us assume $n_{s_h,a_h}\cdot \Var_{P_{s_h,a_h}}(V_{h+1}^\star)>0$ and this is valid since: 1. if $\Var_{P_{s_h,a_h}}(V_{h+1}^\star)=0$ then these is no need to use multiple samples to test $P_h(\cdot|s_h,a_h)$ if $P_h(\cdot|s_h,a_h)$ is deterministic or if $V_{h+1}^\star\equiv 0$ then we can define $\frac{\left(V_{h+1}^\star(s_{h+1})-\E_{P_{s_h,a_h}}[V_{h+1}^\star]\right)}{2\sqrt{n_{s_h,a_h}\cdot \Var_{P_{s_h,a_h}}(V_{h+1}^\star)}}=0$. 2. If $d^\mu_h(s_h,a_h)>0$, then by Lemma~\ref{lem:chernoff_multiplicative} $n_{s_h,a_h}>0$ with high probability.
\end{remark}

\begin{proof}[Proof of Lemma~\ref{lem:local_instance}]
\textbf{Proof of item 1.}	First,
\begin{align*}
&\sum_{s_{h+1}}P'_h(s_{h+1}|s_h,a_h)=\sum_{s_{h+1}}P_h(s_{h+1}|s_h,a_h)\\
+&\sum_{s_{h+1}}\frac{P_h(s_{h+1}|s_h,a_h)\left(V_{h+1}^\star(s_{h+1})-\E_{P_{s_h,a_h}}[V_{h+1}^\star]\right)}{8\sqrt{\zeta\cdot n_{s_h,a_h}\cdot \Var_{P_{s_h,a_h}}(V_{h+1}^\star)}}\\
=&1+\frac{\E_{P_{s_h,a_h}}[V_{h+1}^\star]-\E_{P_{s_h,a_h}}[V_{h+1}^\star]}{8\sqrt{\zeta\cdot n_{s_h,a_h}\cdot \Var_{P_{s_h,a_h}}(V_{h+1}^\star)}}=1\\
\end{align*}

Second, the non-negativity holds as long as
\[
n_{s_h,a_h}\geq \left(\frac{H}{1/P_h(s_{h+1}|s_h,a_h)-1}\right)^2\cdot\frac{1}{\zeta\cdot \Var(V^\star_{h+1})},\quad \forall s_h,a_h \quad s.t. \quad \Var(V^\star_{h+1})>0.
\]
This is guaranteed by Lemma~\ref{lem:chernoff_multiplicative} with high probability when 
\[
n\geq C\cdot \sup_{h,s_{h+1},s_h,a_h} \left(\frac{H}{1/P_h(s_{h+1}|s_h,a_h)-1}\right)^2\cdot\frac{1}{H\cdot\Var(V^\star_{h+1})},
\]
where the $\sup$ is over all terms such that $\Var(V^\star_{h+1})d^\mu_h(s_h,a_h)>0$.

\textbf{Proof of item 2.}
Denote $\Delta_h:=\frac{V_{h+1}^\star(s_{h+1})-\E_{P_{s_h,a_h}}[V_{h+1}^\star]}{8\sqrt{\zeta\cdot n_{s_h,a_h}\cdot \Var_{P_{s_h,a_h}}(V_{h+1}^\star)}}$
By the definition, we have 
\begin{align*}
&d^2_{Hel}(P'_h(\cdot|s_h,a_h),P_h(\cdot|s_h,a_h))=\left\lvert 1-\sum_{s_{h+1}}\sqrt{P_h(s_{h+1}|s_h,a_h)\cdot P'_h(s_{h+1}|s_h,a_h)}\right\rvert \\
&=\left\lvert 1-\sum_{s_{h+1}}\sqrt{P_h^2(s_{h+1}|s_h,a_h)(1+\Delta_h)} \right\rvert=\left\lvert 1-\sum_{s_{h+1}}P_h(s_{h+1}|s_h,a_h)\sqrt{(1+\Delta_h)} \right\rvert\\
&=\left\lvert \sum_{s_{h+1}}P_h(s_{h+1}|s_h,a_h)-\sum_{s_{h+1}}P_h(s_{h+1}|s_h,a_h)\sqrt{(1+\Delta_h)} \right\rvert=\left\lvert \sum_{s_{h+1}}P_h(s_{h+1}|s_h,a_h)\left(1-\sqrt{(1+\Delta_h)} \right)\right\rvert\\
&=\left\lvert \sum_{s_{h+1}}P_h(s_{h+1}|s_h,a_h)\left(1-\left(1+\frac{\Delta}{2}-\frac{\Delta^2}{8}\right) \right)\right\rvert \quad (*)\\
&=\left\lvert \sum_{s_{h+1}}P_h(s_{h+1}|s_h,a_h)\left(-\frac{\Delta}{2}+\frac{\Delta^2}{8}\right) \right\rvert =\left\lvert \sum_{s_{h+1}}P_h(s_{h+1}|s_h,a_h)\cdot\frac{\Delta^2}{8} \right\rvert =\frac{1}{8^3\cdot \zeta\cdot n_{s_h,a_h}  }\leq \frac{1}{2n\cdot H}\\
\end{align*}
the step $(*)$ comes from second order Taylor expansion (where we omit the higher order since $n$ is sufficiently large already) and the next equation uses $\sum_{s_{h+1}} P_h(s_{h+1}|s_h,a_h)\cdot \Delta=0$ and the last inequality uses Lemma~\ref{lem:chernoff_multiplicative} $n_{s_h,a_h}$ such that $n_{s_h,a_h}\cdot\zeta\geq C \cdot n\cdot H$ with high probability and $d^2_{Hel}(P'_h(\cdot|s_h,a_h),P_h(\cdot|s_h,a_h))\leq \frac{1}{2nH}$.

\textbf{Proof of item 3.} Note
\begin{align*}
&[(P'_h-P_h)V^\star_{h+1}](s_h,a_h)=\sum_{s_{h+1}}(P'_h(s_{h+1}|s_h,a_h)-P_h(s_{h+1}|s_h,a_h))V^\star_{h+1}(s_{h+1})\\
&=\sum_{s_{h+1}}(P'_h(s_{h+1}|s_h,a_h)-P_h(s_{h+1}|s_h,a_h))(V^\star_{h+1}(s_{h+1})-\E_{P_{s_h,a_h}}[V^\star_{h+1}])\\
&=\frac{1}{2}\sqrt{\frac{\Var_{P_{s_h,a_h}}(V^\star_{h+1})}{\zeta\cdot n_{s_h,a_h}}}\geq 0
\end{align*}
	
\end{proof}
